\documentclass{article}

\usepackage[final]{neurips_2024}

\usepackage[utf8]{inputenc} 
\usepackage[T1]{fontenc}    
\usepackage{hyperref}       
\usepackage{url}            
\usepackage{booktabs}       
\usepackage{amsfonts}       
\usepackage{nicefrac}       
\usepackage{microtype}      
\usepackage[table, dvipsnames]{xcolor}

\usepackage[normalem]{ulem}
\usepackage{amsmath}

\usepackage{enumitem}

\usepackage{multirow}
\usepackage{makecell}
\usepackage{threeparttable}

\usepackage{array}
\newcommand{\PreserveBackslash}[1]{\let\temp=\\#1\let\\=\temp}
\newcolumntype{C}[1]{>{\PreserveBackslash\centering}p{#1}}
\usepackage{hhline}
\newcolumntype{?}{!{\vrule width 1pt}}

\usepackage{amsmath}
\usepackage{amsthm}
\usepackage{amssymb}
\usepackage{thmtools}
\usepackage{thm-restate}
\usepackage{comment}
\usepackage{dsfont}
\usepackage{cleveref}

\usepackage[toc,page,header]{appendix}
\usepackage{minitoc}

\usepackage{svg}

\usepackage{algorithm,algorithmic}

\usepackage{accents}
\newcommand{\ubar}[1]{\underaccent{\bar}{#1}}

\usepackage{tikz}
\usepackage{mathtools}

\definecolor{darkblue}{rgb}{0,0,0.95}

\usepackage{thmtools}
\usepackage{thm-restate}

\usepackage{subfigure}
\usepackage{wrapfig}

\def\eqdef{:=}
\def\VAR{\mathrm{Var}}

\def\Regret{\mathrm{Reg}}

\newtheorem{lemma}{Lemma}

\newtheorem{remark}{Remark}

\usepackage{bbm}
\newcommand{\Olog}{\tilde{\mathcal{O}}}

\newtheorem{theorem-rst}[theorem]{Theorem}
\newtheorem{lemma-rst}[lemma]{Lemma}
\newtheorem{proposition-rst}[lemma]{Proposition}
\newtheorem{assumption-rst}[lemma]{Assumption}
\newtheorem{claim-rst}[lemma]{Claim}
\newtheorem{corollary-rst}[lemma]{Corollary}

\usepackage{mathtools}
\DeclarePairedDelimiter\br{(}{)}
\DeclarePairedDelimiter\brs{[}{]}
\DeclarePairedDelimiter\brc{\{}{\}}

\DeclarePairedDelimiter\abs{\lvert}{\rvert}
\DeclarePairedDelimiter\norm{\lVert}{\rVert}

\DeclarePairedDelimiter\ceil{\lceil}{\rceil}

\DeclareMathOperator*{\argmax}{arg\,max}

\newcommand{\E}{\mathbb{E}}
\newcommand{\R}{\mathbb{R}}
\newcommand{\G}{\mathbb{G}}

\newcommand{\X}{\mathcal{X}}
\newcommand{\Mcal}{\mathcal{M}}

\newcommand{\Acal}{\mathcal{A}}
\newcommand{\Ocal}{\mathcal{O}}

\newcommand{\Scal}{\mathcal{S}}
\newcommand{\Rcal}{\mathcal{R}}

\newcommand{\Lcal}{\mathcal{L}}

\newcommand{\bR}{\boldsymbol{R}}
\newcommand{\bs}{\boldsymbol{s}}
\newcommand{\bzero}{\boldsymbol{0}}

\newcommand{\Ind}[1]{\mathds{1}\brc*{#1}}

\usepackage[colorinlistoftodos]{todonotes}

\makeatletter
\newcommand{\printfnsymbol}[1]{%
  \textsuperscript{\@fnsymbol{#1}}%
}

\def\showComments{} 

\ifdefined\showComments
    \newcommand{\comN}[1]{\textcolor{blue}{\{Nadav: #1\}}}
\else
    \newcommand{\comN}[1]{}
\fi

\usepackage{selectp}

\title{Reinforcement Learning with Lookahead Information}

\author{%
 Nadav Merlis \\
    FairPlay Joint Team, CREST, ENSAE Paris \\
    \texttt{nadav.merlis@ensae.fr}
}

\begin{document}
\doparttoc
\faketableofcontents

\maketitle

\begin{abstract}
  We study reinforcement learning (RL) problems in which agents observe the reward or transition realizations at their current state \emph{before deciding which action to take}. Such observations are available in many applications, including transactions, navigation and more. When the environment is known, previous work shows that this lookahead information can drastically increase the collected reward. However, outside of specific applications, existing approaches for interacting with unknown environments are not well-adapted to these observations. In this work, we close this gap and design provably-efficient learning algorithms able to incorporate lookahead information. To achieve this, we perform planning using the empirical distribution of the reward and transition observations, in contrast to vanilla approaches that only rely on estimated expectations. We prove that our algorithms achieve tight regret versus a baseline that also has access to lookahead information -- linearly increasing the amount of collected reward compared to agents that cannot handle lookahead information.
\end{abstract}

\section{Introduction}
In reinforcement learning (RL), agents sequentially interact with a changing environment, aiming to collect as much reward as possible. While performing actions that yield immediate rewards is enticing, agents must also bear in mind that actions influence the state of the environment, affecting the potential reward that could be collected in future steps. When the environment is unknown, agents also need to balance reward maximization based on previous data and exploration -- gathering of data that might improve future reward collection.

In the standard interaction model, at each timestep, agents first choose an action and only then observe its outcome on the rewards and state dynamics. As such, agents can only maximize the expected rewards, collected through the expected dynamics. Yet, in many applications, some information on the immediate outcome of actions is known \emph{before} actions are performed. For example, when agents interact through transactions, prices and traded goods are usually agreed upon before performing any exchange (`reward information'). Alternatively, in navigation problems, nearby traffic information is known to the agent before choosing which path to go through (`transition information').

In a recent work, \citet{merlis2024value} shows that even for agents with full statistical knowledge of the environment, such `lookahead' information can drastically increase the reward collected by agents -- by a multiplicative factor of up to $AH$ when immediate rewards are revealed in advance and $A^{H/2}$ when observing the immediate future transitions.\footnote{$A$ is the size of the action space, $S$ is the size of the state space and $H$ is the interaction length.} Intuitively, agents do not only gain from instantaneously using this information -- they can also adapt their planning to account for lookahead information being revealed in subsequent states, significantly increasing their future values. However, the work of \citet{merlis2024value} only tackles planning settings in which the model is known and does not provide algorithms or guarantees when interacting with unknown environments. 

In this work, we aim to design provably-efficient agents that learn how to interact when given immediate (`one-step lookahead') reward or transition information before choosing an action, under the episodic tabular Markov Decision Process model. While such information can always be embedded into the state of the environment, the state space becomes exponential at best, and continuous at worst, rendering most theoretically-guaranteed approaches both computationally and statistically intractable. To alleviate this, we start by deriving dynamic programming (`Bellman') equations \emph{in the original state space} that characterize the optimal lookahead policies. Inspired by these update rules, we present two variants to the MVP algorithm \citep{zhang2021reinforcement} that allow incorporating either reward or transition lookahead. In particular, we suggest a planning procedure that uses the empirical distribution of the reward/transition observations (instead of the estimated expectations), which might also be applied to other complex settings. We prove that these algorithms achieve tight regret bounds of $\Olog\br*{\sqrt{H^3SAK}}$ and $\Olog\br*{\sqrt{H^2SK}\br{\sqrt{H}+\sqrt{A}}}$  after $K$ episodes (for reward and transition lookahead, respectively), compared to a stronger baseline that also has access to lookahead information. As such, they can collect significantly more rewards than vanilla RL algorithms. 

\textbf{Outline.} We formally define RL problems with reward/transition lookahead in \Cref{section: notation} and further discuss the differences between our setting and standard RL problems in \Cref{section: comparison to standard RL}. Then, we present our results in two complementary sections: \Cref{section:reward lookahead} analyzes reward lookahead while \Cref{section:transition lookahead} analyzes transition lookahead. We end with conclusions and future directions in \Cref{section: conclusions}.

\paragraph{Related Work. }
Problems with varying lookahead information have been extensively studied in control, with model predictive control \citep[MPC,][]{camacho2007model} as the most notable example. Conceptually, when interacting with an environment that might be too complex or hard to model, it is oftentimes convenient to use a simpler model that allows accurately predicting its behavior just in the near future. MPC uses such models to repeatedly update its policy using short-term planning. In some cases, the utilized future predictions consist of additive perturbations to the dynamics \citep{yu2020power}, while other cases involve more general future predictions on the model behavior \citep{li2019online,zhang2021regret,lin2021perturbation,lin2022bounded}. To the best of our knowledge, these studies focus on comparing the performance of the controller to one with full future information (and thus, linear regret is inevitable), sometimes also considering prediction errors. They do not, however, attempt to learn the predictions. In contrast, we estimate the reward/transition distributions and leverage them to better plan, thus increasing the value gained by the agent. In addition, these works focus on continuous (mostly linear) control problems, whereas we study tabular settings; results from any one of these settings cannot be directly applied to the other.

In RL, lookahead is mostly used as a planning tool; namely, agents test the possible outcomes after performing multiple steps to decide which actions to take or to better estimate the value \citep{tamar2017learning,efroni2019combine,efroni2020online,moerland2020think,rosenberg2023planning, el2020lookahead, biedenkapp2021temporl,huang2019continuous}. Specifically, the future value at the end of the lookahead is often estimated using rollouts, and a longer lookahead is more robust to suboptimality of the rollout policy \citep{bertsekas2023course}. 
However, when agents actually interact with the environment, no additional lookahead information is observed. 
One notable exception is \citep{merlis2024value}, which analyzes the potential value increase due to multi-step reward lookahead information (and briefly mentions transition lookahead). However, they only tackle planning settings, where the model is known, and do not study learning. In this work, we continue a long line of literature on regret analysis for tabular RL \citep{jaksch2010near,jin2018q,dann2019policy,zanette2019tighter,efroni2019tight,efroni2021confidence,simchowitz2019non,zhang2021reinforcement,zhang2023settling}. Yet, we are not aware of any existing results on regret minimization with reward or transition lookahead information.  \looseness=-1

Finally, various applications that involve one-step lookahead information have been previously studied. The most notable ones are prophet problems \citep{correa2019recent}, where one-step reward lookahead is obtained, and the Canadian traveler problem with resampling \citep{nikolova2008route}, which can be formulated through one-step transition lookahead. We discuss the relation to these problems and the relevant existing results when analyzing each type of feedback, and also discuss the relation between transition lookahead and stochastic action sets \citep{boutilier2018planning}.

\vspace{-.1cm}


\section{Setting and Notations}
\label{section: notation}
We study episodic tabular Markov Decision Processes (MDPs), defined by the tuple $\Mcal=(\Scal,\Acal,H,P,\Rcal)$, where $\Scal$ is the state space  (of size $S$), $\Acal$ is the action space  (of size $A$) and $H$ is the interaction horizon. At each timestep $h\in\brc*{1,\dots,H}\triangleq[H]$ of an episode $k\in[K]$, an agent, located in state $s_h^k\in\Scal$, chooses an action $a_h^k\in\Acal$ and obtains a reward $R_h^k=R_h(s_h^k,a_h^k)\sim \Rcal_h(s_h^k,a_h^k)$. We assume that the rewards are supported by $[0,1]$ and of expectations $r_h(s,a)$. Afterward, the environment transitions to a state $s_{h+1}^k\sim P_h(\cdot\vert s_h^k,a_h^k)$ and the interaction continues until the end of the episode. We use the notation $\bR\sim \Rcal_h(s)$ (or $\bs'\sim P_h(s)$) to denote reward (next-state) samples for all actions simultaneously at step $h$ and state $s$ and assume independence between different timesteps.\footnote{This assumption is not used by our algorithms: it is only to ensure that the optimal policy is Markovian.} On the other hand, samples from different actions at a specific state/timestep are not necessarily independent. 

\paragraph{Reward Lookahead.} With one-step reward lookahead at timestep $h$ and state $s$, agents first observe the rewards for all actions $\bR_h(s)\triangleq\brc*{R_h(s,a)}_{a\in\Acal}$ and only then choose an action to perform. Formally, we define the set of reward lookahead policies as $\Pi^R=\brc*{\pi:[H]\times\Scal\times[0,1]^A\mapsto \Delta_{\Acal}}$, where $\Delta_{\Acal}$ is the probability simplex, and denote $a_h=\pi_h(s_h,\bR_h)$. The value of a reward lookahead agent is the cumulative rewards gathered by it starting at timestep $h$ and state $s$, denoted by
\begin{align*}
    V^{R,\pi}_h(s)=\E\brs*{\sum_{t=h}^HR_t(s_t,\pi_t(s_t,\bR_t(s_t))\vert s_h=s}.
\end{align*}
We also define the optimal reward lookahead value to be $V^{R,^*}_h(s)=\max_{\pi\in\Pi^R}V^{R,\pi}_h(s)$. When interacting with an unknown environment for $K$ episodes, agents sequentially choose reward lookahead policies $\pi^k\in\Pi^R$ based on all historical information and are measured by their regret,
\begin{align*}
    \Regret^R(K)=\sum_{k=1}^K\br*{V^{R,*}_1(s_1^k) - V^{R,\pi^k}_1(s_1^k)}.
\end{align*}
We allow the initial state of each episode $s_1^k$ to be arbitrarily chosen.

\paragraph{Transition Lookahead.} Denoting $s'_{h+1}(s,a)$, the future state when playing action $a$ at step $h$ and state $s$, one-step transition lookahead agents observe $\bs'_{h+1}(s)\triangleq\brc*{s'_{h+1}(s,a)}_{a\in\Acal}$ before acting. The set of transition lookahead agents is denoted by $\Pi^T=\brc*{\pi:[H]\times\Scal\times\Scal^A\mapsto \Delta_{\Acal}}$ with values 
\begin{align*}
    V^{T,\pi}_h(s)=\E\brs*{\sum_{t=h}^HR_t(s_t,\pi_t(s_t,\bs'_{t+1}(s_t)))\vert s_h=s}.
\end{align*}
The optimal value is $V^{T,^*}_h(s)=\max_{\pi\in\Pi^T}V^{T,\pi}_h(s)$, and we similarly define the regret versus optimal transition lookahead agents as $\Regret^T(K)=\sum_{k=1}^K\br*{V^{T,*}_1(s_1^k) - V^{T,\pi^k}_1(s_1^k)}.$ 

When the type of lookahead is clear from the context, we sometimes denote values by $V^\pi_h$ and $V^*_h$.

\paragraph{Other Notations.} For any $p\in\Delta_n$ and $V\in\R^n$, we define $\VAR_p(V)=\sum_{i=1}^n p_iV_i^2 - \br*{\sum_{i=1}^np_iV_i}^2$. Also, given a transition kernel $P$ and a vector $V\in\R^S$, we let $PV(s,a)=\sum_{s'\in\Scal}P(s'\vert s,a)V(s')$ and similarly define it for value or transition kernel differences. We denote by $n_h^k(s,a)$, the number of times the pair $(s,a)$ was visited at timestep $h$ up to episode $k$ (inclusive) and similarly denote $n_h^k(s)=\sum_{a\in\Acal} n_h^k(s,a)$. We also let $\hat{r}_h^k(s,a)=\frac{1}{n^k_h(s,a)}\sum_{k'=1}^k\Ind{s_h^{k'}=s,a_h^{k'}=a}R_h^{k'}$ and $\hat{P}_h(s'\vert s,a)=\frac{1}{n^k_h(s,a)}\sum_{k'=1}^k\Ind{s_h^{k'}=s,a_h^{k'}=a,s_{h+1}^{k'}=s'}$ be the empirical expected rewards and transition kernel at $(s_h,a_h)=(s,a)$ using data up to episode $k$ and assume they are initialized to be zero. Finally, we denote by $\hat{\Rcal}^{k}_h(s)$, the empirical reward distribution across all actions, and use $\hat{P}_h^k(s)$ to denote the empirical joint next-state distribution for all actions. In particular, if $k_i$ is the $i^{th}$ episode where $s$ was visited at step $h$, to sample $\bR\sim\hat{\Rcal}^{k}_h(s)$, we uniformly sample $i\sim U\br*{\brs*{n^k_h(s)}}$ and return $\bR=\brc*{R_h^{k_i}(s,a)}_{a\in\Acal}$. A sample $\bs'\sim\hat{P}_h^k(s)$ similarly returns $\bs'=\brc*{s'^{k_i}_{h+1}(s,a)}_{a\in\Acal}$.

When we want to indicate the distribution used to calculate an expectation, we sometimes state it in a subscript, e.g., write $E_{\mathcal{R}_h(s)}[R(a)]$ to indicate that $R(a)\sim\mathcal{R}_h(s,a)$ or use $\E_\Mcal$ to emphasize that all distributions are according to an environment $\Mcal$. In this paper, $\Ocal$-notation only hides absolute constants while $\Olog$ hides factors of $\textrm{polylog}(S,A,H,K,\delta)$. We also use the notation $a\vee b=\max\brc*{a,b}$.

\section{Comparing the Values of Lookahead Agents and Vanilla RL agents}
\label{section: comparison to standard RL}
In the classic RL formulation \citep[e.g.,][]{azar2017minimax}, agents only observe the reward and transition after performing an action and aim to maximize the 'no-lookahead' value, defined by 
\begin{align*}
    V^{\pi}_h(s) = \E\brs*{\sum_{t=h}^Hr_t(s_t,\pi_t(s_t)\vert s_h=s},
\end{align*}
where $\pi\in\Pi^{\Mcal}=\brc*{\pi:[H]\times\Scal\mapsto \Delta_{\Acal}}$ is a Markovian policy. The optimal value is $V^{no}_h(s)  = \max_{\pi\in\Pi^{\Mcal}}V^{\pi}_h(s)$ and the regret is classically defined as $\Regret(K) = \sum_{k=1}^K\br*{V^{no}_1(s_1^k) - V^{\pi^k}_1(s_1^k)}.$

By definition, the set of lookahead policies also includes all Markovian policies (since agents are not obliged to use reward/transition information), so the optimal lookahead values are always larger than their no-lookahead counterpart. In other words, denoting the value gain due to lookahead information by $G^R(s) = V^{R,*}_1(s) - V^{no}_1(s)$ and $G^T(s) = V^{T,*}_1(s) - V^{no}_1(s)$, it holds that $G^R(s),G^T(s)\ge0$. In terms of regret, for any fixed algorithm, we can also write
\begin{align*}
    & \Regret(K) = \Regret^R(K) - \sum_{k=1}^K G^R(s_1^k) = \Regret^T(K) - \sum_{k=1}^K G^T(s_1^k).
\end{align*}
As the value gains are non-negative, it directly implies that any regret bound w.r.t. the lookahead value also leads to the same bound for the standard regret. Even more so, in most cases, lookahead information leads to a strict improvement in the value, that is, $G^R(s),G^T(s)\ge G_0>0$. When this happens, any algorithm with sub-linear lookahead regret enjoys a \emph{negative linear} standard regret: 
\begin{center}
    \emph{If $\Regret^R(K) = o(K)$ and $G^R(s_1^k)\ge G_0$ for all $k\in[K]$, then $\Regret(K) \leq  -G_0K + o(K)$.}
\end{center}
The same also holds for transition lookahead. Conversely, any agent that suffers positive standard regret will suffer linear regret compared to the best lookahead agent, i.e., 
\begin{center}
    \emph{If $\Regret(K) \ge0$ and $G^R(s_1^k)\ge G_0$ for all $k\in[K]$, then $\Regret^R(K) \ge G_0K$.}
\end{center}
Notably, any agent that does not use lookahead information will suffer linear lookahead regret in any such environment. We now present two illustrative examples for environments where the lookahead value gain is significant, one for reward lookahead and another for transition lookahead.

\begin{wrapfigure}{r}{0.26\textwidth}
    \vspace{-1.25em}
  \centering
    \includegraphics[width=0.2\textwidth]{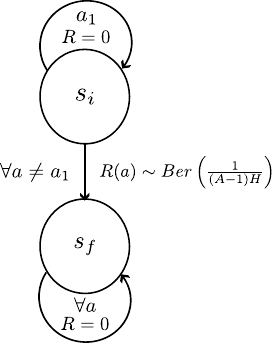}
  \vspace{-.1cm}
  \caption{Two-state\\ prophet-like problem}
  \label{figure:prophet}
\end{wrapfigure}
\paragraph{Reward lookahead.} Consider a simple 2-state environment, depicted in \Cref{figure:prophet}. Starting at $s_i$, agents can either stay there by playing $a_1$, earning no reward, or play any other action and move to the absorbing $s_f$, obtaining a Bernoulli reward $Ber\br*{\nicefrac{1}{(A-1)H}}$. Actions in the terminal state $s_f$ yield no reward. Without observing the rewards, agents will arbitrarily move from $s_i$ to $s_f$, obtaining a reward $V^{no}=\nicefrac{1}{(A-1)H}$ in expectation. On the other hand, when agents observe the rewards before acting, they should move from $s_i$ to $s_f$ only if a reward was realized for some action (and otherwise, stay in $s_i$ by playing $a_1$). Such agents will have $(A-1)H$ opportunities to observe a unit reward across all timesteps and actions, collecting in expectation $V^{R,*}=\br*{1-\nicefrac{1}{(A-1)H}}^{(A-1)H}\ge1-\nicefrac{1}{e}$. In other words, just by observing the rewards before acting, the agent's value multiplicatively increases by almost $V^{R,*}/V^{no}\approx AH$. Moreover, the additive value gain is $G^R\approx 1-\nicefrac{1}{e}$, so sub-linear lookahead regret with reward information results with a negatively-linear standard regret of $\Regret(K)\lesssim -(1-\nicefrac{1}{e})K$.

\paragraph{Transition lookahead.} Consider a chain of $H/2$ states (also described in further detail at \Cref{appendix: transition lookahead example} and depicted at \Cref{figure: chain transition lookahead}). In each state, one action deterministically keeps the agent in its current state, while all other actions move the agent one state forward w.p. $1/A$, but lead to a terminal non-rewarding state otherwise. If the reward is located at the end of the chain, any standard RL agent can collect it only at an exponentially low probability. On the other hand, transition lookahead agents would move forward only if there is an action that allows it while staying at their current state otherwise; such agents will collect the rewards at the end of the chain with constant probability. More specifically, any no-lookahead agent can collect at most $V^{no}=\Ocal(HA^{-H/2})$ rewards, while transition lookahead agents can collect $V^{T,*}=\Omega(H)$; as such, lookahead agents achieve exponential increase in value, and sublinear regret versus the best lookahead agent will yield a standard regret of $\Regret(K)\lesssim -HK$.

In the following sections, we will present agents that are guaranteed to always achieve sublinear regret compared to the best lookahead agent.


\section{Planning and Learning with One-Step Reward Lookahead}
\label{section:reward lookahead}
In this section, we analyze RL settings with one-step reward lookahead, in which immediate rewards are observed before choosing an action. One well-known example of this situation is the prophet problem \citep{correa2019recent}, where an agent sequentially observes values from known distributions. Upon observing a value, the agent decides whether to take it as a reward and stop the interaction, or discard it and continue to observe more values. This problem has numerous applications and extensions concerning auctions and posted-price mechanisms \citep{correa2017posted}. As shown in \citep{merlis2024value}, it is critical to observe the distribution values before taking a decision; otherwise, the agent's revenue can decrease by a factor of $H$. Notably, the example presented in \Cref{figure:prophet} is a small variant of the prophet problem, where the agent can either take one of $A-1$ values and finish the interaction or discard them and continue playing by staying at $s_i$; we showed that for this example, the lookahead information increases the value by a factor of $V^{R,*}/V^{no}\approx AH$.

The most natural way to tackle this setting is to extend (augment) the state space to contain the observed rewards; this way, we transition from a state and reward observations to a new state with new reward observations and return to the vanilla MDP formulation. However, this comes at a great cost. Even for Bernoulli rewards, there are $2^A$ possible reward combinations at any given state, and the augmentation increases the state space by this factor -- leading to an exponentially-large state space. Even worse, for continuous rewards, the augmented state space becomes continuous, and any performance guarantees that depend on the size of the state space immediately become vacuous. Hence, algorithms that na\"ively use this reduction are expected to be both computationally and statistically intractable. We refer to \Cref{appendix: extended MDP rewards} for further details on one such augmentation.

We take a different approach and derive Bellman equations for this setting in the \emph{original state space}.
\begin{restatable}{proposition-rst}{rewardDP}
\label{prop: reward DP}
The optimal value of one-step reward lookahead agents satisfies 
\begin{align*}
    &V^{R,*}_{H+1}(s)=0, &\forall s\in\Scal,\\
    &V^{R,*}_h(s) = \E_{\bR\sim\Rcal_h(s)}\brs*{\max_{a\in\Acal}\brc*{R_h(s,a) +\sum_{s'\in\Scal}P_h(s'\vert s,a)V^{R,*}_{h+1}(s')}},&\forall s\in\Scal, h\in[H].
\end{align*}
Also, given reward observations $\bR=\brc*{R(a)}_{a\in\Acal}$ at state $s$ and step $h$, the optimal policy is $$\pi^*_h(s,\bR)\in\argmax_{a\in\Acal}\brc*{R(a)+\sum_{s'\in\Scal}P_h(s'\vert s,a)V^{R,*}_{h+1}(s')}.$$
\end{restatable}
We prove \Cref{prop: reward DP} in \Cref{appendix: extended MDP rewards}, where we present an equivalent environment with extended state space in which one could apply the standard Bellman equations \citep{puterman2014markov} to calculate the value with reward lookahead. In contrast to the previously discussed augmentation approach, we find it more convenient to divide the augmentation into two steps -- at odd steps $2h-1$, the augmented environment would be in a state $s_h\times\bzero$, while at even steps $2h$, the state is $s_h\times\bR_h$. Doing so creates an overlap between the values of the original and augmented environments at odd steps, simplifying the proofs. We also use this augmentation to prove a variant of the law of total variance \citep[LTV, e.g.][]{azar2017minimax} and a value-difference lemma \citep[e.g.][]{efroni2019tight}.

We remark that calculating the exact value is not always tractable -- even for $S=H=1$ (bandit problems) and Gaussian rewards, \Cref{prop: reward DP} requires calculating the expectation of the maximum of Gaussian random variables, which does not admit any simple closed-form solution. On the other hand, these equations allow approximating the value by using reward samples -- in the following, we show that it can be used to achieve tight regret bounds when the environment is unknown.


\subsection{Regret-Minimization with Reward Lookahead}

\begin{algorithm}[t]
\caption{Monotonic Value Propagation with Reward Lookahead (MVP-RL)} \label{alg: MVP reward lookahead short}
\begin{algorithmic}[1]
\STATE {\bf Require:} $\delta\in(0,1)$, bonuses $b_{k,h}^r(s), b_{k,h}^p(s,a)$
\FOR{$k=1,2,...$}
    \STATE  Initialize $\bar{V}^k_{H+1}(s)=0$ 
    \FOR{$h=H,H-1,..,1$}
        \STATE Calculate the truncated values for all $s\in\Scal$
        \begin{align*}
            &\bar{V}^k_h(s) = \min\brc*{\E_{\bR\sim\hat{\Rcal}^{k-1}_h(s)}\brs*{\max_{a\in\Acal}\brc*{R(a) + b_{k,h}^{p}(s,a) + \hat{P}^{k-1}_{h}\bar{V}^k_{h+1}(s,a)}} + b^r_{k,h}(s), H} 
        \end{align*}
    \ENDFOR
    \FOR{$h=1,2,\dots H$}
        \STATE Observe $s_h^k$ and $R_h^k(s_h^k,a)$ for all $a\in\Acal$
        \STATE Play an action 
        $a_h^k\in\argmax_{a\in\Acal}\brc*{R_h^k(s_h^k,a) + b_{k,h}^{p}(s_h^k,a) + \hat{P}^{k-1}_{h}\bar{V}^k_{h+1}(s_h^k,a)}$
        \STATE Collect the reward $R^k_h(s^k_h,a^k_h)$ and transition to the next state $s^k_{h+1}\sim P_h(\cdot\vert s_h^k,a_h^k)$
    \ENDFOR
\ENDFOR
\end{algorithmic}
\end{algorithm}
We now present a tractable algorithm that achieves tight regret bounds with one-step reward lookahead. Specifically, we modify the Monotonic Value Propagation (MVP) algorithm \citep{zhang2021reinforcement} to perform planning using the \emph{empirical reward distributions} -- instead of using the empirical reward expectations. To compensate for transition uncertainty, we add a transition bonus that uses the variance of the optimistic next-state values (w.r.t. the empirical transition kernel), designed to be monotone in the future value. Such construction permits using the variance of optimistic values for the bonus calculation while being able to later replace it with the variance of the optimal value (see discussion in  \citealt{zhang2021reinforcement}). A reward bonus is used for the value calculation, but does not affect the action choice in  the current state. Intuitively, this is because we get the same amount of information for all the actions of a state, so they have the same level of uncertainty -- there is no need for bonuses to encourage reward exploration at the action level.

A high-level description of the algorithm is presented in \Cref{alg: MVP reward lookahead short}, while the full algorithm and its bonuses are stated in \Cref{appendix: MVP for reward lookahead}. Notice that the planning requires calculating the expected maximum using the empirical distribution, whose support always contains at most $K$ elements, so both the memory and computations are polynomial. The algorithm ensures the following guarantees:
\begin{restatable}{theorem-rst}{MVPRL}
\label{theorem: regret MVP-RL}
When running MVP-RL, with probability at least $1-\delta$ uniformly for all $K\ge1$, it holds that $\Regret^R(K)\le \Ocal\br*{\sqrt{H^3SAK}\ln\frac{SAHK}{\delta} + H^3S^2A\br*{\ln\frac{SAHK}{\delta}}^2}$.
\end{restatable}
See proof in \Cref{appendix: reward lookahead regret}. Remarkably, our upper bound matches the standard lower bound for episodic RL of $\Omega\br*{\sqrt{H^3SAK}}$ \citep{domingues2021episodic} up to log-factors; this lower bound is proved for known deterministic rewards, so in particular, it also holds for problems with reward lookahead. 

To our knowledge, the only comparable bounds in settings with reward lookahead were proven to prophet problems; as agents observe (up to) $n$ distributions at a fixed order, it can be formulated as a deterministic chain-like MDP, with $H= n$, $S=n+1$ and $A=2$. Agents start at the head of the chain and can either advance without collecting a reward or collect the observed reward and move to a terminal non-rewarding state (for more details, see \citealt{merlis2024value}). For this problem, \citep{gatmiry2024bandit} proved a regret bound of $\Olog(n^3\sqrt{K})$ (albeit requiring a weaker form of feedback), and \citep{agarwal2023semi} proved a bound of $\Olog(n\sqrt{T})$ -- slightly better than ours, but heavily relies on the ability to control which distributions to observe, which is a specific instance of deterministic transitions. We are unaware of any previous results that cover general Markovian dynamics.

\subsection{Proof Concepts}
\label{section: reward lookahead proof sketch}
When analyzing the regret of RL algorithms, a key step usually involves bounding the difference between the value of a policy in two different environments (`value-difference lemma'). In particular, for a given policy $\pi^k$, many algorithms maintain a confidence interval on the value $V_h^{\pi^k}(s)\in\brs*{\ubar{V}_h^k(s),\bar{V}_h^k(s)}$, calculated based on optimistic and pessimistic MDPs that use the empirical model with bonuses/penalties \citep{dann2019policy,zanette2019tighter,efroni2021confidence}. Then, the instantaneous regret (without lookahead) is bounded using the optimistic values by
\begin{align*}
    \bar{V}_h^k(s_h) - V_h^{\pi^k}(s_h) 
    &= \br*{\hat{r}_h^{k-1}(s_h,a_h) - r_h(s_h,a_h)} + \br*{\hat{P}_h^{k-1} - P_h}\bar{V}_h^k(s_h,a_h)
    \\
    &\quad+ P_h\br*{\bar{V}_{h+1}^k - V_{h+1}^{\pi^k}}(s_h,a_h) + \textrm{bonuses},
\end{align*}
while the pessimistic values are used either as part of the bonuses or while bounding them. 
However, when trying to perform a similar decomposition with reward lookahead, we do not have the difference of expected rewards, but rather terms of the form
\begin{align*}
    \E_{\bR\sim\hat{\Rcal}_h^{k-1}(s_h)}\brs*{R(\pi^k_h(s_h,\bR))} - \E_{\bR\sim\Rcal_h(s_h)}\brs*{R(\pi^k_h(s_h,\bR))} 
\end{align*}
(see, e.g., the last term of \Cref{lemma: value-difference reward lookahead} in the appendix). As the action can be an arbitrary function of the reward realization, this term is extremely challenging to bound. For example, one could couple both distributions while trying to relate this error term to a Wasserstein distance between the empirical and real reward distribution; however, such distances exhibit much slower error rates than standard mean estimation \citep{fournier2015rate}. Instead, we follow a different approach and show that uniformly for all possible expected next-state values $\hat{P}V\in[0,H]^A$ (as a function of the action at a given state), it holds w.h.p. that 
\begin{align}
    &\abs*{\E_{\bR\sim\hat{\Rcal}_h^{k-1}(s)}\brs*{\max_a\brc*{R(a)+\hat{P}V(s,a)}} - \E_{\bR\sim\Rcal_h(s)}\brs*{\max_a\brc*{R(a)+\hat{P}V(s,a)}} } \nonumber \\
    &\hspace{17em}\lesssim \sqrt{ \frac{A\ln\frac{1}{\delta} }{n^{k-1}_h(s)\vee 1}}. \label{eq: key proof step reward lookahead}
\end{align}
Throughout the proof, whenever we face an expectation w.r.t. the empirical rewards, we reformulate the expression to fit the form of \Cref{eq: key proof step reward lookahead} and use it as a `change of measure' tool. We remark that while this confidence interval admits an extra $A$-factor compared to standard bounds, the counts only depend on the visits to the state (and not to the state-action), which compensates for this factor.

The choice of MVP for the bonus is similarly motivated -- unlike some other bonuses (e.g., \citealt{zanette2019tighter}), MVP does not require pessimistic values -- either in the bonus itself or in its analysis. In contrast to the optimistic ones, the pessimistic values are not calculated via value iteration, but rather by following the policy $\pi^k$ in the pessimistic environment. As such, they cannot be easily manipulated to fit the form in \Cref{eq: key proof step reward lookahead}. 

The analysis of the transitions adapts the techniques in \citep{efroni2021confidence}, while requiring extra care in handling the dependence of actions in the rewards.


\section{Reinforcement Learning with One-Step Transition Lookahead}
\label{section:transition lookahead}
We now move to analyzing problems with one-step transition lookahead, where the resulting next state due to playing any of the actions is revealed before deciding which action to play. For example, consider the stochastic Canadian traveler problem with resampling \citep{nikolova2008route,boutilier2018planning}. In this problem, an agent wants to navigate on a graph as fast as possible from a source to a target, but observes which edges at a node are available only upon reaching this node. 
When edge availability is stochastic and resampled every time a node is visited, this is a clear case of one-step transition lookahead, as the information on the availability of edges is given before trying to traverse them. The example in \Cref{section: comparison to standard RL} and \Cref{appendix: transition lookahead example} is one possible formulation of this problem on a chain -- agents are awarded for arriving at the end of the chain as fast as possible, but trying to use a non-existing edge results with termination. We showed that in this particular instance, the lookahead value is exponentially larger than the standard value, and any lookahead agent with low regret would greatly surpass no-lookahead agents.

As with reward lookahead, the future states for all actions can be embedded into the state, but doing so increases the size of the state space by a factor of $S^A$, again making this approach intractable (see \Cref{appendix: extended MDP transitions} for an example for such an extension). 
We once more show that this is not necessary; the transition-lookahead optimal values can be calculated using the following Bellman equations:
\begin{restatable}{proposition-rst}{transitionDP}
\label{prop: transition DP}
The optimal value of one-step transition lookahead agents satisfies 
\begin{align*}
    &V^{T,*}_{H+1}(s)=0, &\forall s\in\Scal,\\
    &V^{T,*}_h(s) = \E_{\bs'\sim P_h(s)}\brs*{\max_{a\in\Acal}\brc*{r_h(s,a) +V^{T,*}_{h+1}(s'(s,a))}},&\forall s\in\Scal, h\in[H].
\end{align*}
Also, given next-state observations $\bs'=\brc*{s'(a)}_{a\in\Acal}$ at state $s$ and step $h$, the optimal policy is
$$\pi^*_h(s,\bs')\in\argmax_{a\in\Acal}\brc*{r_h(s,a)+V^{T,*}_{h+1}(s'(a))}.$$
\end{restatable}
The proof can be found at \Cref{appendix: extended MDP transitions} and again relies on augmenting the state space to incorporate the transitions; this time, we divide the episode into odd steps whose extended state is $s_h\times \bs'_0$ (for an arbitrary fixed $\bs'_0\in\Scal^A$) and even steps with the state $s_h\times\bs'_{h+1}$. Beyond planning, this again allows proving a variant of the LTV and of a value-difference lemma.

One important insight is that the policy $\pi^*_h(s,\bs')$ admits the form of a \emph{list}. Namely, consider the values  $V_h^*(s,s',a) = r_h(s,a)+V^{T,*}_{h+1}(s')$ and assume some ordering of next-state-action pairs $\brc*{(s'_i,a_i)}_{i=1}^{SA}$ such that $V_h^*(s,s'_1,a_1)\ge\dots\ge V_h^*(s,s'_{SA},a_{SA})$. Then, an optimal policy would look at all realized pairs $(s'(a),a)$ and play the action with the highest location in this list. We refer the readers to \Cref{appendix: transition lookahead list representation} for an additional discussion on list representations in transition lookahead. 

Similar results could be achieved through a reduction to RL problems with stochastic action sets \citep{boutilier2018planning}. There, at every round, a subset of base actions is sampled, and only these actions are available to the agent. In particular, one could sample $A$ actions of the form $(s',a)\in\Scal\times\Acal$ and impose a deterministic transition to $s'$ given this extended action. However, since every original action must be sampled exactly once, this sampling procedure creates a dependence between pairs even when next-states at different actions are independent, adding unnecessary complications. We show that when transitions are independent between states, the expectation in \Cref{prop: transition DP} can be efficiently calculated (see \Cref{appendix: transition lookahead planning} for details), and otherwise, it can be approximated through sampling, as we do in learning settings.


\subsection{Regret-Minimization with Transition Lookahead}
Relying on similar principals as with reward lookahead, we now present MVP-TL, an adaptation of MVP to settings with one-step transition lookahead (summarized in \Cref{alg: MVP transition lookahead short}; the full details can be found at \Cref{appendix: MVP for transition lookahead}). This time, we estimate the empirical expected reward and add a standard Hoeffding-like reward bonus, while performing planning using samples from the \emph{empirical joint distribution} of the next-state for all the actions simultaneously. A variance-based transition bonus is added to the values; though this time, the variance also incorporates the rewards, namely 
{\small\begin{align*}
    b_{k,h}^{p}(s) \approx \sqrt{\frac{\VAR_{\bs'\sim\hat{P}^{k-1}_{h}(s)}(\bar{V}^k_{h}(s,\bs'))}{n^{k-1}_{h}(s)\vee 1}},  
    \quad  \bar{V}^k_h(s,\bs') = \max_{a\in\Acal}\brc*{\hat{r}_h^{k-1}(s,a) + b_{k,h}^{r}(s,a) +\bar{V}^k_{h+1}(s'(a)}.
\end{align*}}
The motivation for this modification is the technical challenges described in \Cref{section: reward lookahead proof sketch}, in the context of reward lookahead. For reward lookahead, we analyzed a value term that included both the rewards and next-state values, and used concentration arguments to move from the empirical reward distribution to the real one. For transition lookahead, similar values are analyzed, but we require variance-based concentration to obtain tighter regret bounds \citep{azar2017minimax}, so this variance naturally arises. The bonus is again designed to be monotone, as in the original MVP algorithm, and does not affect the immediate action choice -- only the optimistic lookahead value. As before, the planning relies on sampling the next-state observations at previous episodes, and so it is polynomial, even if the precise joint distribution is complex.
The algorithm enjoys the following regret bounds: 
\begin{restatable}{theorem-rst}{MVPTL}
\label{theorem: regret MVP-TL}
When running MVP-TL, with probability at least $1-\delta$ uniformly for all $K\ge1$, it holds that $\Regret^T(K)\le \Ocal\br*{\sqrt{H^2SK}\br*{\sqrt{H}+\sqrt{A}}\ln\frac{SAHK}{\delta}+H^3S^4A^3\br*{\ln\frac{SAHK}{\delta}}^2}$.
\end{restatable}
See proof in \Cref{appendix: transition lookahead regret}. For transition lookahead, the regret bounds we provide exhibit two rates, both corresponding to a natural adaptation of known lower bounds to transition lookahead.
\begin{enumerate}[leftmargin=.5cm] 
    \item \emph{`Bandit rate'} $\Ocal(\sqrt{H^2SAK})$: this is the rate due to reward stochasticity. Consider a problem where at odd timesteps $2h-1$ and across all states, all actions have rewards of mean $\nicefrac{1}{2}-\epsilon$, except for one action of mean $\nicefrac{1}{2}$. Assuming that the state-distribution is uniform, each such timestep forms a hard instance of a contextual bandit problem with $S$ contexts, exhibiting a regret of $\Omega(\sqrt{SAK})$ \citep{auer2002nonstochastic,bubeck2012regret}. Since there are $H/2$ odd steps and we can design each step independently, the total regret would be $\Omega(H\sqrt{SAK})$. 
    The even steps can be used to `remove' the lookahead and create a uniform state distribution. To do so, we set that when taking an action at odd steps, we always transition to a fixed state $s_d$. From this state, one action $a_1$ leads uniformly to all states, while the rest of the actions lead to an absorbing non-rewarding state -- rendering them strictly suboptimal. Thus, no-regret agents will only play $a_1$, regardless of the lookahead information, and the state distribution at odd timesteps will be uniform.    
    
    \item \emph{`Transition learning rate'} $\Ocal(\sqrt{H^3SK})$: recall that the vanilla RL lower bound designs a tree with $\Omega(S)$ leaves, to which agents need to navigate at the right timing (with $\Omega(H)$ options) and take the right action (out of $A$). While all leaves might transition agents to a rewarding state, one combination of state-action-timing has a slightly higher probability of doing so \citep{domingues2021episodic}. This roughly creates a bandit problem with $SAH$ arms, constructed such that the maximal reward is $\Omega(H)$, yielding a total regret of $H\sqrt{HSAK}$. Now consider the following simple modification where in each leaf, only one action can lead to a reward (and the rest of the actions are `useless' -- never lead to rewards). Thus, the agent still needs to test all leaves at all timings, and so there are still $SH$ `arms' with a corresponding regret of $\sqrt{H^3SK}$. Moreover, to test a leaf at a certain timing, we must navigate to it, and since the agent is going to play the single useful action at the leaf, transition lookahead does not provide any additional information.
\end{enumerate}
As discussed before, transition lookahead can be formulated as an RL instance with stochastic action sets. While \citet{boutilier2018planning} prove that with stochastic action sets, Q-learning asymptotically converges, they provide no learning algorithm nor regret bounds. Therefore, to our knowledge, our result is the first to achieve sublinear regret with transition lookahead.

\begin{algorithm}[t]
\caption{Monotonic Value Propagation with Transition Lookahead (MVP-TL)} \label{alg: MVP transition lookahead short}
\begin{algorithmic}[1]
\STATE {\bf Require:} $\delta\in(0,1)$, bonuses $b_{k,h}^r(s,a), b_{k,h}^p(s)$
\FOR{$k=1,2,...$}
    \STATE  Initialize $\bar{V}^k_{H+1}(s)=0$ 
    \FOR{$h=H,H-1,..,1$}
        \STATE Calculate the truncated values for all $s\in\Scal$
        \begin{align*}
            &\bar{V}^k_h(s) = \min\brc*{\E_{\bs'\sim\hat{P}^{k-1}_h(s)}\brs*{\max_{a\in\Acal}\brc*{\hat{r}_h^{k-1}(s,a) + b_{k,h}^{r}(s,a) + \bar{V}^k_{h+1}(s'(a))}} + b^p_{k,h}(s), H} 
        \end{align*}
    \ENDFOR
    \FOR{$h=1,2,\dots H$}
        \STATE Observe $s_h^k$ and $s'^k_{h+1}(s_h^k,a)$ for all $a\in\Acal$
        \STATE Play an action 
        $a_h^k\in\argmax_{a\in\Acal}\brc*{\hat{r}_h^{k-1}(s_h^k,a) + b_{k,h}^{r}(s_h^k,a) + \bar{V}^k_{h+1}(s'^k_{h+1}(s_h^k,a))}$
        \STATE Collect the reward $R^k_h\sim\Rcal_h(s_h^k,a_h^k)$ and transition to the next state $s^k_{h+1}=s'^k_{h+1}(s_h^k,a_h^k)$
    \ENDFOR
\ENDFOR
\end{algorithmic}
\end{algorithm}

\subsection{Proof Concepts}
Transition lookahead causes similar issues as reward lookahead. Hence, it is natural to apply a similar analysis approach -- first, formulate the value as the expectation w.r.t. the next-state observations of the maximum of action-observation dependent values; then use uniform concentration as a `change of measure' tool between the empirical and real next-state distribution. In particular, if $V(s,s',a)$ represents the value  starting from state $s$, performing $a$ and transitioning to $s'$, one can show that for all $V(s,\cdot,\cdot)\in[0,H]^{SA}$ (see \Cref{lemma:max concentration transitions}),
\begin{align}
    &\abs*{\E_{\bs'\sim\hat{P}_h^{k-1}(s)}\brs*{\max_aV(s,s'(a),a)} - \E_{\bs'\sim P_h(s)}\brs*{\max_aV(s,s'(a),a)} }\nonumber\\
    &\hspace{14.5em}\lesssim \sqrt{ \frac{SA\ln\frac{1}{\delta}\VAR_{\bs'\sim\hat{P}_h^{k-1}(s)}\max_aV(s,s'(a),a) }{n^{k-1}_h(s)\vee 1}}, \label{eq: key proof step transition lookahead}
\end{align}
where the variance term stems from using a Bernstein-like concentration bound. However, in contrast to the reward lookahead, the $\sqrt{SA}$-factor propagates to the dominant term of the regret, so pursuing this approach would lead to a worse regret bound of $\Olog\br*{\sqrt{H^3S^2AK}}$.

To avoid this, we pinpoint the two locations where this change of measure is needed -- the proof that $\bar{V}^k_h$ is optimistic and the regret decomposition -- and make sure to perform this change of measure only on a single value $V^{*}_h(s,s',a) = r_h(s,a) +V^{*}_{h+1}(s')$, mitigating the need to cover all possible values and removing the additional $\sqrt{SA}$-factor. However, doing so leaves us with a residual term. Defining $V^*_h(s,\bs') = \max_{a\in\Acal}\brc*{V^*_h(s,s'(a),a)}$ and assuming a similar optimistic value $\bar{V}^k_h(s,\bs')$, this term is of the form
\begin{align*}
    \E_{\bs'\sim \hat{P}^{k-1}_h(s)}\brs*{\bar{V}^k_{h}(s,\bs') - V^*_{h}(s,\bs')} - \E_{\bs'\sim P_h(s)}\brs*{\bar{V}^k_{h}(s,\bs') - V^*_{h}(s,\bs')}.
\end{align*}
While similar terms have been analyzed before \citep[e.g.,][]{zanette2019tighter,efroni2021confidence}, the analysis leads to a constant regret term that depends on the support of the distribution in question; in our case, it is the distribution over all possible next-states -- of cardinality $S^A$. Therefore, following the same derivation would lead to an exponential additive regret term.

We overcome it by utilizing the fact that both the optimistic policy and the optimal one decide which action to take according to a list of next-state-actions $(s',a)$. In other words, instead of looking at the next-state $\bs'$ (with $S^A$ possible values) to determine a value, we look at the highest-ranked realized pair $(s',a)$ in the list that corresponds to the policy that induces the value (with $SA$ possible rankings). Since we have two values, we need to calculate the probability of being at a certain list location for both $\pi^k$ and $\pi^*$, but the cardinality of this space is $(SA)^2$: polynomial and not exponential.

\section{Conclusions and Future Work}
\label{section: conclusions}
In this work, we presented an RL setting in which immediate rewards or transitions are observed before actions are chosen. We showed how to design provably and computationally efficient algorithms for this setting that achieve tight regret bounds versus a strong baseline that also uses lookahead information. Our algorithms rely on estimating the distribution of the reward or transition observations, a concept that might be utilized in other settings. In particular, we believe that our techniques for transition lookahead could be extended to RL problems with stochastic action sets \citep{boutilier2018planning}, but leave this for future work.

One natural extension to our work would be to consider multi-step lookahead information -- observing the transition/rewards $L$ steps in advance. We conjecture that from a statistical point of view, a similar algorithmic approach that samples from the empirical observation distribution would be efficient. However, it is not clear how to perform efficient planning with such feedback.

Another possible direction would be to derive model-free algorithms \citep{jin2018q}, with the aim to improve the computation efficiency of the solutions; our model-based algorithms require at most $\Ocal(KS^2AH)$ computations per episode due to the planning stage, while model-free algorithms might potentially allow just $\Ocal(AH)$ computations per episode.

On the practical side, previous works presented RL algorithms that utilize/estimate a world model with multi-step lookahead to perform planning and learning \citep{schrittwieser2020mastering, chung2024thinker}, aiming to achieve the optimal no-lookahead value. For some of these approaches, it is quite natural to replace the simulated world behavior with lookahead information on the real future realization. We leave this adaptation and evaluation to future studies.

Finally, the notion of lookahead could be studied in various other decision-making settings (e.g., linear MDPs \citealt{jin2020provably}) and can also be generalized to situations where lookahead information can be queried under some budget constraints \citep{efroni2021confidence} or when agents only observe noisy lookahead predictions; we leave these problems for future research.

\section*{Acknowledgements}
We thank Alon Cohen and Austin Stromme for the helpful discussions. This project has received funding from the European Union’s Horizon 2020 research and innovation programme under the Marie Skłodowska-Curie grant agreement No 101034255. 

\bibliographystyle{plainnat}
\bibliography{references}


\clearpage

\addcontentsline{toc}{section}{Appendix}
\appendix
\part{}


{\hypersetup{linkcolor=black}
\parttoc}

\section{Structure of the Appendix}
Both reward and transition lookahead appendices share the following structure. First, we describe our assumption on the data generation process and analyze general properties of reward and transition lookahead. This is done by looking at an extended MDP that incorporates the lookahead information into the state. Then, we present the full algorithm and describe the relevant probabilistic events that ensure the concentration of all the empirical quantities. For transition lookahead, we require some additional notions for the event definitions (including the list representation of values and policies), which are explained in a separate subsection.

Given the concentration-related good event, we can prove that the planning procedure in the algorithm is optimistic, which we do in the subsequent subsection. Then, we define an additional good event that allows adding and removing conditional expectations in a way that will be needed for the proof.

At this point, we provided all (almost all) the results required for the regret analysis, and the proof of the main theorems is stated. The proofs also require some additional analysis for the bonuses (and especially variance terms), which is located at the end of the regret analysis.

For transition lookahead, the appendix includes one more part that further analyzes the example presented in \Cref{section: comparison to standard RL}.

At the end of the appendix, we state and prove several lemmas that will be used throughout our analysis, while also stating several existing results that will be of use.
\newpage

\section{Proofs for Reward Lookahead}
\label{appendix: reward lookahead proofs}
\subsection{Data Generation Process}
\label{appendix: reward lookahead data generation}
To simplify the proofs, we assume the following 'tabular' data-generation process: Before the game starts, a set of $K$ samples from the transition probabilities and rewards is generated for all $(s,a,h)$. Once a state $s$ at step $h$ is visited for the $i^{th}$ time, the $i^{th}$ sample from the reward distribution $\Rcal_h(s)$ is the reward realization for all action $a\in\Acal$. When a state-action pair is visited for the $i^{th}$ time, the $i^{th}$ sample from the transition kernel $P_h(\cdot\vert s,a)$ determines the next-state realization. In particular, it implies that the reward samples from the first $i$ visits to a state are i.i.d., and the same for the next-states samples and state-action visitations. Throughout this appendix, we use the notation $\bR_h^k = \brc*{R_h^k(s_h^k,a}_{a\in\Acal}$ to denote the reward observation at episode $k$ and timestep $h$ for all the actions.

For the proof, we define the following three filtrations. Let
\begin{align*}
    F_{k,h}&=\sigma\br*{\brc*{s_t^1,a_t^1,\bR_t^1}_{t\in[H]},\dots, \brc*{s_t^{k-1},a_t^{k-1},\bR_t^{k-1}}_{t\in[H]}, \brc*{s_t^{k},a_t^{k},\bR_t^k}_{t\in[h]},s_{h+1}^k},\\
    F_{k,h}^R&=\sigma\br*{\brc*{s_t^1,a_t^1,\bR_t^1}_{t\in[H]},\dots, \brc*{s_t^{k-1},a_t^{k-1},\bR_t^{k-1}}_{t\in[H]}, \brc*{s_t^{k},a_t^{k},\bR_t^k}_{t\in[h+1]}},
\end{align*}
the filtrations that contains all information until episode $k$ and step $h$, as well as the state at timestep $h+1$, or all information of time $h+1$, respectively. We make this distinction so that $F_{k,h-1}$ contains only $s_h^k$, while $F^R_{k,h-1}$ also contains $a_h^k$. We also define
\begin{align*}
    F_{k}&=\sigma\br*{\brc*{s_t^1,a_t^1,\bR_t^1}_{t\in[H]},\dots, \brc*{s_t^{k},a_t^{k},\bR_t^k}_{t\in[H]},s_{1}^{k+1}},
    \end{align*}
which contains all information up to the end of the $k^{th}$ episode, as well as the initial state at episode $k+1$.

\subsection{Extended MDP for Reward Lookahead}
\label{appendix: extended MDP rewards}
In this appendix, we present an alternative formulation of the one-step reward lookahead that falls under the vanilla (no-lookahead) model and would be helpful for the analysis. 

Throughout the section, we study the relations between MDPs with and without reward lookahead, and between different MDPs with lookahead. Therefore, for clarity, we state the concerning MDP in the value, e.g. $V^{R,\pi}(s\vert\Mcal)$. Specifically in this subsection, we distinguish between values without lookahead (denoted $V^{\pi}$) and values with lookahead (denoted $V^{R,\pi}$). In the following subsections, unless stated otherwise, we will only consider lookahead values; for brevity, and with some abuse of notations, we will then omit the $R$ in the value notation.

For any MDP $\Mcal=(\Scal,\Acal,H,P,\Rcal)$, define an equivalent extended MDP $\Mcal^R$ of horizon $2H$ that separates the state transition and reward generation as follows:
\begin{enumerate}
    \item Assume w.l.o.g. that $\Mcal$ starts at some initial state $s_1$. The extended environment starts at a state $s_1\times \bzero$, where $\bzero\in\R^{A}$ is the zeros vector.
    \item For any $h\in[H]$, at timestep $2h-1$, the environment $\Mcal^R$ transitions from state $s_h\times \bzero$ to $s_h\times \bR$, where $\bR\sim \Rcal_h(s)$ is a vector containing the rewards for all actions $a\in\Acal$. This transition occurs regardless of the action that was played. At timestep $2h$, given an action $a_h$ the environment transitions from $s_h\times \bR$ to $s_{h+1}\times \bzero$, where $s_{h+1}\sim P_h(\cdot\vert s_h,a_h)$.
    \item The reward at a state $s\times \bR$ when playing an action $a$ is $R(a)$, namely, the reward is deterministic and only obtained on even timesteps.
\end{enumerate}
We emphasize that throughout the section, we assume that $\Mcal$ and $\Mcal^R$ are coupled; that is, assume that under a policy $\pi$ in $\Mcal$, the agent visits a state $s_h$, observes $\bR_{h}$, plays an action $a_h$ and transitions to $s_{h+1}$. Then, in $\Mcal^R$, the agent starts from $s_h\times \bzero$, transitions to $s_h\times \bR$ (regardless of the action it played), takes the action $a_h$ and finally transitions to $s_{h+1}\times \bzero$.

Since the reward is embedded into the state, any state-dependent policy in $\Mcal^R$ is a one-step reward lookahead policy in the original MDP. Moreover, the policy at the odd steps of $\Mcal$ does not affect the value, and assuming that the policy at the even steps in $\Mcal^R$ is the same as the policy in $\Mcal$, we trivially get the following relation between the values
\begin{align}   
    \label{eq: extended reward MDP to standard value relation}
    & V_{2h}^\pi(s,\bR\vert \Mcal^R) = \E\brs*{\sum_{t=h}^H R_t(s_t,a_t)\vert s_h=s,R_h(s,\cdot)=\bR,\pi} \triangleq V_h^{R,\pi}(s,\bR\vert \Mcal), \nonumber\\
    & V_{2h-1}^\pi(s,\bzero\vert \Mcal^R) = \E\brs*{\sum_{t=h}^H R_t(s_t,a_t)\vert s_h=s,\pi} = V_h^{R,\pi}(s\vert \Mcal) .
\end{align}
While $\Mcal^R$ has a continuous state space, which generally makes algorithm design impractical, this representation permits applying classic results on MDPs to environments with one-step lookahead. 

As a remark, rewards could be directly embedded into the state without separating the state and reward updates. However, this creates unnecessary complications when analyzing the relations between similar environments. This is because we are mainly interested in the value given the state -- in expectation over the realized rewards. In particular, value-difference are analyzed assuming a shared initial state, but in our case, we do not want to assume the same reward realization, but rather also account for the distance between reward distributions, which the step separation enables. For similar reasons, this representation also simplifies the proof of the law of total variance \citep{azar2017minimax}.

\rewardDP*

\begin{proof}
    We prove the result in the extended MDP $\Mcal^R$ and remind the reader that in this formulation, the policy only uses state information, as in the standard RL formulation. In particular, it implies that there exists a Markovian optimal policy that uniformly maximizes the value (in the extended state space), and the optimal value is given through the dynamic-programming equations \citep{puterman2014markov}
    \begin{align}
        \label{eq: Bellman extended reward MDP}
        &V_{2H+1}^*(s,\bR\vert \Mcal^R)=0, &\forall s\in\Scal, \bR\in\R^{A},\nonumber\\
        &V_{2h}^*(s,\bR\vert \Mcal^R) = \max_a\brc*{R(a)+\sum_{s'\in\Scal}P_h(s'\vert s,a)V_{2h+1}^*(s',\bzero\vert \Mcal^R)}, &\forall h\in[H],s\in\Scal, \bR\in\R^{A},\nonumber\\
        &V_{2h-1}^*(s,\bzero\vert \Mcal^R) = \E_{\Rcal_h(s)}\brs*{V_{2h}^*(s,\bR\vert \Mcal^R)}, &\forall h\in[H],s\in\Scal.
    \end{align}
    By the equivalence between $\Mcal$ and $\Mcal^R$ for all policies, this is also the optimal value in $\Mcal$. Specifically, combining both recursion equations and substituting the relation between the original and extended values of \Cref{eq: extended reward MDP to standard value relation}, we get the desired value recursion for any $h\in[H]$ and $s\in\Scal$:
    \begin{align*}
        V^{R,*}_h(s\vert\Mcal) 
        &= V_{2h-1}^*(s,\bzero\vert \Mcal^R) \\
        & = \E_{\Rcal_h(s)}\brs*{V_{2h}^*(s,\bR\vert \Mcal^R)} \\
        & = \E_{\Rcal_h(s)}\brs*{\max_a\brc*{R(a)+\sum_{s'\in\Scal}P_h(s'\vert s,a)V_{2h+1}^*(s',\bzero\vert \Mcal^R)}} \\
        & = \E_{\Rcal_h(s)}\brs*{\max_a\brc*{R(a)+\sum_{s'\in\Scal}P_h(s'\vert s,a)V^{R,*}_{h+1}(s\vert\Mcal)}}.
    \end{align*}
    Similarly, for any $h\in[H]$, $s\in\Scal$ and $\bR\in\R^A$, the optimal policy at the even stages of the extended MDP is 
    \begin{align*}
        \pi^*_{2h}(s,\bR)\in\argmax_{a\in\Acal}\brc*{R(a)+\sum_{s'\in\Scal}P_h(s'\vert s,a)V_{2h+1}^*(s',\bzero\vert \Mcal^R)},
    \end{align*}
    alongside arbitrary actions at odd steps. Playing this policy in the original MDP will lead to an optimal one-step reward lookahead policy, as it achieves the optimal value of the original MDP. This policy directly translates to the optimal policy in the statement, by the equivalence between the original and extended MDPs and the relation $V_{2h+1}^*(s',\bzero\vert \Mcal^R)=V_{h+1}^{R,*}(s'\vert \Mcal)$.
\end{proof}
\begin{remark}
    \label{remark: reward value with reward observations}
    As in \Cref{eq: Bellman extended reward MDP}, one could also write the dynamic programming equations for any policy $\pi\in\Pi^R$, namely 
    \begin{align*}
        &V_{2h}^\pi(s,\bR\vert \Mcal^R) = R(\pi_h(s,\bR))+\sum_{s'\in\Scal}P_h(s'\vert s,\pi_h(s,\bR))V_{2h+1}^\pi(s',\bzero\vert \Mcal^R), &\forall h\in[H],s\in\Scal, \bR\in\R^{A},\nonumber\\
        &V_{2h-1}^\pi(s,\bzero\vert \Mcal^R) = \E_{\Rcal_h(s)}\brs*{V_{2h}^\pi(s,\bR\vert \Mcal^R)}, &\forall h\in[H],s\in\Scal.
    \end{align*}
    In particular, following the notation of \Cref{eq: extended reward MDP to standard value relation}, one can also write
    \begin{align*}
        &V_h^{R,\pi}(s,\bR\vert \Mcal) = R(\pi_h(s,\bR))+\sum_{s'\in\Scal}P_h(s'\vert s,\pi_h(s,\bR))V_{h+1}^{R,\pi}(s'\vert \Mcal), \quad \textrm{and,}\\
        & V_h^{R,\pi}(s\vert \Mcal) = \E_{\Rcal_h(s)}\brs*{V_{h}^{R,\pi}(s,\bR\vert \Mcal)} \\
        &\hspace{5em}=\E_{\Rcal_h(s)}\brs*{R(\pi_h(s,\bR))+\sum_{s'\in\Scal}P_h(s'\vert s,\pi_h(s,\bR))V_{h+1}^{R,\pi}(s'\vert \Mcal))}.
    \end{align*}
    We will use this notation in some of the proofs.
\end{remark}

Another useful application of the extended MDP is a variation of the law of total variance (LTV), which will be useful in our analysis

\begin{lemma}
    \label{lemma: ltv reward-lookahead} 
    For any deterministic one-step reward lookahead policy $\pi\in\Pi^R$, it holds that
\begin{align*}
    \E\brs*{\sum_{h=1}^H\VAR_{P_h(\cdot\vert s_h,a_h)}(V^{R,\pi}_{h+1}(s_{h+1}) )\vert\pi,s_1} \leq \E\brs*{\br*{\sum_{h=1}^H R_h(s_h,a_h) - V_1^{R,\pi}(s_1) }^2\vert\pi,s_1}.
\end{align*}
\end{lemma}
\begin{proof}
    We apply the law of total variance (\Cref{lemma: ltv no-lookahead}) in the extended MDP; there, the rewards are deterministic and equal to either $0$ (at odd steps) or $R_h(s_h,a_h)$ (at even steps), so the total expected rewards are $\sum_{h=1}^HR_h(s_h,a_h)$.
    \begin{align*}
        \E&\brs*{\br*{\sum_{h=1}^H R_h(s_h,a_h) - V_{1}^\pi(s_1,\bzero\vert \Mcal^R) }^2\vert\pi,s_1} \\
        &= \E\brs*{\underbrace{\sum_{h=1}^H\VAR(V^\pi_{2h}(s_{h},\bR_h(s_{h})\vert \Mcal^R)\vert (s_h,\bzero) )}_{\textrm{Odd steps}}  + \underbrace{\sum_{h=1}^H\VAR(V^\pi_{2h+1}(s_{h+1},\bzero\vert \Mcal^R)\vert (s_h,\bR_h(s_{h})) )}_{\textrm{Even steps}}\vert\pi,s_1} \\
        & \geq \E\brs*{\sum_{h=1}^H\VAR(V^\pi_{2h+1}(s_{h+1},\bzero\vert \Mcal^R)\vert (s_h,\bR_h(s_{h})) )\vert\pi,s_1} \\
        & = \E\brs*{\sum_{h=1}^H\VAR_{P_h(\cdot\vert s_h,a_h)}(V^\pi_{2h+1}(s_{h+1},\bzero\vert \Mcal^R) )\vert\pi,s_1} \\
        & = \E\brs*{\sum_{h=1}^H\VAR_{P_h(\cdot\vert s_h,a_h)}(V^{R,\pi}_{h+1}(s_{h+1}\vert\Mcal) )\vert\pi,s_1}.
    \end{align*}
    Noting that $V_{1}^\pi(s_1,\bzero\vert \Mcal^R)=V_1^{R,\pi}(s_1\vert\Mcal)$ concludes the proof.
\end{proof}


Finally, though not needed in our analysis, we use the extended MDP to prove the following value-difference lemma, which could be of further use in follow-up works. While we prove decomposition just using the next-step values, one could recursively apply the formula until the end of the episode to immediately get another formula that does not depend on the next value.
\begin{lemma}[Value-Difference Lemma with Reward Lookahead]
\label{lemma: value-difference reward lookahead}
    Let $\Mcal_1=(\Scal,\Acal,H,P^1,\Rcal^1)$ and $\Mcal_2=(\Scal,\Acal,H,P^2,\Rcal^2)$ be two environments. For any deterministic one-step reward lookahead policy $\pi\in\Pi^R$, any $h\in[H]$ and $s\in\Scal$, it holds that
    \begin{align*}
        V_h^{R,\pi}&(s\vert \Mcal_1) -  V_h^{R,\pi}(s\vert \Mcal_2) \\
        & = \E_{\Mcal_1}\brs*{ V_{h+1}^{R,\pi}(s_{h+1}\vert \Mcal_1) - V_{h+1}^{R,\pi}(s_{h+1}\vert \Mcal_2)\vert s_h=s}  \\
        &\quad+ \E_{\Mcal_1}\brs*{ \sum_{s'\in\Scal}\br*{P_h^1(s'\vert s_h,\pi_h(s_h,\bR_h)) - P_h^2(s'\vert s_h,\pi_h(s_h,\bR_h))}V_{h+1}^{R,\pi}(s'\vert \Mcal_2) \vert s_h=s} \\
        & \quad + \E_{\Mcal_1}\brs*{ \E_{\Rcal^1_h(s)}\brs*{V_{h}^{R,\pi}(s_h,\bR\vert \Mcal_2)} - \E_{\Rcal^2_h(s)}\brs*{V_{h}^{R,\pi}(s_h,\bR\vert \Mcal_2)} \vert s_h=s} ,
    \end{align*}
    where $V_h^{R,\pi}(s,\bR\vert \Mcal) $ is the value at a state given the reward realization, defined in \Cref{eq: extended reward MDP to standard value relation} and given in \Cref{remark: reward value with reward observations}.
\end{lemma}
\begin{proof}
    We again work with the extended MDPs $\Mcal_1^R,
    \Mcal_2^R$. Since under the extension, both the environments and the policy are Markovian, all values obey the following Bellman equations:
    \begin{align*}
        &V_{2h}^{\pi}(s,\bR\vert \Mcal^R) = R(\pi_h(s,\bR))+\sum_{s'\in\Scal}P_h(s'\vert s,\pi(s,\bR))V_{2h+1}^{\pi}(s',\bzero\vert \Mcal^R), &\forall h\in[H],s\in\Scal, \bR\in\R^{A}\nonumber\\
        &V_{2h-1}^{\pi}(s,\bzero\vert \Mcal^R) = \E_{\Rcal_h(s)}\brs*{V_{2h}^{\pi}(s,\bR\vert \Mcal^R)}, &\forall h\in[H],s\in\Scal.
    \end{align*}
    Using the relation between the value of the original and extended MDP (\cref{eq: extended reward MDP to standard value relation}) and the Bellman equations of the extended MDP, for any $h\in[H]$, we have
    \begin{align}
    \label{eq: reward value diff basic decomp}
        &V_h^{R,\pi}(s\vert \Mcal_1) -  V_h^{R,\pi}(s\vert \Mcal_2) \nonumber\\
        & = V_{2h-1}^{\pi}(s,\bzero\vert \Mcal^R_1) -  V_{2h-1}^{\pi}(s,\bzero\vert \Mcal^R_2)  \nonumber\\
        & = \E_{\Rcal^1_h(s)}\brs*{V_{2h}^{\pi}(s,\bR\vert \Mcal^R_1)} - \E_{\Rcal^2_h(s)}\brs*{V_{2h}^{\pi}(s,\bR\vert \Mcal^R_2)}  \nonumber\\
        & = \E_{\Rcal^1_h(s)}\brs*{V_{2h}^{\pi}(s,\bR\vert \Mcal^R_1) - V_{2h}^{\pi}(s,\bR\vert \Mcal^R_2)} + \E_{\Rcal^1_h(s)}\brs*{V_{2h}^{\pi}(s,\bR\vert \Mcal^R_2)} - \E_{\Rcal^2_h(s)}\brs*{V_{2h}^{\pi}(s,\bR\vert \Mcal^R_2)}  \nonumber\\
        & = \E_{\Rcal^1_h(s)}\brs*{V_{2h}^{\pi}(s,\bR\vert \Mcal^R_1) - V_{2h}^{\pi}(s,\bR\vert \Mcal^R_2)} + \E_{\Rcal^1_h(s)}\brs*{V_{h}^{R,\pi}(s,\bR\vert \Mcal_2)} - \E_{\Rcal^2_h(s)}\brs*{V_{h}^{R,\pi}(s,\bR\vert \Mcal_2)} \nonumber\\
        & = \E_{\Mcal_1}\brs*{V_{2h}^{\pi}(s_h,\bR_h\vert \Mcal^R_1) - V_{2h}^{\pi}(s_h,\bR_h\vert \Mcal^R_2)\vert s_h=s} \nonumber\\
        &\quad+ \E_{\Rcal^1_h(s)}\brs*{V_{h}^{R,\pi}(s,\bR\vert \Mcal_2)} - \E_{\Rcal^2_h(s)}\brs*{V_{h}^{R,\pi}(s,\bR\vert \Mcal_2)} .
    \end{align}
    We now focus on the first term. Denoting $a_h=\pi_h(s_h,\bR_h)$ the action taken by the agent at environment $\Mcal_1$, We have
    \begin{align*}
        V_{2h}^{\pi}(s_h,\bR_h\vert \Mcal^R_1) &- V_{2h}^{\pi}(s_h,\bR_h\vert \Mcal^R_2) \\
        & = \br*{R_h(a_h)+\sum_{s'\in\Scal}P_h^1(s'\vert s_h,a_h)V_{2h+1}^{\pi}(s',\bzero\vert \Mcal^R_1)} \\
        &\quad- \br*{R_h(a_h)+\sum_{s'\in\Scal}P_h^2(s'\vert s_h,a_h)V_{2h+1}^{\pi}(s',\bzero\vert \Mcal^R_2)} \\
        & = \sum_{s'\in\Scal}P_h^1(s'\vert s_h,a_h)V_{h+1}^{R,\pi}(s'\vert \Mcal_1) - \sum_{s'\in\Scal}P_h^2(s'\vert s_h,a_h)V_{h+1}^{R,\pi}(s'\vert \Mcal_2) \\
        & = \sum_{s'\in\Scal}P_h^1(s'\vert s_h,a_h)\br*{V_{h+1}^{R,\pi}(s'\vert \Mcal_1) - V_{h+1}^{R,\pi}(s'\vert \Mcal_2)} \\
        &\quad+ \sum_{s'\in\Scal}\br*{P_h^1(s'\vert s_h,a_h) - P_h^2(s'\vert s_h,a_h)}V_{h+1}^{R,\pi}(s'\vert \Mcal_2) \\
        & = E_{\Mcal_1}\brs*{V_{h+1}^{R,\pi}(s_{h+1}\vert \Mcal_1) - V_{h+1}^{R,\pi}(s_{h+1}\vert \Mcal_2)\vert s_h,a_h} \\
        &\quad+ \sum_{s'\in\Scal}\br*{P_h^1(s'\vert s_h,a_h) - P_h^2(s'\vert s_h,a_h)}V_{h+1}^{R,\pi}(s'\vert \Mcal_2) .
    \end{align*}
    Substituting this back into \Cref{eq: reward value diff basic decomp}, we have 
    \begin{align*}
        V_h^{\pi}&(s\vert \Mcal_1) -  V_h^{\pi}(s\vert \Mcal_2) \\
        & = \E_{\Mcal_1}\brs*{ E_{\Mcal_1}\brs*{V_{h+1}^{R,\pi}(s_{h+1}\vert \Mcal_1) - V_{h+1}^{R,\pi}(s_{h+1}\vert \Mcal_2)\vert s_h,a_h}\vert s_h=s}\\
        &\quad+ \E_{\Mcal_1}\brs*{\sum_{s'\in\Scal}\br*{P_h^1(s'\vert s_h,a_h) - P_h^2(s'\vert s_h,a_h)}V_{h+1}^{R,\pi}(s'\vert \Mcal_2) \vert s_h=s} \\
        & \quad + \E_{\Rcal^1_h(s)}\brs*{V_{h}^{R,\pi}(s,\bR\vert \Mcal_2)} - \E_{\Rcal^2_h(s)}\brs*{V_{h}^{R,\pi}(s,\bR\vert \Mcal_2)} \\
        & = \E_{\Mcal_1}\brs*{ V_{h+1}^{R,\pi}(s_{h+1}\vert \Mcal_1) - V_{h+1}^{R,\pi}(s_{h+1}\vert \Mcal_2)\vert s_h=s}  \\
        &\quad+ \E_{\Mcal_1}\brs*{ \sum_{s'\in\Scal}\br*{P_h^1(s'\vert s_h,\pi_h(s_h,\bR_h)) - P_h^2(s'\vert s_h,\pi_h(s_h,\bR_h))}V_{h+1}^{R,\pi}(s'\vert \Mcal_2) \vert s_h=s} \\
        & \quad + \E_{\Mcal_1}\brs*{ \E_{\Rcal^1_h(s)}\brs*{V_{h}^{R,\pi}(s_h,\bR\vert \Mcal_2)} - \E_{\Rcal^2_h(s)}\brs*{V_{h}^{R,\pi}(s_h,\bR\vert \Mcal_2)} \vert s_h=s} .
    \end{align*}   
\end{proof}

\clearpage


\subsection{Full Algorithm Description for Reward Lookahead}
\label{appendix: MVP for reward lookahead}

\begin{algorithm}[ht]
\caption{Monotonic Value Propagation with Reward Lookahead (MVP-RL)} \label{alg: MVP reward lookahead}
\begin{algorithmic}[1]
\STATE {\bf Require:} $\delta\in(0,1)$, bonuses $b_{k,h}^r(s), b_{k,h}^p(s,a)$
\FOR{$k=1,2,...$}
    \STATE  Initialize $\bar{V}^k_{H+1}(s)=0$ 
    \FOR{$h=H,H-1,..,1$}   
         \FOR{$s\in\Scal$}
            \IF{$n_h^{k-1}(s)=0$}
                \STATE $\bar{V}^k_h(s) = H$
            \ELSE
                \STATE Calculate the truncated values \label{algline:VI reward lookahead}
                {\small\begin{align*}
                    &\bar{V}^k_h(s) = \min\brc*{\frac{1}{n_h^{k-1}(s)} \sum_{t=1}^{n_h^{k-1}(s)}\max_{a\in\Acal}\brc*{R_h^{k_h^t(s)}(s,a) + b_{k,h}^{p}(s,a) + \hat{P}^{k-1}_{h}\bar{V}^k_{h+1}(s,a)} + b^r_{k,h}(s), H} 
                \end{align*}}
            \ENDIF
            \STATE For any vector $\bR\in\R^A$, define the policy $\pi^k$ 
            \begin{align*}
                \pi_h^k(s,\bR) \in\argmax_{a\in\Acal}\brc*{R(a) + b_{k,h}^{p}(s,a) + \hat{P}^{k-1}_{h}\bar{V}^k_{h+1}(s,a)}
            \end{align*}
        \ENDFOR
    \ENDFOR
    \FOR{$h=1,2,\dots H$}
        \STATE Observe $s_h^k$ and $\bR_h^k = \brc*{R_h^k(s_h^k,a)}_{a\in\Acal}$
        \STATE Play an action 
        $a_h^k=\pi_h^k(s_h^k,\bR_h^k)$
        \STATE Collect the reward $R^k_h(s^k_h,a^k_h)$ and transition to the next state $s^k_{h+1}\sim P_h(\cdot\vert s_h^k,a_h^k)$
    \ENDFOR
    \STATE Update the empirical estimators and counts for all visited state-actions
\ENDFOR
\end{algorithmic}
\end{algorithm}

We use a variant of the MVP algorithm \citep{zhang2021reinforcement} while adapting their proof and the one from \citep{efroni2021confidence}. The algorithm is described in \Cref{alg: MVP reward lookahead} and uses the following bonuses:
\begin{align*}
    &b_{k,h}^r(s) = 3\sqrt{ \frac{AL^k_{\delta} }{2(n^{k-1}_{h}(s)\vee 1)}},\\
    &b_{k,h}^{p}(s,a) = \min\brc*{\frac{20}{3}\sqrt{\frac{\VAR_{\hat{P}^{k-1}_{h}(\cdot| s,a)}(\bar{V}^k_{h+1}) L^k_{\delta}}{n^{k-1}_{h}(s,a)\vee 1}} + \frac{400}{9}\frac{H L^k_{\delta}}{n^{k-1}_{h}(s,a) \vee 1}, H}  
\end{align*}
where $L^k_{\delta}=\ln \frac{144S^2AH^2 k^3(k+1)}{\delta}$, and for brevity, we shorten $\VAR_{\hat{P}^{k-1}_{h}(\cdot| s,a)}(\bar{V}^k_{h+1}(s'))$ to $\VAR_{\hat{P}^{k-1}_{h}(\cdot| s,a)}(\bar{V}^k_{h+1})$ (omitting the state from the value). 

For the optimistic value iteration, we use the notation $k_h^t(s)$ to represent the $t^{th}$ episode where the state $s$ was visited at the $h^{th}$ timestep. Thus, line \ref{algline:VI reward lookahead} of \Cref{alg: MVP reward lookahead} is the expectation w.r.t. the empirical reward distribution $\hat{\Rcal}^{k-1}_h(s)$ (when defining its realization to be zero when $n^{k-1}_{h}(s)=0$). Since the bonuses are larger than $H$ when $n^{k-1}_{h}(s)=0$, one could write the update in more concisely as
\begin{align*}
    &\bar{V}^k_h(s) = \min\brc*{\E_{\bR\sim\hat{\Rcal}^{k-1}_h(s)}\brs*{\max_{a\in\Acal}\brc*{R(a) + b_{k,h}^{p}(s,a) + \hat{P}^{k-1}_{h}\bar{V}^k_{h+1}(s,a)}} + b^r_{k,h}(s), H}.
\end{align*}
We will often use this representation in our analysis.
\clearpage


\subsection{The First Good Event -- Concentration}
We now define the first good event, which ensures that all empirical quantities are well-concentrated. For the transitions, we require each element to concentrate well, as well as both the inner product and the variance w.r.t. the optimal value function. For the reward, we make sure that the maximum of the rewards to concentrate well  (with any possible bias, that will later correspond with the next-state values). Formally, for any fixed vector $u\in\R^A$, denote
\begin{align*}
    & m_h(s,u) = \E_{\bR\sim\Rcal_h(s)}\brs*{\max_a\brc*{R_h(a)+u(a)}}, \\
    & \hat{m}_h^{k}(s,u)=\E_{\bR\sim\hat{\Rcal}^k_h(s)}\brs*{\max_a\brc*{R_h(a)+u(a)}}
\end{align*}
with the convention that $\hat{m}_h^{k}(s,u)=\max_au(a)$ if $n^{k}_h(s)=0$. We define the following good events:
\begin{align*}
    &E^p(k) = \brc*{\forall s,s',a,h:\ |P_h\br*{s'|s,a} - \hat{P}^{k-1}_h\br*{s'|s,a}| \le \sqrt{\frac{2P(s'|s,a)L^k_{\delta}}{n^{k-1}_h(s,a)\vee 1}} + \frac{L^k_{\delta}}{n^{k-1}_h(s,a)\vee 1}} \\
    &E^{pv1}(k)=\brc*{\forall s,a,h:\ \abs*{\br*{\hat{P}^{k-1}_h-P_h} V_{h+1}^*(s,a)} \leq \sqrt{\frac{2\VAR_{P_h(\cdot| s,a)}(V^*_{h+1})L^k_{\delta}}{n^{k-1}_h(s,a)\vee 1}} + \frac{HL^k_{\delta}}{n^{k-1}_h(s,a)\vee 1}}\\
    &E^{pv2}(k)=\brc*{\forall s,a,h:\ \abs*{ \sqrt{\VAR_{P_h(\cdot| s,a )}(V_{h+1}^*)} -  \sqrt{\VAR_{\hat{P}^{k-1}_h(\cdot| s,a )}(V_{h+1}^*)} } \leq 4H\sqrt{\frac{L^k_{\delta}}{n^{k-1}_h(s,a)\vee 1}} }\\
    &E^r(k) = \brc*{\forall s,h,\forall  u\in[0,2H]^A:\ \abs*{m_h(s,u) - \hat{m}_h^{k-1}(s,u) } \leq 3\sqrt{ \frac{AL^k_{\delta} }{2(n^{k-1}_h(s)\vee 1)}}} 
\end{align*}
where we again use $L^k_{\delta}=\ln \frac{144S^2AH^2 k^3(k+1)}{\delta}$. 
Then, we define the first good event as 
$$\G_1 = \bigcap_{k\geq 1} E^r(k) \bigcap_{k\geq 1} E^p(k) \bigcap_{k\geq 1} E^{pv1}(k) \bigcap_{k\geq 1} E^{pv2}(k),$$
for which, the following holds:
\begin{lemma}[The First Good Event]\label{lemma: the first good event reward lookahead}
The good event $\G_1$ holds w.p. $\Pr(\G_1)\geq 1-\delta/2$.
\end{lemma}
\begin{proof}  
    The proof of the first three events uses standard concentration arguments (see, e.g., \citealt{efroni2021confidence}) and is stated for completeness.     
    For any fixed $k\ge1, s,a,h$ and number of visits $n\in[k]$, we utilize \Cref{lemma: concentration results} w.r.t. the transition kernel $P_h(\cdot\vert s,a)$, the value $V^*_{h+1}\in[0,H]$ and probability $\delta'=\frac{\delta}{8SAHk^2(k+1)}$; notice that by the assumption that samples are generated i.i.d. before the game starts, given the number of visits, all samples are i.i.d., so standard concentration could be applied. By taking the union bound over all $n\in[k]$ and slightly increasing the constants to ensure that $n=0$ trivially holds, we get that the events also hold for any number of visit $n_h^{k-1}(s,a)\in\brc*{0\dots,k}$, and taking another union bound over all $k\ge1,s,a,h$ ensures that each of the events $\cap_{k\geq 1} E^p(k), \cap_{k\geq 1} E^{pv1}(k)$ and $\cap_{k\geq 1} E^{pv2}(k)$ holds w.p. at least $1-\frac{\delta}{8}$

    We now focus on bounding the probability of the event $\cap_k E^r(k)$. For any fixed $k$, $h$ and $s$, observe that the event trivially holds if $n_h^k=0$, then the event trivially holds, since for all $u\in[0,2H]^A$,
    \begin{align*}
        \abs*{m_h(s,u) - \hat{m}_h^{k-1}(s,u) } 
        = \abs*{\E_{\bR\sim\Rcal_h(s)}\brs*{\max_a\brc*{R_h(s,a)+u(a)}} - \max_a\brc*{u(a)}}
        \overset{(*)}\leq 1
        \leq 3\sqrt{ \frac{AL^k_{\delta} }{2}} ,
    \end{align*}
    where $(*)$ uses the boundedness of the rewards in $[0,1]$. Next, recall that for any fixed $n_h^{k-1}=n\in[k]$, the rewards samples at state $s$ and step $h$ are i.i.d. vectors on $[0,1]^A$. Therefore, by \Cref{lemma:max concentration rewards},     \begin{align*}
        &\Pr\brc*{n^{k-1}_h(s)=n, \forall  u\in[0,2H]^A:\ \abs*{m_h(s,u) - \hat{m}_h^{k-1}(s,u) } > 3\sqrt{ \frac{AL^k_{\delta} }{2(n^{k-1}_h(s)\vee 1)}}} \leq \frac{\delta}{8SAHk^2(k+1)}.
    \end{align*}
    Taking a union bound on all possible values of $n\in[k]$, $s$ and $h$, we get 
    \begin{align*}
        \Pr\brc*{E^r(k)}\ge 1-SAk\cdot \frac{\delta}{8SAHk^2(k+1)}
        \geq 1-\frac{\delta}{8k(k+1)}.
    \end{align*}
    By summing over all $k\ge1$, the event $\cap_k E^r(k)$ holds with a probability of at least $1-\delta/8$. Finally, taking the union bound with the other three events leads to the desired result of $\Pr(\G_1)\geq 1-\delta/2$.
\end{proof}



\subsection{Optimism of the Upper Confidence Value Functions}
In this subsection, we prove that under the good event $\G_1$, the values $\bar{V}^k$ that MVP-RL produces are optimistic.

\begin{lemma}[Optimism] \label{lemma: optimism of values MVP-RL}
Under the first good event $\G_1$, for all $k\in[K]$, $h\in [H]$ and $s\in \mathcal{S}$, it holds that $V^*_h(s)\leq \bar{V}^k_{h}(s)$.
\end{lemma}
\begin{proof}
    The proof follows by backward induction on $H$; see that the claim trivially holds for $h=H+1$, where both values are defined to be zero. 

    Now assume by induction that for some $k\in[K]$ and $h\in[H]$, the desired inequalities hold at timestep $h+1$ for all $s\in\Scal$; we will show that this implies that they also hold at timestep $h$. 

    At this point, we also assume w.l.o.g. that $\bar{V}^k_{h}(s)< H$, and in particular, the value is not truncated; otherwise, by the boundedness of the rewards, $V^*_h(s)\leq H=\bar{V}^k_{h}(s).$ For similar reasons, we assume w.l.o.g. that $b_{k,h}^{p}(s,a)<H$, so that it is also not truncated.

     By the optimism of the value at step $h+1$ due to the induction hypothesis and the monotonicity of the bonus (\Cref{lemma: bonus monotonicity}), under the good event, we have for all $s\in\Scal$ and $a\in\Acal$ that
    \begin{align*}
        &\hat{P}_h^{k-1}\bar{V}^k_{h+1}(s,a) + b_{k,h}^{p}(s,a)\\
        &\geq   \hat{P}_h^{k-1}\bar{V}^k_{h+1}(s,a) + \max\brc*{\frac{20}{3}\sqrt{\frac{\VAR_{\hat{P}^{k-1}_{h}(\cdot| s,a)}(\bar{V}^k_{h+1}) L^k_{\delta}}{n^{k-1}_{h}(s,a)\vee 1}},\frac{400}{9}\frac{H L^k_{\delta}}{n^{k-1}_{h}(s,a) \vee 1}}\\
        &\geq \hat{P}_h^{k-1}V^*_{h+1}(s,a) + \max\brc*{\frac{20}{3}\sqrt{\frac{\VAR_{\hat{P}^{k-1}_{h}(\cdot| s,a)}(V^*_{h+1}) L^k_{\delta}}{n^{k-1}_{h}(s,a)\vee 1}},\frac{400}{9}\frac{H L^k_{\delta}}{n^{k-1}_{h}(s,a) \vee 1}}\tag{\Cref{lemma: bonus monotonicity}}\\
        &\geq \hat{P}_h^{k-1}V^*_{h+1}(s,a) + \frac{10}{3}\sqrt{\frac{\VAR_{\hat{P}^{k-1}_{h}(\cdot| s,a)}(V^*_{h+1}) L^k_{\delta}}{n^{k-1}_{h}(s,a)\vee 1}}+\frac{200}{9}\frac{H L^k_{\delta}}{n^{k-1}_{h}(s,a) \vee 1} \\
        & \geq \hat{P}_h^{k-1}V^*_{h+1}(s,a) + \frac{10}{3}\sqrt{\frac{\VAR_{P_{h}(\cdot| s,a)}(V^*_{h+1}) L^k_{\delta}}{n^{k-1}_{h}(s,a)\vee 1}}+\frac{8H L^k_{\delta}}{n^{k-1}_{h}(s,a) \vee 1} \tag{Under $E^{pv2}(k)$}\\
        &\geq P_hV^*_{h+1}(s,a).  \tag{Under $E^{pv1}(k)$}
    \end{align*}
    Thus, under the good event and the induction hypothesis, we have that
    \begin{align*}
        \bar{V}^k_h(s) &=\E_{\bR\sim\hat{\Rcal}_h(s)}\brs*{\max_{a\in\Acal}\brc*{R(a) + b_{k,h}^{p}(s,a) +\hat{P}_h^{k-1}\bar{V}^k_{h+1}(s,a)}}+ b^r_{k,h}(s) \\
        & \geq \E_{\bR\sim\hat{\Rcal}_h(s)}\brs*{\max_{a\in\Acal}\brc*{R(a) +P_hV^*_{h+1}(s,a)}} + b^r_{k,h}(s).         
    \end{align*}
    In particular, using \Cref{prop: reward DP}, we get 
    \begin{align*}
        \bar{V}^k_h(s) - V^*_h(s)
        & \geq \E_{\bR\sim\hat{\Rcal}_h(s)}\brs*{\max_{a\in\Acal}\brc*{R(a) +P_hV^*_{h+1}(s,a)}} + b^r_{k,h}(s) \\
        &\quad- \E_{\bR\sim\Rcal_h(s)}\brs*{\max_{a\in\Acal}\brc*{R(a) +P_hV^*_{h+1}(s,a)}}\\
        &\ge0,
    \end{align*}
    where the last inequality holds under the event $E^r(k)$ with $u(a)=P_hV^*_{h+1}(s,a)\in[0,H]^A$.
\end{proof}

\clearpage

\subsection{The Second Good Event -- Martingale Concentration}
In this subsection, we present four good events that will allow us to replace the expectation over the randomizations inside each episode with their realization.

Define the following bonus-like term that will later appear in the proof due to value concentration:
\begin{align*}
    b_{k,h}^{pv1}(s,a) = \min\brc*{\sqrt{\frac{2\VAR_{P_h(\cdot| s,a)}(V^*_{h+1})L^k_{\delta}}{n^{k-1}_h(s,a)\vee 1}} + \frac{4H^2SL^k_{\delta}}{n^{k-1}_h(s,a)\vee 1},H},
\end{align*}
and let 
\begin{align*}
    &Y^k_{1 ,h} \eqdef \bar{V}_{h+1}^k(s_{h+1}^k) - V_{h+1}^{\pi^k}(s_{h+1}^k),\\
    &Y^k_{2, h} = \VAR_{P_h(\cdot|s_{t,h},a_{t,h})}( V^{\pi^k}_{h+1}),\\
    &Y^k_{3, h} = b_{k,h}^{p}(s_h^k,a_h^k) + b_{k,h}^{pv1}(s_h^k,a_h^k).
\end{align*}
The second good event is the intersection of the events $\G_2 =E^{\mathrm{diff1}} \cap E^{\mathrm{diff2}} \cap  E^{\VAR}  \cap  E^{bp}$ defined as follows.
\begin{align*}
    &E^{\mathrm{diff}1}=\brc*{\forall h\in[H], K\geq 1:\ \sum_{k=1}^K \E[Y_{1 ,h}^k|F_{k,h-1}]\leq \br*{1+\frac{1}{2H}} \sum_{k=1}^K Y_{1 ,h}^k + 18H^2 \ln\frac{8HK(K+1)}{\delta}},\\
    &E^{\mathrm{diff}2}=\brc*{\forall h\in[H], K\geq 1:\ \sum_{k=1}^K \E[Y_{1 ,h}^k|F_{k,h-1}^R]\leq \br*{1+\frac{1}{2H}} \sum_{k=1}^K Y_{1 ,h}^k + 18H^2 \ln\frac{8HK(K+1)}{\delta}},\\
    &E^{\VAR}= \brc*{ K\geq 1:\  \sum_{k=1}^K \sum_{h=1}^H Y_{2, h}^k\leq 2\sum_{k=1}^K \sum_{h=1}^H\E[Y_{2, h}^k|F_{k-1}]  + 4H^3 \ln\frac{8HK(K+1)}{\delta}},\\
    &E^{bp}= \brc*{\forall h\in[H], K\geq 1:\ \sum_{k=1}^K \E[Y_{3, h}^k|F_{k,h-1}]\leq 2\sum_{k=1}^K Y_{3 ,h}^k + 50H^2 \ln\frac{8HK(K+1)}{\delta}},
\end{align*}
We define the good event $\G=\G_1\cap\G_2$.
\begin{lemma}
    \label{lemma: good event reward lookahead}
    The good event $\G$ holds with a probability of at least $1-\delta$.
\end{lemma}
\begin{proof}
    The proof follows similarly to Lemmas 15 and 21 of \citep{efroni2021confidence}.

    First, define the random process $W_k=\Ind{\bar{V}_{h}^k(s) - V_{h}^{\pi^k}(s)\in[0,H], \forall h\in[H], s\in\Scal}$ and define $\tilde{Y}^k_{1 ,h} = W_kY^k_{1 ,h}$, which is bounded in $[0,H]$. Also observe that $W_k$ is $F_{k-1}$ measurable, since both values and policies are calculated based on data up to the episode $k-1$, and in particular, it is $F_{k,h-1}$ measurable and $\tilde{Y}^k_{1, h}$ is $F_{k,h}$ measurable. thus, by \Cref{lemma: consequences of optimism and freedman's inequality}, for any $k\in[K]$ and $h\in[H]$, we have w.p. at least $1-\frac{\delta}{8HK(K+1)}$ that
    \begin{align*}
        \sum_{k=1}^K \E[\tilde{Y}_{1 ,h}^k|F_{k,h-1}]\leq \br*{1+\frac{1}{2H}} \sum_{k=1}^K \tilde{Y}_{1 ,h}^k + 18H^2 \ln\frac{8HK(K+1)}{\delta}.
    \end{align*}
    Since $W_k$ is $F_{k,h-1}$ measurable, we can write the event as
    \begin{align*}
        \sum_{k=1}^K W_k\E[Y_{1 ,h}^k|F_{k,h-1}]\leq \br*{1+\frac{1}{2H}} \sum_{k=1}^K W_kY_{1 ,h}^k + 18H^2 \ln\frac{8HK(K+1)}{\delta},
    \end{align*}
    and taking the union bound over all $h\in[H]$ and $K\ge1$, we get w.p. at least $1-\frac{\delta}{8}$ that the event
    \begin{align*}
        \tilde{E}^{\mathrm{diff}1}=\brc*{\forall h\in[H], K\geq 1:\ \sum_{k=1}^K W_k\E[Y_{1 ,h}^k|F_{k,h-1}]\leq \br*{1+\frac{1}{2H}} \sum_{k=1}^K W_kY_{1 ,h}^k + 18H^2 \ln\frac{8HK(K+1)}{\delta}}.
    \end{align*}
    Importantly, by optimism (\Cref{lemma: optimism of values MVP-RL}), under $\G_1$, it holds that $W_k=1$ for all $k\ge1$, so we immediately get that $\G_1\cap\tilde{E}^{\mathrm{diff}1}=\G_1\cap E^{\mathrm{diff}1}$.

    Following the exact same proof just with the filtration $F^R_{k,h}$ and defining the equivalent $\tilde{E}^{\mathrm{diff}2}$, we get that this event also holds w.p. $1-\frac{\delta}{8}$ and is the desired event when $\G_1$ holds.
    
    Next, we prove that the other two events also hold w.p. at least $1-\frac{\delta}{8}$.
    
    By the assumptions of our setting, we know that $V^{\pi^k}_{h}(s)\in[0,H]$, and so 
    \begin{align*}
        \sum_{h=1}^H Y^k_{2, h} = \sum_{h=1}^H\VAR_{P_h(\cdot|s_{t,h},a_{t,h})}( V^{\pi^k}_{h+1}) \in[0,H^3].
    \end{align*}
    In particular, applying \Cref{lemma: consequences of optimism and freedman's inequality} (w.r.t. the filtration $F_k$) with $C=H^3$ and any fixed $K$, we get w.p. $1-\frac{\delta}{8HK(K+1)}$ that 
    \begin{align*}
        \sum_{k=1}^K \sum_{h=1}^H Y_{2, h}^k\leq 2\sum_{k=1}^K \sum_{h=1}^H\E[Y_{2, h}^k|F_{k-1}]  + 4H^3 \ln\frac{8HK(K+1)}{\delta}.
    \end{align*}
    Taking the union bound on all possible values of $K\ge1$ proves that $E^{\VAR}$ holds w.p. at least $1-\frac{\delta}{8}$.

    Similarly, by definition, we have that $Y^k_{3, h} = b_{k,h}^{p}(s_h^k,a_h^k) + b_{k,h}^{pv1}(s_h^k,a_h^k)\in[0,2H]$ and is $F_{k,h}$ measurable. Thus, for any fixed $k\ge1$ and $h\in[H]$, using \Cref{lemma: consequences of optimism and freedman's inequality}, we have w.p. $1-\frac{\delta}{8HK(K+1)}$ that
    \begin{align*}
        \sum_{k=1}^K \E[Y_{3, h}^k|F_{k,h-1}]
        &\leq \br*{1+\frac{1}{4H}}\sum_{k=1}^K Y_{3 ,h}^k + 50H^2 \ln\frac{8HK(K+1)}{\delta} \\
        & \leq 2\sum_{k=1}^K Y_{3 ,h}^k + 50H^2 \ln\frac{8HK(K+1)}{\delta},
    \end{align*}
    applying the union bound on all $K\ge1$, the event $E^{bp}$ holds w.p. $1-\frac{\delta}{8}$.

    To summarize, we have that the event $\G_1$ holds w.p. $1-\frac{\delta}{2}$ (\Cref{lemma: the first good event reward lookahead}), and we proved that the events $\tilde{E}^{\mathrm{diff}1}, \tilde{E}^{\mathrm{diff}2},  E^{\VAR},  E^{bp}$ hold each w.p. $1-\frac{\delta}{8}$, so we also have that the event
    \begin{align*}
        \G&=\G_1\cap\G_2 \\
        &=\G_1\cap E^{\mathrm{diff}1} \cap   E^{\mathrm{diff}2} \cap  E^{\VAR}  \cap  E^{bp} \\
        &= \G_1\cap\tilde{E}^{\mathrm{diff}1} \cap   \tilde{E}^{\mathrm{diff}2} \cap  E^{\VAR}  \cap  E^{bp}
    \end{align*}
    holds w.p. at least $1-\delta$.
\end{proof}

\clearpage

\subsection{Regret Analysis}
\label{appendix: reward lookahead regret}
We finally analyze the regret of the algorithm
\MVPRL*
\begin{proof}
    Assume that the good events $\G$ holds, which by \Cref{lemma: good event reward lookahead}, happens with probability at least $1-\delta$. Then, by optimism (\Cref{lemma: optimism of values MVP-RL}), for any $k\in[K]$, $h\in[H]$ and $s\in\Scal$, it holds that $V^*_h(s)\leq \bar{V}^k_{h}(s)$. Moreover, we can lower bound the value of the policy $\pi^k$ as follows (see \Cref{remark: reward value with reward observations}):
    \begin{align}
    V^{\pi^k}_h(s) &= \E_{\bR\sim\Rcal_h(s)}\brs*{R(\pi^k_h(s,\bR)) +P_hV^{\pi^k}_{h+1}(s,\pi^k_h(s,\bR))} \nonumber\\
    & = \E_{\bR\sim\Rcal_h(s)}\brs*{R(\pi^k_h(s,\bR)) +\hat{P}_h^{k-1}\bar{V}^k_{h+1}(s,\pi^k_h(s,\bR))+b_{k,h}^p(s,\pi^k_h(s,\bR))} \nonumber\\
    &\quad+ \E_{\bR\sim\Rcal_h(s)}\brs*{P_hV^{\pi^k}_{h+1}(s,\pi^k_h(s,\bR))-\hat{P}_h^{k-1}\bar{V}^k_{h+1}(s,\pi^k_h(s,\bR))-b_{k,h}^p(s,\pi^k_h(s,\bR))} \nonumber\\
    & \overset{(1)}{=} \E_{\bR\sim\Rcal_h(s)}\brs*{\max_{a\in\Acal}\brc*{R(a) +\hat{P}_h^{k-1}\bar{V}^k_{h+1}(s,a)+b_{k,h}^p(s,a)}} \nonumber\\
    &\quad+ \E_{\bR\sim\Rcal_h(s)}\brs*{P_hV^{\pi^k}_{h+1}(s,\pi^k_h(s,\bR))-\hat{P}_h^{k-1}\bar{V}^k_{h+1}(s,\pi^k_h(s,\bR))-b_{k,h}^p(s,\pi^k_h(s,\bR))} \nonumber\\
    & \overset{(2)}\geq \E_{\bR\sim\hat{\Rcal}^{k-1}_h(s)}\brs*{\max_{a\in\Acal}\brc*{R(a) +\hat{P}_h^{k-1}\bar{V}^k_{h+1}(s,a)+b_{k,h}^p(s,a)}} -b_{k,h}^r(s)\nonumber\\
    &\quad+ \E_{\bR\sim\Rcal_h(s)}\brs*{P_hV^{\pi^k}_{h+1}(s,\pi^k_h(s,\bR))-\hat{P}_h^{k-1}\bar{V}^k_{h+1}(s,\pi^k_h(s,\bR))-b_{k,h}^p(s,\pi^k_h(s,\bR))} \nonumber\\
    &\overset{(3)}\geq \bar{V}_h^k(s)-2b_{k,h}^r(s) \nonumber\\
    &\quad+ \E_{\bR\sim\Rcal_h(s)}\brs*{P_hV^{\pi^k}_{h+1}(s,\pi^k_h(s,\bR))-\hat{P}_h^{k-1}\bar{V}^k_{h+1}(s,\pi^k_h(s,\bR))-b_{k,h}^p(s,\pi^k_h(s,\bR))}. \label{eq: value lower bound reward lookahead}
\end{align}
Relation $(1)$ is by the definition of $\pi^k$ (see \Cref{alg: MVP reward lookahead}), while $(2)$ holds under the good event $E^r(k)$ with $u(a)=\hat{P}_h^{k-1}\bar{V}^k_{h+1}(s,a)+b_{k,h}^p(s,a)\in[0,2H]$ (due to the value and bonus truncation). Finally, $(3)$ is by the definition of $\bar{V}_h^k(s)$, where the inequality also accounts for its possible truncation.

To further bound this, we need to bound
\begin{align*}
    \hat{P}_h^{k-1}\bar{V}^k_{h+1}(s,a) - P_hV^{\pi^k}_{h+1}(s,a)
    &= P_h\br*{\bar{V}^k_{h+1} - V^{\pi^k}_{h+1}}(s,a)+\br*{\hat{P}_h^{k-1} - P_h}\bar{V}^k_{h+1}(s,a) \\
    &= P_h\br*{\bar{V}^k_{h+1} - V^{\pi^k}_{h+1}}(s,a)\\
    &\quad+\br*{\hat{P}_h^{k-1}-P_h}V^*_{h+1}(s,a) + \br*{\hat{P}_h^{k-1}-P_h}\br*{\bar{V}^k_{h+1}-V^*_{h+1}}(s,a).
\end{align*}
The first error term can be bounded under the good event, while the second using \Cref{lemma: transition different to next state expectation}. More formally, under the good event $E^{pv1}(k)$, we have 
\begin{align*}
    \abs*{\br*{\hat{P}_h^{k-1}-P_h}V^*_{h+1}(s,a)}\leq \sqrt{\frac{2\VAR_{P_h(\cdot| s,a)}(V^*_{h+1})L^k_{\delta}}{n^{k-1}_h(s,a)\vee 1}} + \frac{HL^k_{\delta}}{n^{k-1}_h(s,a)\vee 1},
\end{align*}
and by \Cref{lemma: transition different to next state expectation} with $\alpha=4H$ (using and $P_1=P_h$, $P_2=\hat{P}^{k-1}_h$, under $E^p(k)$),
\begin{align*}
    \abs*{\br*{\hat{P}_h^{k-1}-P_h}\br*{\bar{V}^k_{h+1}-V^*_{h+1}}(s,a)} &\leq \frac{1}{4H} \E_{P_h(\cdot\vert s,a)}\brs*{\bar{V}^{k}_{h+1}(s')  - V^{*}_{h+1}(s')} + \frac{H S L^k_{\delta}(1+ 4H\cdot 2/4)}{n^{k-1}_h(s,a)\vee 1} \\
        & \leq \frac{1}{4H} \E_{P_h(\cdot\vert s,a)}\brs*{\bar{V}^{k}_{h+1}(s') - V^{\pi^k}_{h+1}(s')} + \frac{3H^2 S L^k_{\delta}}{n^{k-1}_h(s,a)\vee 1} \\
        & = \frac{1}{4H}P_h\br*{\bar{V}^{k}_{h+1} - V^{\pi^k}_{h+1}}(s,a)+\frac{3H^2 S L^k_{\delta}}{n^{k-1}_h(s,a)\vee 1},
\end{align*}
where the second inequality is since the value of $\pi^k$ cannot exceed the optimal value.

Since under the good event by \Cref{lemma: optimism of values MVP-RL}, we have $0\leq V^{\pi^k}_{h+1}(s')\leq V^*_{h+1}(s')\leq \bar{V}^k_{h+1}(s') \leq H$, we can trivially bound the error by $H$ and bound
\begin{align*}
    &\hat{P}_h^{k-1}\bar{V}^k_{h+1}(s,a) - P_hV^{\pi^k}_{h+1}(s,a)\\
    & \leq \min\brc*{\br*{1+\frac{1}{4H}}\underbrace{P_h\br*{\bar{V}^{k}_{h+1} - V^{\pi^k}_{h+1}}(s,a)}_{\ge0}+\frac{3H^2 S L^k_{\delta}}{n^{k-1}_h(s,a)\vee 1}+ \sqrt{\frac{2\VAR_{P_h(\cdot| s,a)}(V^*_{h+1})L^k_{\delta}}{n^{k-1}_h(s,a)\vee 1}} + \frac{HL^k_{\delta}}{n^{k-1}_h(s,a)\vee 1}, H}\\
    & \leq \br*{1+\frac{1}{4H}}P_h\br*{\bar{V}^{k}_{h+1} - V^{\pi^k}_{h+1}}(s,a)+\min\brc*{\sqrt{\frac{2\VAR_{P_h(\cdot| s,a)}(V^*_{h+1})L^k_{\delta}}{n^{k-1}_h(s,a)\vee 1}} + \frac{4H^2 S L^k_{\delta}}{n^{k-1}_h(s,a)\vee 1}, H}\\
    & \triangleq \br*{1+\frac{1}{4H}}P_h\br*{\bar{V}^{k}_{h+1} - V^{\pi^k}_{h+1}}(s,a)+b_{k,h}^{pv1}(s,a).
\end{align*}
Substituting back to \Cref{eq: value lower bound reward lookahead} while writing the linear operation $P_hV(s,a)$ as an expectation and letting the action be $a_h=\pi^k_h(s,\bR)$, we get under $\G$ for all $k\in[K]$, $h\in[H]$ and $s\in\Scal$ that
\begin{align*}
    &\bar{V}_h^{k}(s) - V_h^{\pi^k}(s) \\
    & \leq \E_{\bR\sim\Rcal_h(s)}\brs*{\hat{P}_h^{k-1}\bar{V}^k_{h+1}(s,\pi^k_h(s,\bR)) - P_hV^{\pi^k}_{h+1}(s,\pi^k_h(s,\bR)) + b_{k,h}^p(s,\pi^k_h(s,\bR))}+2b_{k,h}^r(s)\\
     & \leq  \E_{\bR\sim\Rcal_h(s)}\brs*{\br*{1+\frac{1}{4H}}\E\brs*{\bar{V}^{k}_{h+1}(s_{h+1}) - V^{\pi^k}_{h+1}(s_{h+1})\vert s_h=s,a_h}(s,a)+b_{k,h}^{pv1}(s,a_h) + b_{k,h}^p(s,a_h)}+2b_{k,h}^r(s)\\
     & = \E\brs*{\br*{1+\frac{1}{4H}}\br*{\bar{V}^{k}_{h+1}(s_{h+1}) - V^{\pi^k}_{h+1}(s_{h+1})} + b_{k,h}^{p}(s_h,a_h) + b_{k,h}^{pv1}(s_h,a_h) \vert s_h=s,\pi^k}+2b_{k,h}^r(s). 
\end{align*}
Next, taking $s=s_h^k$, the action $a_h=\pi^k_h(s,\bR)$ becomes $a_h^k$, and summing on all $k$, we can rewrite
\begin{align*}
    &\sum_{k=1}^K\bar{V}_h^{k}(s_h^k) - V_h^{\pi^k}(s_h^k) \\
     & \leq \sum_{k=1}^K\E\brs*{\br*{1+\frac{1}{4H}}\br*{\bar{V}^{k}_{h+1}(s_{h+1}^k) - V^{\pi^k}_{h+1}(s_{h+1}^k)} +  b_{k,h}^{p}(s_h^k,a_h^k) + b_{k,h}^{pv1}(s_h^k,a_h^k)  \vert F_{k,h-1}}+2\sum_{k=1}^Kb_{k,h}^r(s_h^k) \\
     & \overset{(1)}\leq \br*{1+\frac{1}{2H}}\br*{1+\frac{1}{4H}}\sum_{k=1}^K\br*{\bar{V}^{k}_{h+1}(s_{h+1}^k) - V^{\pi^k}_{h+1}(s_{h+1}^k)} \\
     &\quad+ 2 \sum_{k=1}^K\br*{b_{k,h}^{p}(s_h^k,a_h^k) + b_{k,h}^{pv1}(s_h^k,a_h^k)}+2\sum_{k=1}^Kb_{k,h}^r(s_h^k) + 68H^2 \ln\frac{8HK(K+1)}{\delta} \\
     & \overset{(2)}\leq \br*{1+\frac{1}{2H}}\br*{1+\frac{1}{4H}}\sum_{k=1}^K\br*{\bar{V}^{k}_{h+1}(s_{h+1}^k) - V^{\pi^k}_{h+1}(s_{h+1}^k)} +  \frac{1}{4H}\br*{1+\frac{1}{2H}}\sum_{k=1}^K \br*{\bar{V}^k_{h+1}(s_{h+1}^k)-V^{\pi^k}_{h+1}(s_{h+1}^k)}\\
     & \quad+ 18\sum_{k=1}^K\sqrt{\frac{\VAR_{P_h(\cdot| s_h^k,a_h^k)}(V^{\pi^k}_{h+1})L^k_{\delta}}{n^{k-1}_h(s_h^k,a_h^k)\vee 1}} 
        + \sum_{k=1}^K\frac{1620H^2SL^k_{\delta}}{n^{k-1}_h(s_h^k,a_h^k)\vee 1}+ 68H^2 \ln\frac{8HK(K+1)}{\delta} +2\sum_{k=1}^Kb_{k,h}^r(s_h^k) \\
     & \leq \br*{1+\frac{1}{2H}}^2\sum_{k=1}^K\br*{\bar{V}^{k}_{h+1}(s_{h+1}^k) - V^{\pi^k}_{h+1}(s_{h+1}^k)}+ 18\sum_{k=1}^K\frac{\sqrt{L^k_{\delta}\VAR_{P_{h}(\cdot| s_h^k,a_h^k)}(V^{\pi^k}_{h+1})}}{\sqrt{n^{k-1}_{h}(s_h^k,a_h^k)\vee 1}} \\
     &\quad+  \sum_{k=1}^K \frac{1700H^2SL^k_{\delta}}{n^{k-1}_{h}(s_h^k,a_h^k)\vee 1}+6\sum_{k=1}^K\sqrt{ \frac{AL^k_{\delta} }{2n^{k-1}_{h}(s)\vee 1}}
\end{align*}
where inequality $(1)$ holds when both $E^{\mathrm{diff}1}$ and $E^{bp}$ occur and inequality $(2)$ is by \Cref{lemma: reward lookahead bonus sum}. In the last inequality, we also substituted the definition of the reward bonus. Recursively applying this inequality up to $h=H+1$ (where both values are zero), w.p. at least $1-\delta$, we get 
\begin{align*}
    \Regret^R(K)
    &\leq \sum_{k=1}^K\br*{V^{*}_1(s_1^k) - V^{\pi^k}_1(s_1^k)}\\
    & \leq \sum_{k=1}^K\br*{\bar{V}_1^{k}(s_1^k) - V^{\pi^k}_1(s_1^k)}\tag{\Cref{lemma: optimism of values MVP-RL}}\\
    & \leq 18\br*{1+\frac{1}{2H}}^{2H}\sum_{k=1}^K\frac{\sqrt{L^k_{\delta}\VAR_{P_{h}(\cdot| s_h^k,a_h^k)}(V^{\pi^k}_{h+1})}}{\sqrt{n^{k-1}_{h}(s_h^k,a_h^k)\vee 1}} 
    + \br*{1+\frac{1}{2H}}^{2H}\sum_{k=1}^K \frac{1700H^2SL^k_{\delta}}{n^{k-1}_{h}(s_h^k,a_h^k)\vee 1} \\
    &\quad+6\br*{1+\frac{1}{2H}}^{2H}\sum_{k=1}^K\sqrt{ \frac{AL^k_{\delta} }{2n^{k-1}_{h}(s)\vee 1}} \\
    & \overset{(*)}\leq 100\sqrt{H^3SAK}L^K_{\delta} + 50\sqrt{2SA}H^2 \br*{L^K_{\delta}}^{1.5} \\
    & \quad + 5000H^2SL^K_{\delta}\cdot SAH\br*{2 + \ln(K)} + 12\sqrt{AL^K_{\delta}}\br*{SH + 2\sqrt{SH^2K}} \\
    & = \Ocal\br*{\sqrt{H^3SAK}L^K_{\delta} + H^3S^2A(L^K_{\delta})^2}.
\end{align*}
Relation $(*)$ is by \Cref{lemma: sum value variance bound reward lookahead} and \Cref{lemma: count sum bounds}.
\end{proof}

\clearpage
\subsubsection{Lemmas for Bounding Bonus Terms}

\begin{lemma}
\label{lemma: reward lookahead bonus sum}
    Conditioned on the good event $\G$, for any $h\in[H]$, it holds that
    \begin{align*}
        \sum_{k=1}^K \br*{b_{k,h}^{p}(s_h^k,a_h^k) + b_{k,h}^{pv1}(s_h^k,a_h^k)}
        &\leq  \frac{1}{8H}\br*{1+\frac{1}{2H}}\sum_{k=1}^K \br*{\bar{V}^k_{h+1}(s_{h+1}^k)-V^{\pi^k}_{h+1}(s_{h+1}^k)}\\
        &\quad + 9\sum_{k=1}^K\sqrt{\frac{\VAR_{P_h(\cdot| s_h^k,a_h^k)}(V^{\pi^k}_{h+1})L^k_{\delta}}{n^{k-1}_h(s_h^k,a_h^k)\vee 1}} 
        + \sum_{k=1}^K\frac{810H^2SL^k_{\delta}}{n^{k-1}_h(s_h^k,a_h^k)\vee 1}.
    \end{align*}
\end{lemma}
\begin{proof}
    We start by analyzing each of the terms separately. First, we apply \Cref{lemma: variance difference bound with different measures} with $\alpha = \frac{20}{3}\cdot32HL^k_{\delta}$, noting that under the good event (by \Cref{lemma: optimism of values MVP-RL}), $0\leq V_{h+1}^{\pi^k}(s)\leq V_{h+1}^*(s)\leq \bar{V}_{h+1}^k(s)\leq H$ and using the event $E^{pv}$; doing so yields
    \begin{align*}
        b_{k,h}^{p}(s,a) 
        &\leq \frac{20}{3}\sqrt{\frac{\VAR_{\hat{P}^{k-1}_{h}(\cdot| s,a)}(\bar{V}^k_{h+1}) L^k_{\delta}}{n^{k-1}_{h}(s,a)\vee 1}} + \frac{400}{9}\frac{H L^k_{\delta}}{n^{k-1}_{h}(s,a) \vee 1} \\
        & \leq \frac{20\sqrt{L^k_{\delta}\VAR_{P_{h}(\cdot| s,a)}(V^{\pi^k}_{h+1})}}{3\sqrt{n^{k-1}_{h}(s,a)\vee 1}} 
        + \frac{1}{32H}P_h\br*{\bar{V}_{h+1}^k-V_{h+1}^{\pi^k}}(s,a) + \frac{1}{32H}\hat{P}^{k-1}_h\br*{\bar{V}_{h+1}^k-V_{h+1}^{\pi^k}}(s,a) \\
         &\quad+ \frac{6400H^2L^k_{\delta}}{9n^{k-1}_h(s,a)\vee 1} + \frac{20}{3}\frac{4HL^k_{\delta}}{n^{k-1}_h(s,a)\vee 1} + \frac{400}{9}\frac{H L^k_{\delta}}{n^{k-1}_{h}(s,a) \vee 1}
    \end{align*}
    Using \Cref{lemma: transition different to next state expectation} with $\alpha=1$, under the good event $E^p(k)$ and for any $s,a$, we can further bound
    \begin{align*}
        &\hat{P}^{k-1}_{h}\br*{\bar{V}^k_{h+1}-V^{\pi^k}_{h+1}}(s,a) \\
        &\quad = P_{h}\br*{\bar{V}^k_{h+1}-V^{\pi^k}_{h+1}}(s,a) 
        + \br*{\hat{P}^{k-1}_{h}-P_h}\br*{\bar{V}^k_{h+1}(s')-V^{\pi^k}_{h+1}}(s,a) \\
        &\quad\leq  P_{h}\br*{\bar{V}^k_{h+1}-V^{\pi^k}_{h+1}}(s,a) 
        +  P_{h}\br*{\bar{V}^k_{h+1}-V^{\pi^k}_{h+1}}(s,a) + \frac{H S L^k_{\delta}(1+ 2\cdot 1/4)}{n^{k-1}_h(s,a)\vee 1}\tag{\Cref{lemma: transition different to next state expectation}}\\
        &\quad\leq  2P_{h}\br*{\bar{V}^k_{h+1}-V^{\pi^k}_{h+1}}(s,a) +\frac{1.5H S L^k_{\delta}}{n^{k-1}_h(s,a)\vee 1} 
    \end{align*}
    Thus, we get the overall bound
    \begin{align*}
        b_{k,h}^{p}(s,a) 
        &\leq \frac{20\sqrt{L^k_{\delta}\VAR_{P_{h}(\cdot| s,a)}(V^{\pi^k}_{h+1})}}{3\sqrt{n^{k-1}_{h}(s,a)\vee 1}} 
        + \frac{3}{32H}P_h\br*{\bar{V}_{h+1}^k-V_{h+1}^{\pi^k}}(s,a)+ \frac{785H^2S L^k_{\delta}}{n^{k-1}_{h}(s,a) \vee 1}
    \end{align*}
    For the second bonus, we apply \Cref{lemma: variance difference bound} w.r.t. $V_{h+1}^{\pi^k}(s)\leq V_{h+1}^*(s)$ and $\alpha = 32\sqrt{2L^k_{\delta}}H$ and get
    \begin{align*}
        b_{k,h}^{pv1}(s,a) 
        &\leq \sqrt{\frac{2\VAR_{P_h(\cdot| s,a)}(V^*_{h+1})L^k_{\delta}}{n^{k-1}_h(s,a)\vee 1}} + \frac{4H^2SL^k_{\delta}}{n^{k-1}_h(s,a)\vee 1} \\
        &\leq \sqrt{\frac{2\VAR_{P_h(\cdot| s,a)}(V^{\pi^k}_{h+1})L^k_{\delta}}{n^{k-1}_h(s,a)\vee 1}} 
        + \frac{1}{32H}P_h\br*{V_{h+1}^*-V_{h+1}^{\pi^k}}(s,a) + \frac{16HL^k_{\delta}}{n^{k-1}_h(s,a)} + \frac{4H^2SL^k_{\delta}}{n^{k-1}_h(s,a)\vee 1} \\
        & \leq \sqrt{\frac{2\VAR_{P_h(\cdot| s,a)}(V^{\pi^k}_{h+1})L^k_{\delta}}{n^{k-1}_h(s,a)\vee 1}} 
        + \frac{1}{32H}P_h\br*{\bar{V}_{h+1}^k-V_{h+1}^{\pi^k}}(s,a) + \frac{20H^2SL^k_{\delta}}{n^{k-1}_h(s,a)\vee 1} 
    \end{align*}
    where we again used the optimism. Combining both and summing over all $k$, we get 
    \begin{align*}
        \sum_{k=1}^K \br*{b_{k,h}^{p}(s_h^k,a_h^k) + b_{k,h}^{pv1}(s_h^k,a_h^k)}
        &\leq  9\sum_{k=1}^K\sqrt{\frac{\VAR_{P_h(\cdot| s_h^k,a_h^k)}(V^{\pi^k}_{h+1})L^k_{\delta}}{n^{k-1}_h(s_h^k,a_h^k)\vee 1}} 
        + \frac{1}{8H}\sum_{k=1}^KP_h\br*{\bar{V}_{h+1}^k-V_{h+1}^{\pi^k}}(s_h^k,a_h^k)\\
        &\quad+ \sum_{k=1}^K\frac{805H^2SL^k_{\delta}}{n^{k-1}_h(s_h^k,a_h^k)\vee 1} 
    \end{align*}
    Finally, under the good event $E^{\mathrm{diff}2}$, it holds that 
    \begin{align*}
        \sum_{k=1}^K P_{h}\br*{\bar{V}^k_{h+1}-V^{\pi^k}_{h+1}}(s_h^k,a_h^k)  
        &= \sum_{k=1}^K \E\brs*{\bar{V}_{h+1}^k(s_{h+1}^k) - V_{h+1}^{\pi^k}(s_{h+1}^k)\vert F_{k,h-1}^R}\\
        &\leq \br*{1+\frac{1}{2H}}\sum_{k=1}^K \br*{\bar{V}^k_{h+1}(s_{h+1}^k)-V^{\pi^k}_{h+1}(s_{h+1}^k)} + 18H^2\ln\frac{8HK(K+1)}{\delta}.
    \end{align*}
    Substituting this relation back concludes the proof.
\end{proof}

\begin{lemma}
\label{lemma: sum value variance bound reward lookahead}
    Under the event $E^{\VAR}$ it holds that
    \begin{align*}
        \sum_{k=1}^K\sum_{h=1}^H \frac{\sqrt{\VAR_{P_{h}(\cdot| s_h^k,a_h^k)}(V^{\pi^k}_{h+1})}}{\sqrt{n^{k-1}_{h}(s_h^k,a_h^k)\vee 1}}
        \leq 2\sqrt{H^3SAKL^K_{\delta}} + \sqrt{8SA}H^2 L^K_{\delta}.
    \end{align*}
\end{lemma}
\begin{proof}
    Following Lemma 24 of \citep{efroni2021confidence}, by Cauchy-Schwartz inequality, it holds that
    \begin{align*}
        \sum_{k=1}^K\sum_{h=1}^H &\frac{\sqrt{\VAR_{P_{h}(\cdot| s_h^k,a_h^k)}(V^{\pi^k}_{h+1})}}{\sqrt{n^{k-1}_{h}(s_h^k,a_h^k)\vee 1}}
        \leq \sqrt{ \sum_{k=1}^K\sum_{h=1}^H\VAR_{P_{h}(\cdot| s_h^k,a_h^k)}(V^{\pi^k}_{h+1})}\sqrt{ \sum_{k=1}^K\sum_{h=1}^H \frac{1}{n^{k-1}_{h}(s_h^k,a_h^k)\vee 1}}.
    \end{align*}
    The second term can be bounded by \Cref{lemma: count sum bounds}, namely,
    \begin{align*}
        \sum_{k=1}^K\sum_{h=1}^H \frac{1}{n^{k-1}_{h}(s_h^k,a_h^k)\vee 1}
        \leq SAH\br*{2 + \ln(K)}.
    \end{align*}
    We further focus on bounding the first term. Under $E^{\VAR}$, we have
    \begin{align*}
        &\sum_{k=1}^K\sum_{h=1}^H\VAR_{P_{h}(\cdot| s_h^k,a_h^k)}(V^{\pi^k}_{h+1})\\
        &\leq 2\sum_{k=1}^K \E\brs*{\sum_{h=1}^H\VAR_{P_{h}(\cdot| s_h^k,a_h^k)}(V^{\pi^k}_{h+1})|F_{k-1}}  + 4H^3 \ln\frac{8HK(K+1)}{\delta} \tag{Under $E^{\VAR}$}\\
        & \leq 2\sum_{k=1}^K \E\brs*{\br*{\sum_{h=1}^H R_h(s_h^k,a_h^k) - V_1^{\pi^k}(s_1^k) }^2|F_{k-1}}  + 4H^3 \ln\frac{8HK(K+1)}{\delta} \tag{By \Cref{lemma: ltv reward-lookahead} }\\
        & \leq 2H^2K  + 4H^3 \ln\frac{8HK(K+1)}{\delta},
    \end{align*}
    where the last inequality is since both the values and cumulative rewards are bounded in $[0,H]$. Combining both, we get
    \begin{align*}
        \sum_{k=1}^K\sum_{h=1}^H \frac{\sqrt{\VAR_{P_{h}(\cdot| s_h^k,a_h^k)}(V^{\pi^k}_{h+1})}}{\sqrt{n^{k-1}_{h}(s_h^k,a_h^k)\vee 1}}
        &\leq \sqrt{2H^2K  + 4H^3 \ln\frac{8HK(K+1)}{\delta}} \sqrt{SAH\br*{2 + \ln(K)}} \\
        & \leq \sqrt{2H^2K  + 4H^3 \ln\frac{8HK(K+1)}{\delta}} \sqrt{2SAH\ln\frac{8HK(K+1)}{\delta}} \\
        & \leq 2\sqrt{H^3SAKL^K_{\delta}} + \sqrt{8SA}H^2 L^K_{\delta}.
    \end{align*}
\end{proof}
\clearpage


\section{Proofs for Transition Lookahead}
\label{appendix: transition lookahead proofs}
\subsection{Data Generation Process}
\label{appendix: transition lookahead data generation}
As for the reward transition, we also assume that all data was generated before the game starts for all state-action-timesteps, and it is given to the agent when the relevant  $(s,a,h)$ is visited. Thus, the rewards and next-state from the first $i^{th}$ visits at a state (or a state-action pair) at a certain timestep are i.i.d.

Throughout this appendix, we use the notation $\bs'^k_{h+1} = \brc*{s'^k_{h+1}(s_h^k,a)}_{a\in\Acal}$ to denote the next-state observations at episode $k$ and timestep $h$ for all the actions, and use the equivalent filtrations to the ones defined at \Cref{appendix: reward lookahead data generation}, namely
\begin{align*}
    &F_{k,h}=\sigma\br*{\brc*{s_t^1,a_t^1,\bs'^1_{t+1},R^1_t}_{t\in[H]},\dots, \brc*{s_t^{k-1},a_t^{k-1},\bs'^{k-1}_{t+1},R^{k-1}_t}_{t\in[H]}, \brc*{s_t^{k},a_t^{k},\bs'^{k}_{t+1},R^{k}_t}_{t\in[h]}},\\
    &F_{k}=\sigma\br*{\brc*{s_t^1,a_t^1,\bs'^1_{t+1}}_{t\in[H]},\dots, \brc*{s_t^{k},a_t^{k},\bs'^{k}_{t+1},R^{k}_t}_{t\in[H]},s_{1}^{k+1}}.
\end{align*}
In particular, notice that since both $\bs'^{k}_{h+1}$ and $a^k_h$ are $F_{k,h}$ measurable, then so does $s^k_{h+1}$.

\subsection{Extended MDP for Transition Lookahead}
\label{appendix: extended MDP transitions}
In this appendix, we present an equivalent extended MDP that embeds the lookahead into the state to fall under the vanilla MDP model, similarly to \Cref{appendix: extended MDP rewards}. We use this equivalence to apply various existing results on MDPs without the need to reprove them. We follow the same conventions as \Cref{appendix: extended MDP rewards} while denoting transition lookahead values by $V^{T,\pi}(s\vert\Mcal)$ (and again, the superscript $T$ will be omitted in subsequent subsections).

For any MDP $\Mcal=(\Scal,\Acal,H,P,\Rcal)$, let $\Mcal^T$ be an MDP of horizon $2H$ and state space $\Scal^{A+1}$ that separates the state transition and next-state generation as follows:
\begin{enumerate}
    \item Assume w.l.o.g. that $\Mcal$ starts at some initial state $s_1$. The extended environment starts at a state $s_1\times \bs'_0$, where $\bs'_0\in\Scal^A$ is a vector of $A$ copies of some arbitrary state $s_0\in\Scal$.
    \item For any $h\in[H]$, at timestep $2h-1$, the environment $\Mcal^T$ transitions from state $s_h\times \bs'_0$ to $s_h\times \bs'_{h+1}$, where $\bs'_{h+1}\sim P_h(s)$ is a vector containing the next state for all actions $a\in\Acal$; this transition happens regardless of the action that the agent played. At timestep $2h$, given an action $a_h$, the environment transitions from $s_h\times \bs'_{h+1}$ to $s'_{h+1}(a)\times \bs'_0$. 
    \item The rewards at odd steps $2h-1$ are zero, while the rewards at even steps $2h$ are $R_h(s_h,a_h)\sim \Rcal_h(s_h,a_h)$ of expectation $r_h(s_h,a_h)$.
\end{enumerate}
As before, since the next state is embedded into the extended state space, any state-dependent policy in $\Mcal^T$ is a one-step transition lookahead policy in the original MDP. Also, the policy at even timesteps does not affect either the rewards or transitions, so it does not affect the value in any way. We again couple the two environments to have the exact same randomness, so assuming that the policy at the even steps in $\Mcal^T$ is the same as the policy in $\Mcal$, we trivially get the following relation between the values
\begin{align}   
    \label{eq: extended transition MDP to standard value relation}
    & V_{2h}^\pi(s,\bs'\vert \Mcal^T) = \E\brs*{\sum_{t=h}^H R_t(s_t,a_t)\vert s_h=s,s'_{h+1}(s,\cdot)=\bs',\pi} \triangleq V_h^{T,\pi}(s,\bs'\vert \Mcal), \nonumber\\
    & V_{2h-1}^\pi(s, \bs'_0\vert \Mcal^T) = \E\brs*{\sum_{t=h}^H R_t(s_t,a_t)\vert s_h=s,\pi} = V_h^{T,\pi}(s\vert \Mcal) .
\end{align}
While $\Mcal^T$ is finite, it is exponential in size, so applying any standard algorithm in this environment would lead to exponentially-bad performance bounds. Nonetheless, as with the extended-reward environment, we use this representation to prove useful results on one-step transition lookahead. 
\clearpage
\transitionDP*
\begin{proof}
    We prove the result in the extended MDP $\Mcal^T$, in which (as with reward lookahead) the optimal value can be calculated using the Bellman equations as follows \citep{puterman2014markov}
    \begin{align}
        \label{eq: Bellman extended transition MDP}
        &V_{2H+1}^T(s,\bs'\vert \Mcal^T)=0, &\forall s\in\Scal, \bs'\in\Scal^A,\nonumber\\
        &V_{2h}^*(s,\bs'\vert \Mcal^T) = \max_a\brc*{r_h(s,a)+V_{2h+1}^*(s'(a), \bs'_0\vert \Mcal^T)}, &\forall h\in[H],s\in\Scal, \bs'\in\Scal^A,\nonumber\\
        &V_{2h-1}^*(s, \bs'_0\vert \Mcal^T) = \E_{\bs'\sim P_h(s)}\brs*{V_{2h}^*(s,\bs'\vert \Mcal^T)}, &\forall h\in[H],s\in\Scal.
    \end{align}
    By the equivalence between $\Mcal$ and $\Mcal^T$ for all policies, this is also the optimal value in $\Mcal$. Combining both recursion equations and substituting  \Cref{eq: extended transition MDP to standard value relation} leads to the stated value calculation for all $h\in[H]$ and $s\in\Scal$:
    \begin{align*}
        V^{T,*}_h(s\vert\Mcal) 
        &= V_{2h-1}^*(s, \bs'_0\vert \Mcal^T) \\
        & = \E_{\bs'\sim P_h(s)}\brs*{V_{2h}^*(s,\bs'_{h+1}\vert \Mcal^T)} \\
        & = \E_{\bs'\sim P_h(s)}\brs*{ \max_a\brc*{r_h(s,a)+V_{2h+1}^*(s'_{h+1}(a), \bs'_0\vert \Mcal^T)}} \\
        & = \E_{\bs'\sim P_h(s)}\brs*{ \max_a\brc*{r_h(s,a)+V^{T,*}_{h+1}(s'_{h+1}(a)\vert \Mcal)}}.
    \end{align*}
    
    In addition, a given state $s$ and next-state observations $\bs'$, the optimal policy at the even stages of the extended MDP is
    \begin{align*}
       \pi^*_{2h}(s,\bs')\in\argmax_{a\in\Acal}\brc*{r_h(s,a)+V^{*}_{2h+1}(s'(a))},
    \end{align*}
    alongside arbitrary actions at odd steps. Playing this policy in the original MDP will lead to the optimal one-step transition lookahead policy, as it achieves the optimal value of the original MDP. 
    By the value relations between the two environments ($V_{2h+1}^*(s, \bs'_0\vert \Mcal^T)=V_{h+1}^{T,*}(s\vert \Mcal)$), this is equivalent to the stated policy.
\end{proof}

\begin{remark}
    \label{remark: transition value with next-state observations}
    As in \Cref{remark: reward value with reward observations}, one could write the dynamic programming equations for any policy $\pi\in\Pi^T$, and not just to the optimal one, namely 
    \begin{align*}
        &V_{2h}^{\pi}(s,\bs'\vert \Mcal^T) = r_h(s,\pi(s,\bs'))+V_{2h+1}^*(s'(\pi_h(s,\bs')), \bs'_0\vert \Mcal^T), &\forall h\in[H],s\in\Scal, \bs'\in\Scal^A,\nonumber\\
        &V_{2h-1}^{\pi}(s, \bs'_0\vert \Mcal^T) = \E_{\bs'\sim P_h(s)}\brs*{V_{2h}^{\pi}(s,\bs'\vert \Mcal^T)}, &\forall h\in[H],s\in\Scal.
    \end{align*}
    In particular, following the notation of \Cref{eq: extended transition MDP to standard value relation}, we can write
    \begin{align*}
        &V_h^{T,\pi}(s,\bs'\vert \Mcal) =r_h(s,\pi_h(s,\bs'))+V_{h+1}^{T,\pi}(s'(\pi_h(s,\bs'))\vert \Mcal), \qquad \textrm{and,}\\
        & V_h^{T,\pi}(s\vert \Mcal) = \E_{\bs'\sim P_h(s)}\brs*{V_h^{T,\pi}(s,\bs'\vert \Mcal)}\\
        &\hspace{4.875em}=\E_{\bs'\sim P_h(s)}\brs*{r_h(s,\pi_h(s,\bs'))+V_{h+1}^{T,\pi}(s'(\pi_h(s,\bs'))\vert \Mcal)},
    \end{align*}
    a notation that will be extensively used for transition lookahead.
\end{remark}

\clearpage
We also prove a variation of the law of total variance (LTV) for transition lookahead:
\begin{lemma}
    \label{lemma: ltv transition-lookahead} 
    For any one-step transition lookahead policy $\pi\in\Pi^T$, it holds that
\begin{align*}
    \E\brs*{\sum_{h=1}^H\VAR_{\bs'\sim P_h(s_{h})}(V^{T,\pi}_{h}(s_{h},\bs') )\vert\pi,s_1} \leq \E\brs*{\br*{\sum_{h=1}^H r_h(s_h,a_h) - V_1^{T,\pi}(s_1) }^2\vert\pi,s_1}.
\end{align*}
\end{lemma}
\begin{proof}
    We apply the law of total variance in the extended MDP; there, the expected rewards are either $0$ (at odd steps) or $r_h(s_h,a_h)$ (at even steps), so the total expected rewards are $\sum_{h=1}^Hr_h(s_h,a_h)$. Hence, by  \Cref{lemma: ltv no-lookahead},
    \begin{align*}
        \E&\brs*{\br*{\sum_{h=1}^H r_h(s_h,a_h) - V_{1}^\pi(s_1,\bs'_0\vert \Mcal^T) }^2\vert\pi,s_1} \\
        &= \E\brs*{\underbrace{\sum_{h=1}^H\VAR(V^\pi_{2h}(s_{h},\bs'_{h+1}\vert \Mcal^T)\vert (s_h,\bs'_0) )}_{\textrm{Odd steps}}  + \underbrace{\sum_{h=1}^H\VAR(V^\pi_{2h+1}(s_{h+1},\bs'_0\vert \Mcal^T)\vert (s_h,\bs'_{h+1}) )}_{\textrm{Even steps}}\vert\pi,s_1} \\
        & \geq\E\brs*{\sum_{h=1}^H\VAR(V^\pi_{2h}(s_{h},\bs_{h+1}\vert \Mcal^T)\vert (s_h,\bs'_0) )\vert\pi,s_1} \\
        & = \E\brs*{\sum_{h=1}^H\VAR_{\bs'\sim P_h(s_h)}(V^\pi_{2h}(s_{h},\bs'\vert \Mcal^T) )\vert\pi,s_1} \\
        & = \E\brs*{\sum_{h=1}^H\VAR_{\bs'\sim P_h(s_h)}(V^{T,\pi}_{h}(s_{h},\bs'\vert \Mcal) )\vert\pi,s_1}.
    \end{align*}
    Using again the identity $V_{1}^\pi(s_1,\bs'_0\vert \Mcal^T) = V_{1}^{T,\pi}(s_1\vert \Mcal)$ leads to the desired result.
\end{proof}

Finally, prove a value-difference lemma also for transition lookahead
\begin{lemma}[Value-Difference Lemma with Transition Lookahead]
    Let $\Mcal_1=(\Scal,\Acal,H,P^1,\Rcal^1)$ and $\Mcal_2=(\Scal,\Acal,H,P^2,\Rcal^2)$ be two environments. For any deterministic one-step transition lookahead policy $\pi\in\Pi^T$, any $h\in[H]$ and $s\in\Scal$, it holds that
    \begin{align*}
        V_h^{T,\pi}(s\vert \Mcal_1) &-  V_h^{T,\pi}(s\vert \Mcal_2) \\
        &=\E_{\Mcal_1}\brs*{r^1_h(s_h,\pi_h(s_h,\bs'_{h+1})) - r^2_h(s_h,\pi_h(s_h,\bs'_{h+1}))\vert s_h=s} \\
        &\quad+ \E_{\Mcal_1}\brs*{V_{h+1}^{T,\pi}(s_{h+1}\vert \Mcal_1) -  V_{h+1}^{T,\pi}(s_{h+1}\vert \Mcal_2)\vert s_h=s}\\
        &\quad + \E_{\Mcal_1}\brs*{\E_{\bs'\sim P^1_h(s_h)}\brs*{V_{h}^{T,\pi}(s_h,\bs'\vert \Mcal_2)} - \E_{\bs'\sim P^2_h(s_h)}\brs*{V_{h}^{T,\pi}(s_h,\bs'\vert \Mcal_2)}\vert s_h=s}.
    \end{align*}
    where $V_h^{T,\pi}(s,\bs'\vert \Mcal) $ is the value at a state given the reward realization, defined in \Cref{eq: extended transition MDP to standard value relation} and given in \Cref{remark: transition value with next-state observations}.
\end{lemma}
\begin{proof}
    We again work with the extended MDPs $\Mcal_1^T,
    \Mcal_2^T$ and use their Bellman equations, namely,
    \begin{align*}
        &V_{2h}^{\pi}(s,\bs'\vert \Mcal^T) = r_h(s,\pi(s,\bs'))+V_{2h+1}^*(s'(\pi_h(s,\bs')), \bs'_0\vert \Mcal^T), &\forall h\in[H],s\in\Scal, \bs'\in\Scal^A,\nonumber\\
        &V_{2h-1}^{\pi}(s, \bs'_0\vert \Mcal^T) = \E_{\bs'\sim P_h(s)}\brs*{V_{2h}^{\pi}(s,\bs'\vert \Mcal^T)}, &\forall h\in[H],s\in\Scal.
    \end{align*}
    Using the relation between the value of the original and extended MDP (\cref{eq: extended transition MDP to standard value relation}) and the Bellman equations of the extended MDP, for any $h\in[H]$, we have
    \begin{align}
    \label{eq: transition value diff basic decomp}
        &V_h^{T,\pi}(s\vert \Mcal_1) -  V_h^{T,\pi}(s\vert \Mcal_2) \nonumber\\
        & = V_{2h-1}^\pi(s, \bs'_0\vert \Mcal^T_1) -  V_{2h-1}^\pi(s, \bs'_0\vert \Mcal^T_2)  \nonumber\\
        & = \E_{\bs'\sim P^1_h(s)}\brs*{V_{2h}^{\pi}(s,\bs'\vert \Mcal^T_1)} - \E_{\bs'\sim P^2_h(s)}\brs*{V_{2h}^{\pi}(s,\bs'\vert \Mcal^T_2)}  \nonumber\\
        & = \E_{\bs'\sim P^1_h(s)}\brs*{V_{2h}^{\pi}(s,\bs'\vert \Mcal^T_1) - V_{2h}^{\pi}(s,\bs'\vert \Mcal^T_2)} + \E_{\bs'\sim P^1_h(s)}\brs*{V_{2h}^{\pi}(s,\bs'\vert \Mcal^T_2)} - \E_{\bs'\sim P^2_h(s)}\brs*{V_{2h}^{\pi}(s,\bs'\vert \Mcal^T_2)}  \nonumber\\
        & = \E_{\bs'\sim P^1_h(s)}\brs*{V_{2h}^{\pi}(s,\bs'\vert \Mcal^T_1) - V_{2h}^{\pi}(s,\bs'\vert \Mcal^T_2)} + \E_{\bs'\sim P^1_h(s)}\brs*{V_{h}^{T,\pi}(s,\bs'\vert \Mcal_2)} - \E_{\bs'\sim P^2_h(s)}\brs*{V_{h}^{T,\pi}(s,\bs'\vert \Mcal_2)} \nonumber\\
        & = \E_{\Mcal_1}\brs*{V_{2h}^{\pi}(s_h,\bs'_{h+1}\vert \Mcal^T_1) - V_{2h}^{\pi}(s_h,\bs'_{h+1}\vert \Mcal^T_2)\vert s_h=s} \nonumber\\
        &\quad+ \E_{\bs'\sim P^1_h(s)}\brs*{V_{h}^{T,\pi}(s,\bs'\vert \Mcal_2)} - \E_{\bs'\sim P^2_h(s)}\brs*{V_{h}^{T,\pi}(s,\bs'\vert \Mcal_2)} .
    \end{align}
    Denoting $a_h=\pi_h(s_h,\bs'_{h+1})$ the action taken by the agent at environment $\Mcal_1$, We have
    \begin{align*}
        V_{2h}^{\pi}&(s_h,\bs'_{h+1}\vert \Mcal^T_1) - V_{2h}^{\pi}(s_h,\bs'_{h+1}\vert \Mcal^T_2) \\
        & = \br*{r^1_h(s_h,a_h)+ V_{2h+1}^{\pi}(s_{h+1}'(a_h), \bs'_0\vert \Mcal^T_1)} - \br*{r^2_h(s_h,a_h)+V_{2h+1}^{\pi}(s_{h+1}'(a_h), \bs'_0\vert \Mcal^T_2)} \\
        & = r^1_h(s_h,a_h) - r^2_h(s_h,a_h) +  V_{h+1}^{T,\pi}(s'_{h+1}(a_h)\vert \Mcal_1) -  V_{h+1}^{T,\pi}(s'_{h+1}(a_h)\vert \Mcal_2),
    \end{align*}
    when taking the expectation w.r.t. $\Mcal_1$, it holds that $s'_{h+1}(a_h)=s_{h+1}$; substituting this back into \Cref{eq: transition value diff basic decomp}, we get 
    \begin{align*}
        V_h^{\pi}&(s\vert \Mcal_1) -  V_h^{\pi}(s\vert \Mcal_2) \\
        & = \E_{\Mcal_1}\brs*{r^1_h(s_h,a_h) - r^2_h(s_h,a_h) +  V_{h+1}^{T,\pi}(s'_{h+1}(a_h)\vert \Mcal_1) -  V_{h+1}^{T,\pi}(s'_{h+1}(a_h)\vert \Mcal_2)\vert s_h=s} \nonumber\\
        &\quad+ \E_{\bs'\sim P^1_h(s)}\brs*{V_{h}^{T,\pi}(s,\bs'\vert \Mcal_2)} - \E_{\bs'\sim P^2_h(s)}\brs*{V_{h}^{T,\pi}(s,\bs'\vert \Mcal_2)} \\
        & = \E_{\Mcal_1}\brs*{r^1_h(s_h,\pi_h(s_h,\bs'_{h+1})) - r^2_h(s_h,\pi_h(s_h,\bs'_{h+1}))\vert s_h=s} \\
        &\quad+ \E_{\Mcal_1}\brs*{V_{h+1}^{T,\pi}(s_{h+1}\vert \Mcal_1) -  V_{h+1}^{T,\pi}(s_{h+1}\vert \Mcal_2)\vert s_h=s}\\
        &\quad + \E_{\Mcal_1}\brs*{\E_{\bs'\sim P^1_h(s_h)}\brs*{V_{h}^{T,\pi}(s_h,\bs'\vert \Mcal_2)} - \E_{\bs'\sim P^2_h(s_h)}\brs*{V_{h}^{T,\pi}(s_h,\bs'\vert \Mcal_2)}\vert s_h=s}.
    \end{align*}   
\end{proof}

\clearpage

\subsection{Full Algorithm Description for Transition Lookahead}
\label{appendix: MVP for transition lookahead}

\begin{algorithm}[ht]
\caption{Monotonic Value Propagation with Transition Lookahead (MVP-TL)} \label{alg: MVP transition lookahead}
\begin{algorithmic}[1]
\STATE {\bf Require:} $\delta\in(0,1)$, bonuses $b_{k,h}^r(s,a), b_{k,h}^p(s)$
\FOR{$k=1,2,...$}
    \STATE  Initialize $\bar{V}^k_{H+1}(s)=0$ 
    \FOR{$h=H,H-1,..,1$}   
         \FOR{$s\in\Scal$}
            \IF{$n_h^{k-1}(s)=0$}
                \STATE $\bar{V}^k_h(s) = H$
            \ELSE
                \STATE Calculate the truncated values \label{algline:VI transition lookahead}
                {\small\begin{align*}
                    &\bar{V}^k_h(s) = \min\brc*{\frac{1}{n_h^{k-1}(s)} \sum_{t=1}^{n_h^{k-1}(s)}\max_{a\in\Acal}\brc*{\hat{r}_h^{k-1}(s,a) + b_{k,h}^{r}(s,a) +\bar{V}^k_{h+1}(s'^{k^t_h(s)}_{h+1}(s,a))} + b^p_{k,h}(s), H} 
                \end{align*}}
            \ENDIF
            \STATE For any set of next-states $\bs'\in\Scal^A$, define the policy $\pi^k$ 
            \begin{align*}
                \pi_h^k(s,\bs') \in\argmax_{a\in\Acal}\brc*{\hat{r}_h^{k-1}(s,a) + b_{k,h}^{r}(s,a) +\bar{V}^k_{h+1}(s'(a))}
            \end{align*}
        \ENDFOR
    \ENDFOR
    \FOR{$h=1,2,\dots H$}
        \STATE Observe $s_h^k$ and $\bs'^k_{h+1} = \brc*{s'^k_{h+1}(s_h^k,a)}_{a\in\Acal}$
        \STATE Play an action 
        $a_h^k=\pi_h^k(s_h^k,\bs'^k_h)$
        \STATE Collect the reward $R^k_h\sim\Rcal_h(s_h^k,a_h^k)$ and transition to the next state $s^k_{h+1}=s'^k_{h+1}(s_h^k,a_h^k)$
    \ENDFOR
    \STATE Update the empirical estimators and counts for all visited state-actions
\ENDFOR
\end{algorithmic}
\end{algorithm}
As with reward lookahead, we again use a variant of the MVP algorithm \citep{zhang2021reinforcement}, described in \Cref{alg: MVP transition lookahead}. For the bonuses, we use the notation
\begin{align*}
    \bar{V}^k_h(s,\bs') = \max_{a\in\Acal}\brc*{\hat{r}_h^{k-1}(s,a) + b_{k,h}^{r}(s,a) +\bar{V}^k_{h+1}(s'(a)}
\end{align*}
and define the following bonuses:\begin{align*}
    &b_{k,h}^r(s,a) = \min\brc*{\sqrt{ \frac{L^k_{\delta} }{n^{k-1}_{h}(s,a)\vee 1}},1},\\
    &b_{k,h}^{p}(s) = \frac{20}{3}\sqrt{\frac{\VAR_{\bs'\sim\hat{P}^{k-1}_{h}(s)}(\bar{V}^k_{h}(s,\bs')) L^k_{\delta}}{n^{k-1}_{h}(s)\vee 1}} + \frac{400}{3}\frac{H L^k_{\delta}}{n^{k-1}_{h}(s) \vee 1},  
\end{align*}
where $L^k_{\delta}=\ln\frac{16S^3A^2Hk^2(k+1)}{\delta}$ and
\begin{align*}
    \VAR_{\bs'\sim\hat{P}^{k-1}_{h}(s)}(\bar{V}^k_{h}(s,\bs')) = \E_{\bs'\sim\hat{P}^{k-1}_h(s)}\brs*{\bar{V}^k_{h}(s,\bs')^2} - \br*{\E_{\bs'\sim\hat{P}^{k-1}_h(s)}\brs*{\bar{V}^k_{h}(s,\bs')}}^2.
\end{align*}
The notation $k_h^t(s)$ again represents the $t^{th}$ episode where the state $s$ was visited at the $h^{th}$ timestep; in particular, line \ref{algline:VI transition lookahead} of the algorithm is the expectation w.r.t. the empirical reward distribution $\hat{P}^{k-1}_h(s)$. Since the transition bonus is larger than $H$ when $n^{k-1}_{h}(s)=0$, we can arbitrarily define the expectation w.r.t. $\hat{P}^{k-1}_h(s)$ when $n^{k-1}_{h}(s)=0$ to be 0, and one could write the update in a more concise way as
\begin{align*}
    &\bar{V}^k_h(s) = \min\brc*{\E_{\bs'\sim\hat{P}^{k-1}_h(s)}\brs*{\bar{V}^k_h(s,\bs')} + b^p_{k,h}(s), H}.
\end{align*}

\clearpage

\subsection{Additional Notations and List Representation}
\label{appendix: transition lookahead list representation}
In this subsection, we present additional notations for both values and transition distributions that will be helpful in the analysis. In particular, we show that instead of looking at the distribution over all combinations of next state $\bs'\in\Scal^A$, we can look at a ranking of all the next-state-actions and represent important quantities using the effective distribution on these ranks -- this moves the problem from being $S^A$-dimensional to a dimension of $SA$.

We start by defining the values starting from state $s\in\Scal$, playing $a\in\Acal$ and transitioning to $s'\in\Scal$, denoted by
\begin{align*}
    &V^{\pi}_h(s,s',a) = r_h(s,a) +V^{\pi}_{h+1}(s'),\\
    &V^{*}_h(s,s',a) = r_h(s,a) +V^{*}_{h+1}(s'), \\
    & \bar{V}^k_h(s,s',a) = \hat{r}_h^{k-1}(s,a) + b_{k,h}^{r}(s,a) +\bar{V}^k_{h+1}(s'),
\end{align*}
We similarly define (consistently with \Cref{remark: transition value with next-state observations})
\begin{align*}
    &V^{\pi}_h(s,\bs') = V^{\pi}_h(s,s'(\pi_h(s,\bs')),\pi_h(s,\bs')),\\
    &V^{*}_h(s,\bs') = \max_aV^{*}_h(s,s'(a),a),\qquad\qquad\quad \textrm{ and },\\
    &\bar{V}^k_h(s,\bs') = \max_a \bar{V}^k_h(s,s'(a),a).
\end{align*}

\textbf{List representation.} We now move to defining lists of next-state-actions and distributions with respect to such lists. Let $\ell$ be a list that orders all next-state-action pairs from $(s'_{\ell(1)},a_{\ell(1)})$ to $(s'_{\ell(SA)},a_{\ell(SA)})$ and define the set of all possible lists to be $\Lcal$ (with $\abs{\Lcal}=(SA)!$). Also, define $\ell^u$, the list induced by a function $u:\Scal\times\Acal\mapsto \R$ such that $u(s'_{\ell^u(1)},a_{\ell^u(1)})\ge\dots\ge u(s'_{\ell^u(SA)},a_{\ell^u(SA)})$, where ties are broken in any fixed arbitrary way.  From this point forward, for brevity and when clear from the context, we omit the list from the indexing, e.g., write the list $\ell$ by $(s'_1,a_1),\dots,(s'_{SA},a_{SA})$.

We now define the probability of list elements. Denote by $E^{\ell}_i$ the event that the highest-ranked realized element in the list is element $i$, namely
\begin{align}
    \label{eq: list probability}
    E^{\ell}_i=\brc*{\bs'\in\Scal^A : s'(a_i)=s'_i \; \textrm{ and } \; \forall j<i, s'(a_j)\ne s'_j}.
\end{align}
Then, for a probability measure $P$ on $\Scal^A$, define $\mu(i\vert \ell, P)=P(\bs'\in E^{\ell}_i)$. Notably, when the list is induced by $u$ and element $i$ is the realized highest-ranked elements, we can write $\max_au(s'(a),a)=u(s'_i,a_i)$, so we have that (e.g. by \Cref{lemma: empirical expectation of subsets} with $f(\bs')=\max_au(s'(a),a)$)
\begin{align*}
    \E_{\bs'\sim P_h(s)}\brs*{\max_a\brc*{u(s'(a),a)}} = \E_{i\sim\mu(\cdot\vert \ell,  P_h(s))}\brs*{u(s'_i,a_i)}
\end{align*}
We also denote by $\hat{\mu}^{k}_{h}(i\vert s; \ell)=\frac{1}{n_h^{k}(s)\vee1}\sum_{t=1}^K\Ind{s_h^t=s,\bs'^t_{h+1}\in E_i^{\ell}}$, the empirical probability for a list location $i$ to be the highest-realized ranking according to a list $\ell$ at state $s$ and step $h$, based on samples up to episode $k$; We have by \Cref{lemma: empirical expectation of subsets} that $\hat{\mu}^{k}_{h}(i\vert s; \ell)=\hat{P}^{k}_{h}(E^{\ell}_i\vert s)$ and
\begin{align*}
    \E_{\bs'\sim \hat{P}^{k-1}_h(s)}\brs*{\max_a\brc*{u(s'(a),a)}} = \E_{i\sim\hat{\mu}^{k-1}_{h}(\cdot\vert s; \ell^u)}\brs*{u(s'_i,a_i)}.
\end{align*}
Similarly, we will require the distribution probability w.r.t. two lists -- the probability that the top element w.r.t. list $\ell$ is $i$ and the top element w.r.t. list $\ell'$ is $j$; we denote the real and empirical probability distributions by $\mu(i,j\vert \ell,\ell',P)$ and $\hat{\mu}^{k}_{h}(i,j\vert s; \ell,\ell')$, respectively. This allows, for example, using \Cref{lemma: empirical expectation of subsets} to write for any $u,v:\Scal\times\Acal\mapsto \R$,
\begin{align}
     &\E_{\bs'\sim P_h(s)}\brs*{\max_a\brc*{u(s'(a),a)} - \max_a\brc*{v(s'(a),a)}} \nonumber\\
     &\hspace{12em}=\E_{i,j\sim\mu(\cdot\vert \ell^u,\ell^v, P_h(s))}\brs*{u(s'_{\ell^u(i)},a_{\ell^u(i)}) - v(s'_{\ell^v(j)},a_{\ell^v(j)})},\nonumber\\
     &\E_{\bs'\sim \hat{P}^{k-1}_h(s)}\brs*{\max_a\brc*{u(s'(a),a)} - \max_a\brc*{v(s'(a),a)}} \nonumber\\
     &\hspace{12em}=\E_{i,j\sim\hat{\mu}^{k}_{h}(\cdot\vert s;\ell^u,\ell^v)}\brs*{u(s'_{\ell^u(i)},a_{\ell^u(i)}) - v(s'_{\ell^v(j)},a_{\ell^v(j)})}.
     \label{eq: diff in list representation}
\end{align}
Finally, we say that a policy $\pi_h(s,\bs')$ is induced by lists $\ell_h(s)$ if it chooses an action $a$ such that its next-state $s'(a)$ is ranked higher in $\ell$ than all other realized next-state-action pairs. In particular, the policy $\pi^k$ and the optimal policy $\pi^*$ (defined in \Cref{prop: transition DP}) are such policies w.r.t. the lists $\bar{\ell}^k_h(s)$ and $\ell^*_h(s)$ -- induced by $\bar{V}_h^k(s,s',a)$ and $V_h^*(s,s',a)$, respectively. As such, for any probability measure $P_h(s)$, function $u:\Scal\times\Scal\times\Acal\mapsto \R$ and a policy $\pi$ induced by a list $\ell$,  it holds that
\begin{align}
    \label{eq:dist to list equivalence}
    \E_{\bs'\sim P_h(s)}\brs*{u(s,s'(\pi(a)),\pi(a))} = \E_{i\sim\mu(\cdot\vert\ell_h(s), P_h(s))}\brs*{u(s,s'_i,a_i)}.
\end{align}



\subsubsection{Planning with Transition Lookahead}
\label{appendix: transition lookahead planning}
We have already seen the optimal policy is induced by a list $\ell^*_h(s)$, and in particular, we can write the dynamic programming equations of \Cref{prop: transition DP} as
\begin{align*}
    V^{*}_h(s) 
    &= \E_{\bs'\sim P_h(s)}\brs*{\max_{a\in\Acal}\brc*{r_h(s,a) +V^{T,*}_h(s'(a))}}\\
    &=\E_{i\sim\mu(\cdot\vert\ell^*_h(s), P_h(s))}\brs*{r_h(s,a_i) +V^{*}_{h+1}(s'(a_i))}.
\end{align*}
Therefore, one way to perform the planning is to build a list $\ell^*_h(s)$ of $(s',a)$ s.t. the values 
\begin{align*}
    V^{*}_h(s,s',a) = r_h(s,a) +V^{*}_{h+1}(s')
\end{align*}
are sorted in a non-increasing order and calculate the probability of any pair in the list to be the highest-realized pair:
\begin{align*}
    \mu(i\vert\ell, P_h(s)) = P_h(E_i^{\ell})
    = \Pr\br*{s'_{h+1}(a_i)=s'_i \; \textrm{ and } \; \forall j<i, s'_{h+1}(a_j)\ne s'_j\vert s_h=s}.
\end{align*}
In general, calculating this distribution is intractable, and one must resort to approximating it by sampling (as done in \Cref{alg: MVP transition lookahead}. Nonetheless, if next states are generated independently between actions, this distribution could be efficiently calculated as follows:
\begin{align*}
   \mu(i\vert\ell, P_h(s)) 
   &= \Pr\br*{s'_{h+1}(a_i)=s'_i \; \textrm{ and } \; \forall j<i, s'_{h+1}(a_j)\ne s'_j\vert s_h=s} \\
    & \overset{(1)}= \Pr\brc*{s'(a_i)=s'_i \; \textrm{ and } \; \forall j<i \textrm{ s.t. }a_j\ne a_i, s'(a_j)\ne s'_j\vert s_h=s} \\
    & \overset{(2)}= \Pr\brc*{s'(a_i)=s'_i\vert s_h=s}\prod_{a\ne a_i}\Pr\brc*{ \forall j<i \textrm{ s.t. }a_j=a, s'(a)\ne s'_j\vert s_h=s} \\
    & \overset{(3)}= P_h(s'_i\vert s,a_i) \prod_{a\ne a_i}\br*{1 - \sum_{j=1}^{i-1}\Ind{a_j=a}P_h(s'_j\vert s,a)}.
\end{align*}
Relation $(1)$ holds since if $s'(a_i)=s'_i$, it cannot get any previous value of the same action in the list, so these events can be removed. Relation $(2)$ is by the independence and $(3)$ directly calculates the probabilities.

\clearpage

\subsection{The First Good Event -- Concentration}
\label{appendix: transition lookahead concentration event}
Next, we define the events that ensure the concentration of all empirical measures. For rewards, an event handles the convergence of the empirical rewards to their mean. For the transitions, we want the Bellman operator, applied on the optimal value with the empirical model, to concentrate well, and we require the variance of values w.r.t. the empirical and real model to be close. Finally, the empirical measure $\hat{\mu}^{k}_{h}(i,j\vert s; \ell,\ell^*_h(s))$ must concentrate well around its mean for any list $\ell$ -- this will allow the change-of-measure argument described in the proof sketch.

Formally, define the following good events:
{\small
\begin{align*}
    &E^r(k) = \brc*{\forall s,a,h:\ |r_h\br*{s,a} - \hat{r}^{k-1}_h\br*{s,a}| \le \sqrt{ \frac{L^k_{\delta} }{n^{k-1}_{h}(s,a)\vee 1}}} \\
    &E^{\ell}(k) = \left\{\forall s,h, \forall \ell\in\Lcal, \forall i,j\in\brs*{SA}:\ \abs*{\hat{\mu}^{k-1}_h\br*{i,j|s; \ell,\ell^*_h(s)} - \mu\br*{i,j|\ell,\ell^*_h(s); P_h(s)}}\right.\\
    &\hspace{26.25em}\left.\le \sqrt{\frac{4SAL^k_{\delta}\mu\br*{i,j|s; \ell,\ell^*_h(s); P_h(s)}}{n^{k-1}_h(s)\vee 1}} + \frac{2SAL^k_{\delta}}{n^{k-1}_h(s)\vee 1}\right\} \\
    &E^{pv1}(k)=\brc*{\forall s,h:\ \abs*{\E_{\bs'\sim P_h(s)}\brs*{V^*_h(s,\bs')} - \E_{\bs'\sim\hat{P}^{k-1}_h(s)}\brs*{V^*_h(s,\bs')}} \leq \sqrt{\frac{2\VAR_{\bs'\sim P_{h}(s)}(V^*_{h}(s,\bs')) L^k_{\delta}}{n^{k-1}_{h}(s)\vee 1}} + \frac{HL^k_{\delta}}{n^{k-1}_h(s)\vee 1}}\\
    &E^{pv2}(k)=\brc*{\forall s,h:\ \abs*{ \sqrt{\VAR_{\bs'\sim P_{h}(s)}(V^*_{h}(s,\bs'))} -  \sqrt{\VAR_{\bs'\sim \hat{P}^{k-1}_{h}(s)}(V^*_{h}(s,\bs'))} } \leq 4H\sqrt{\frac{L^k_{\delta}}{n^{k-1}_h(s)\vee 1}} }
\end{align*}
}
where we again use $L^k_{\delta}=\ln\frac{16S^3A^2Hk^2(k+1)}{\delta}$. We define the first good event as 
$$\G_1 = \bigcap_{k\geq 1} E^r(k)\bigcap_{k\geq 1} E^{\ell}(k) \bigcap_{k\geq 1} E^{pv1}(k) \bigcap_{k\geq 1} E^{pv2}(k),$$
for which the following holds: 
\begin{lemma}[The First Good Event]\label{lemma: the first good event transition lookahead}
It holds that $\Pr(\G_1)\geq 1-\delta/2$.
\end{lemma}
\begin{proof}
    We prove that each of the events holds w.p. at least $1-\delta/8$. The result then directly follows by the union bound. We also remark that due to the domain of the variables and their estimators (e.g., $[0,1]$ for the rewards), all bounds trivially hold when the counts equal zero, so w.l.o.g., we only prove the results for cases in which states/state-actions were already previously visited. 

    \textbf{Event $\cap_{k\ge1}E^r(k)$.} Fix $k\ge1,s,a,h$ and visits $n\ge1$. Given all of these, the reward observations are i.i.d. random variables supported by $[0,1]$. Denoting the empirical mean based on these $n$ samples by $\hat{r}_h(s,a,n)$, by Hoeffding's inequality, it holds w.p. $1-\frac{\delta}{8SAHk^2(k+1)}$ that
    \begin{align*}
        \abs*{r_h(s,a) - \hat{r}_h(s,a,n)}\leq \sqrt{\frac{\ln\frac{16SAHk^2(k+1)}{\delta}}{2n}} \leq \sqrt{ \frac{L^k_{\delta} }{n}}.
    \end{align*}
    Taking the union bound over all $n\in[k]$ at timestep $k$, we get that w.p. $1-\frac{\delta}{8SAHk(k+1)}$ 
    \begin{align*}
        |r_h\br*{s,a} - \hat{r}^{k-1}_h\br*{s,a}| \le \sqrt{ \frac{L^k_{\delta} }{n^{k-1}_{h}(s,a)\vee 1}},
    \end{align*}
    and another union bound over all possible values of $s,a,h$ and $k\ge1$ implies that $\cap_{k\ge1}E^r(k)$ holds w.p. at least $1-\delta/8$.
    
    \textbf{The event $\cap_{k\ge1}E^{\ell}(k)$.} For  any fixed $k\ge1, s, h$, a list $\ell\in\Lcal$ and number of visits $n\in[k]$, we utilize \Cref{lemma: concentration results} (event $E^p$) w.r.t. the distribution $\mu(i,j\vert \ell,\ell_h^*(s),P)$ (whose support is of size $M=(SA)^2$). When applying the lemma, notice that given the number of visits $n\ge1$, the empirical distribution $\hat{\mu}^{k-1}_{h}(i,j\vert s; \ell,\ell_h^*(s))$ is the average of $n=n_h^{k-1}(s)$ i.i.d samples, so that for all $i,j\in[SA]$,
    {\small\begin{align*}
       \abs*{\hat{\mu}^{k-1}_h\br*{i,j|s; \ell,\ell^*_h(s)} - \mu\br*{i,j| \ell,\ell^*_h(s); P_h(s)}}
       &\le \sqrt{\frac{2\mu\br*{i,j|\ell,\ell^*_h(s); P_h(s)}\ln\frac{2(SA)^2}{\delta'}}{n}} + \frac{2\ln\frac{2(SA)^2}{\delta'}}{3n}\\
       &\le \sqrt{\frac{4\mu\br*{i,j|\ell,\ell^*_h(s); P_h(s)}\ln\frac{2SA}{\delta'}}{n}} + \frac{2\ln\frac{2SA}{\delta'}}{n}
    \end{align*}}
    w.p. $1-\delta'$. Choosing $\delta'=\frac{\delta}{8\abs{\Lcal}SHk^2(k+1)}$ (such that $\ln\frac{2SA}{\delta'} \leq SA\ln\frac{16S^3A^2Hk^2(k+1)}{\delta}$ since $\abs{\Lcal}\leq (SA)^{SA}$), while taking the union bound on all $n\in[k]$, all $s,h$ and all lists $\ell\in\Lcal$ implies that $\cap_{k\ge1}E^{\ell}(k)$ holds w.p. at least $1-\frac{\delta}{8}$. 

    \textbf{Events $\cap_{k\ge1}E^{pv1}(k)$ and $\cap_{k\ge1}E^{pv2}(k)$.} We repeat the arguments stated in \Cref{lemma: the first good event reward lookahead}. For any fixed $k\ge1, s,h$ and number of visits $n\in[k]$ , we utilize \Cref{lemma: concentration results} w.r.t. the next-state distribution for all actions $P_h(s)$, the value $V^*_{h}(s,\bs')\in[0,H]$ and probability $\delta'=\frac{\delta}{8SHk^2(k+1)}$; we yet again remind that given the number of visits, samples are i.i.d.

    As before, the events $\cap_{k\geq 1} E^{pv1}(k)$ and $\cap_{k\geq 1} E^{pv2}(k)$ hold w.p. at least $1-\frac{\delta}{8}$ through the union bound first on $n\in[k]$ (to get the empirical quantities) and then on $s,h$ and $k\ge1$. This proves that each of the events in $\G_1$ holds w.p. at least $1-\frac{\delta}{8}$, so  $\G_1$ holds w.p. at least $1-\frac{\delta}{2}$.    
\end{proof}

\clearpage


\subsection{Optimism of the Upper Confidence Value Functions}
We now prove that under the event $\G_1$, the values that MVP-TL outputs are optimistic.

\begin{lemma}[Optimism] \label{lemma: optimism of values MVP-TL}
Under the first good event $\G_1$, for all $k\in[K]$, $h\in [H]$, $a\in\Acal$ and $s,s'\in \mathcal{S}$, it holds that $V^*_h(s,s',a)\leq \bar{V}^k_{h}(s,s',a)$. Moreover, for all $\bs'\in\Scal^A$, $V^*_h(s,\bs')\leq \bar{V}^k_{h}(s,\bs')$ and also $V^*_h(s)\leq \bar{V}^k_{h}(s)$.
\end{lemma}
\begin{proof}
    The proof of all claims follows by backward induction on $H$; the base case naturally holds for $h=H+1$, where all values are defined to be zero. 
    
    Assume by induction that for some $k\in[K]$ and $h\in[H]$, the inequality $V^*_{h+1}(s)\leq \bar{V}^k_{h+1}(s)$ holds for all $s\in\Scal$; we will show that this implies that all stated inequalities also hold at timestep $h$. 
    At this point, we also assume w.l.o.g. that $\bar{V}^k_{h}(s)< H$ (namely, not truncated), since otherwise, by the boundedness of the rewards, $V^*_h(s)\leq H=\bar{V}^k_{h}(s).$ In particular, under the good event $E^r(k)$, for all $s$ and $a$ , it holds that $\hat{r}_h^{k-1}(s,a) + b_{k,h}^{r}(s,a)\geq r_h(s,a)$, so for all $s,a$ and $s'$, we have
    \begin{align*}
        \bar{V}^k_{h}(s,s',a) = \hat{r}_h^{k-1}(s,a) + b_{k,h}^{r}(s,a) +\bar{V}^k_{h+1}(s')
         \geq r_h(s,a) +V^*_{h+1}(s') 
         = V^*_h(s,s',a).
    \end{align*}
    where the inequality also uses the induction hypothesis. This proves the first part of the lemma. Moreover, it implies that 
    \begin{align}
        \bar{V}^k_h(s,\bs') 
        = \max_{a\in\Acal}\brc*{\bar{V}^k_{h}(s,s'(a),a)}
        \geq \max_{a\in\Acal}\brc*{V^*_h(s,s'(a),a)}
        =V^*_h(s,\bs'), \label{eq: value monotonicity transition lookahead}
    \end{align}
    and proves the second part of the statement.
    
    To prove the last claim of the lemma, we use the monotonicity of the bonus, relying on \Cref{lemma: bonus monotonicity}. This lemma can be used when applied to the empirical distribution of all possible next-states $\hat{P}_h^{k-1}(s)$; indeed, the non-truncated optimistic value can be written as
    \begin{align*}
        \bar{V}^k_{h}(s)
        &= \E_{\bs'\sim\hat{P}^{k-1}_h(s)}\brs*{\max_{a\in\Acal}\brc*{\hat{r}_h^{k-1}(s,a) + b_{k,h}^{r}(s,a) +\bar{V}^k_{h+1}(s'(a))}} + b^p_{k,h}(s) \\
        & \geq \E_{\bs'\sim\hat{P}^{k-1}_h(s)}\brs*{\bar{V}^k_h(s,\bs')} + \max\brc*{\frac{20}{3}\sqrt{\frac{\VAR_{\bs'\sim\hat{P}^{k-1}_{h}(s)}(\bar{V}^k_{h}(s,\bs')) L^k_{\delta}}{n^{k-1}_{h}(s)\vee 1}},\frac{400}{9}\frac{3H L^k_{\delta}}{n^{k-1}_{h}(s) \vee 1}},
    \end{align*}
    which is exactly the required form in \Cref{lemma: bonus monotonicity}, w.r.t. the distribution $\hat{P}^{k-1}_h(s)$ and the values $\bar{V}^k_h(s,\bs')$ (while noticing that due to the truncation of the values and bonuses, $\bar{V}^k_h(s,\bs')\in[0,3H]$). Thus, the lemma guarantees monotonicity in the value, so by \Cref{eq: value monotonicity transition lookahead}, 
    {\small
    \begin{align*}
        \bar{V}^k_{h}(s)
        &\geq \E_{\bs'\sim\hat{P}^{k-1}_h(s)}\brs*{V^*_h(s,\bs')} + \max\brc*{\frac{20}{3}\sqrt{\frac{\VAR_{\bs'\sim\hat{P}^{k-1}_{h}(s)}(V^*_h(s,\bs')) L^k_{\delta}}{n^{k-1}_{h}(s)\vee 1}},\frac{400}{9}\frac{3H L^k_{\delta}}{n^{k-1}_{h}(s) \vee 1}} \\
        & \geq  \E_{\bs'\sim\hat{P}^{k-1}_h(s)}\brs*{V^*_h(s,\bs')} + \frac{10}{3}\sqrt{\frac{\VAR_{\bs'\sim\hat{P}^{k-1}_{h}(s)}(V^*_h(s,\bs')) L^k_{\delta}}{n^{k-1}_{h}(s)\vee 1}} + \frac{200}{3}\frac{H L^k_{\delta}}{n^{k-1}_{h}(s) \vee 1} \\
        & \geq  \E_{\bs'\sim\hat{P}^{k-1}_h(s)}\brs*{V^*_h(s,\bs')} + \frac{10}{3}\sqrt{\frac{\VAR_{\bs'\sim P_{h}(s)}(V^*_h(s,\bs')) L^k_{\delta}}{n^{k-1}_{h}(s)\vee 1}} +\frac{50H L^k_{\delta}}{n^{k-1}_{h}(s) \vee 1}\tag{Under $E^{pv2}(k)$} \\
        & \geq  \E_{\bs'\sim P_h(s)}\brs*{V^*_h(s,\bs')}\tag{Under $E^{pv1}(k)$} \\
        & = V^*_{h}(s).
    \end{align*}}
\end{proof}

\clearpage


\subsection{The Second Good Event -- Martingale Concentration}
In this subsection, we present three good events that allow replacing the expectation over the randomizations inside each episode by their realization. 
Let 
\begin{align*}
    &Y^k_{1 ,h} \eqdef \bar{V}_{h+1}^k(s^k_{h+1}) - V_{h+1}^{\pi^k}(s^k_{h+1})\\
    &Y^k_{2, h} = \VAR_{\bs'\sim P_{h}(s_h^k)}(V^{\pi^k}_{h}(s_h^k,\bs'))\\
    &Y^k_{3, h} = b_{k,h}^{r}(s_h^k,a_h^k).
\end{align*}
The second good event is the intersection of the events $\G_2 =E^{\mathrm{diff}} \cap  E^{\VAR}  \cap  E^{br}$ defined as follows.
\begin{align*}
    &E^{\mathrm{diff}}=\brc*{\forall h\in[H], K\geq 1:\ \sum_{k=1}^K \E[Y_{1 ,h}^k|F_{k,h-1}]\leq \br*{1+\frac{1}{2H}} \sum_{k=1}^K Y_{1 ,h}^k + 18H^2 \ln\frac{6HK(K+1)}{\delta}},\\
    &E^{\VAR}= \brc*{ K\geq 1:\  \sum_{k=1}^K \sum_{h=1}^H Y_{2, h}^k\leq 2\sum_{k=1}^K \sum_{h=1}^H\E[Y_{2, h}^k|F_{k-1}]  + 4H^3 \ln\frac{6HK(K+1)}{\delta}},\\
    &E^{br}= \brc*{\forall h\in[H], K\geq 1:\ \sum_{k=1}^K \E[Y_{3, h}^k|F_{k,h-1}]\leq 2\sum_{k=1}^K Y_{3 ,h}^k + 18 \ln\frac{6HK(K+1)}{\delta}},
\end{align*}
We define the good event $\G=\G_1\cap\G_2$.
\begin{lemma}
    \label{lemma: good event transition lookahead}
    The good event $\G$ holds with a probability of at least $1-\delta$.
\end{lemma}
\begin{proof}
    The analysis of the first event follows $E^{\mathrm{diff}}$ exactly as the one of $E^{\mathrm{diff}1}$ in \Cref{lemma: good event reward lookahead}: define $W_k=\Ind{\bar{V}_{h}^k(s) - V_{h}^{\pi^k}(s)\in[0,H], \forall h\in[H], s\in\Scal}$ (which happens a.s. under $\G_1$ due to the optimism in \Cref{lemma: optimism of values MVP-TL} and truncation) and $\tilde{Y}^k_{1 ,h} = W_kY^k_{1 ,h}$, which is bounded in $[0,H]$ and $F_{k,h}$-measurable. The corresponding event w.r.t. this modified variables $\tilde{E}^{\mathrm{diff}}$ then holds w.p. $1-\frac{\delta}{6}$ by \Cref{lemma: consequences of optimism and freedman's inequality}, and as in \Cref{lemma: good event reward lookahead}, we can use the fact that $\G_1\cap\tilde{E}^{\mathrm{diff}}=\G_1\cap E^{\mathrm{diff}}$ to conclude this part of the proof.

    Moving to the second event, since $V^{\pi^k}_{h}(s,\bs')\in[0,H]$, then $\sum_{h=1}^H Y^k_{2, h}\in [0,H^3]$. Therefore, by \Cref{lemma: consequences of optimism and freedman's inequality} (w.r.t. the filtration $F_k$) with $C=H^3$ and any fixed $K$, we get w.p. $1-\frac{\delta}{6HK(K+1)}$ that 
    \begin{align*}
        \sum_{k=1}^K \sum_{h=1}^H Y_{2, h}^k\leq 2\sum_{k=1}^K \sum_{h=1}^H\E[Y_{2, h}^k|F_{k-1}]  + 4H^3 \ln\frac{6HK(K+1)}{\delta}.
    \end{align*}
    Taking the union bound on all possible values of $K\ge1$ proves that $E^{\VAR}$ holds w.p. at least $1-\frac{\delta}{6}$.

    Finally, by definition, we have that $Y^k_{3, h} = b_{k,h}^{r}(s_h^k,a_h^k)\in[0,1]$ and is $F_{k,h}$-measurable. Thus, for any fixed $k\ge1$ and $h\in[H]$, using \Cref{lemma: consequences of optimism and freedman's inequality}, we have w.p. $1-\frac{\delta}{6HK(K+1)}$ that
    \begin{align*}
        \sum_{k=1}^K \E[Y_{3, h}^k|F_{k,h-1}]
        &\leq \br*{1+\frac{1}{2}}\sum_{k=1}^K Y_{3 ,h}^k + 18 \ln\frac{6HK(K+1)}{\delta} 
        \leq 2\sum_{k=1}^K Y_{3 ,h}^k + 18 \ln\frac{6HK(K+1)}{\delta},
    \end{align*}
    so that due to the union bound, $E^{br}$ holds w.p. $1-\frac{\delta}{6}$.

    To conclude, $\G_1$ holds w.p. $1-\frac{\delta}{2}$ (\Cref{lemma: the first good event reward lookahead}) and the events $\tilde{E}^{\mathrm{diff}},  E^{\VAR},  E^{br}$ each hold w.p. $1-\frac{\delta}{6}$. As before, when accounting to the fact that $\tilde{E}^{\mathrm{diff}}$ and $E^{\mathrm{diff}}$ are identical under $\G_1$, the event $G=\G_1\cap\G_2$
    holds w.p. at least $1-\delta$.
\end{proof}

\clearpage


\subsection{Regret Analysis}
\label{appendix: transition lookahead regret}
\MVPTL*
\begin{proof}
    Assume that the event $\G$ holds, which by \Cref{lemma: good event transition lookahead}, happens with probability at least $1-\delta$. In particular, throughout the proof, we use optimism (\Cref{lemma: optimism of values MVP-TL}), which implies that $0\leq V_h^{\pi^k}(s,\bs')\leq V_h^{*}(s,\bs')\leq \bar{V}_h^{k}(s,\bs')\leq 3H$ (the upper bound is also by the truncation), as well as $0\leq V_h^{\pi^k}(s)\leq V_h^{*}(s)\leq \bar{V}_h^{k}(s)\leq H$. 
    
    We first focus on lower-bounding the value of the policy $\pi^k$: by \Cref{remark: transition value with next-state observations}, we have
    \begin{align*}
    V^{\pi^k}_h(s) &= \E_{\bs'\sim P_h(s)}\brs*{r_h(s,\pi^k_h(s,\bs')) +V^{\pi^k}_{h+1}(s'(\pi^k_h(s,\bs')))} \\
    & = \E_{\bs'\sim P_h(s)}\brs*{\hat{r}_h^{k-1}(s,\pi^k_h(s,\bs)) +\bar{V}^{k}_{h+1}(s'(\pi^k_h(s,\bs'))) + b_{k,h}^r(s,\pi^k_h(s,\bs'))} \\
    &\quad+ \E_{\bs'\sim P_h(s)}\brs*{r_h(s,\pi^k_h(s,\bs'))  - \hat{r}_h^{k-1}(s,\pi^k_h(s,\bs')) - b_{k,h}^r(s,\pi^k_h(s,\bs'))} \\
    &\quad+ \E_{\bs'\sim P_h(s)}\brs*{V^{\pi^k}_{h+1}(s'(\pi^k_h(s,\bs'))) - \bar{V}^{k}_{h+1}(s'(\pi^k_h(s,\bs')))} \\
    & \overset{(1)}= \E_{\bs'\sim P_h(s)}\brs*{\max_{a\in\Acal}\brc*{\hat{r}_h^{k-1}(s,a) + \bar{V}^{k}_{h+1}(s'(a)) + b_{k,h}^r(s,a)}} \\
    &\quad+ \E_{\bs'\sim P_h(s)}\brs*{r_h(s,\pi^k_h(s,\bs'))  - \hat{r}_h^{k-1}(s,\pi^k_h(s,\bs)) - b_{k,h}^r(s,\pi^k_h(s,\bs'))} \\
    &\quad+ \E_{\bs'\sim P_h(s)}\brs*{V^{\pi^k}_{h+1}(s'(\pi^k_h(s,\bs'))) - \bar{V}^{k}_{h+1}(s'(\pi^k_h(s,\bs')))} \\
    &\overset{(2)}\geq \E_{\bs'\sim P_h(s)}\brs*{\bar{V}^k_{h}(s,\bs')} 
    - 2\E_{\bs'\sim P_h(s)}\brs*{b_{k,h}^r(s,\pi^k_h(s,\bs'))} \\
    &\quad- \E_{\bs'\sim P_h(s)}\brs*{\bar{V}^{k}_{h+1}(s'(\pi^k_h(s,\bs'))) - V^{\pi^k}_{h+1}(s'(\pi^k_h(s,\bs'))) } 
\end{align*}
where $(1)$ is by the definition of $\pi^k$ and $(2)$  uses the reward concentration event. Thus, we can write
\begin{align}
    \bar{V}^{k}_h(s) - V^{\pi^k}_h(s) 
    &\leq \E_{\bs'\sim \hat{P}^{k-1}_h(s)}\brs*{\bar{V}^k_{h}(s,\bs')} - \E_{\bs'\sim P_h(s)}\brs*{\bar{V}^k_{h}(s,\bs')}+ 2\E_{\bs'\sim P_h(s)}\brs*{b_{k,h}^r(s,\pi^k_h(s,\bs'))}\nonumber\\
    &\quad  +\E_{\bs'\sim P_h(s)}\brs*{\bar{V}^{k}_{h+1}(s'(\pi^k_h(s,\bs'))) - V^{\pi^k}_{h+1}(s'(\pi^k_h(s,\bs'))) }+ b_{k,h}^p(s) \nonumber\\
    & = \underbrace{\E_{\bs'\sim \hat{P}^{k-1}_h(s)}\brs*{\bar{V}^k_{h}(s,\bs') - V^*_{h}(s,\bs')} - \E_{\bs'\sim P_h(s)}\brs*{\bar{V}^k_{h}(s,\bs') - V^*_{h}(s,\bs')} + b_{k,h}^p(s)}_{(i)} \nonumber\\
        &\quad + \underbrace{\E_{\bs'\sim P_h(s)}\brs*{V^*_{h}(s,\bs')} - \E_{\bs'\sim \hat{P}^{k-1}_h(s)}\brs*{V^*_{h}(s,\bs')}}_{(ii)}+ 2\E_{\bs'\sim P_h(s)}\brs*{b_{k,h}^r(s,\pi^k_h(s,\bs'))} \nonumber\\
    &\quad +\E_{\bs'\sim P_h(s)}\brs*{\bar{V}^{k}_{h+1}(s'(\pi^k_h(s,\bs'))) - V^{\pi^k}_{h+1}(s'(\pi^k_h(s,\bs'))) }
    \label{eq: value-difference transition lookahead}
\end{align}

\textbf{Bounding term $(ii)$:} using the concentration event $E^{pv1}(k)$, we have
\begin{align}
    (ii) &\leq \sqrt{\frac{2\VAR_{\bs'\sim P_{h}(s)}(V^*_{h}(s,\bs')) L^k_{\delta}}{n^{k-1}_{h}(s)\vee 1}} + \frac{HL^k_{\delta}}{n^{k-1}_h(s)\vee 1} \nonumber\\
    & \overset{(1)}\leq \sqrt{\frac{2\VAR_{\bs'\sim P_{h}(s)}(V^{\pi^k}_{h}(s,\bs')) L^k_{\delta}}{n^{k-1}_{h}(s)\vee 1}} + \frac{1}{8H}\E_{\bs'\sim P_{h}(s)}\brs*{V^{\pi^k}_{h}(s,\bs') - V^{\pi_k}_{h}(s,\bs')} + \frac{4H^2L^k_{\delta}}{n^{k-1}_h(s)\vee 1}  + \frac{HL^k_{\delta}}{n^{k-1}_h(s)\vee 1} \nonumber\\
    & \overset{(2)}\leq \sqrt{\frac{2\VAR_{\bs'\sim P_{h}(s)}(V^{\pi^k}_{h}(s,\bs')) L^k_{\delta}}{n^{k-1}_{h}(s)\vee 1}} + \frac{1}{8H}\E_{\bs'\sim P_{h}(s)}\brs*{\bar{V}^k_{h}(s,\bs') - V^{\pi_k}_{h}(s,\bs')} + \frac{5H^2L^k_{\delta}}{n^{k-1}_h(s)\vee 1}.\label{eq: transition lookahead regret term (ii)}
\end{align}
Relation $(1)$ uses \Cref{lemma: variance difference bound} with the values $0\leq V_h^{\pi^k}(s,\bs')\leq V_h^{*}(s,\bs')\leq H$ with $\alpha=8H\cdot\sqrt{2L^k_{\delta}}$ and $(2)$ is by optimism.

\textbf{Bounding term $(i)$:} We first focus on the transition bonus; to bound it, we apply \Cref{lemma: variance difference bound with different measures} w.r.t. $\hat{P}_h^{k-1}(\bs'\vert s),P_h(\bs'\vert s)$, the values $0\leq V_h^{\pi^k}(s,\bs')\leq V_h^{*}(s,\bs')\leq \bar{V}_h^{k}(s,\bs')\leq 3H$ (by optimism), under the event $E^{pv2}(k)$ and with $\alpha=8H\cdot \frac{20}{3}\sqrt{L^k_{\delta}}$:
\begin{align*}
     b_{k,h}^p(s) 
     & = \frac{20}{3}\sqrt{\frac{\VAR_{\bs'\sim\hat{P}^{k-1}_{h}(s)}(\bar{V}^k_{h}(s,\bs')) L^k_{\delta}}{n^{k-1}_{h}(s)\vee 1}} + \frac{400}{3}\frac{HL^k_{\delta}}{n^{k-1}_{h}(s) \vee 1} \\
     & \leq  \frac{1}{8H}\E_{\bs'\sim\hat{P}^{k-1}_{h}(s)}\brs*{\bar{V}^k_{h}(s,\bs') - V^*_{h}(s,\bs')} + \frac{1}{8H}\E_{\bs'\sim P_{h}(s)}\brs*{V^*_{h}(s,\bs') - V^{\pi_k}_{h}(s,\bs')}\\
     &\quad+ \frac{20}{3}\sqrt{\frac{\VAR_{\bs'\sim P_{h}(s)}(V^{\pi^k}_{h}(s,\bs')) L^k_{\delta}}{n^{k-1}_{h}(s)\vee 1}}
     + \frac{1600H^2}{3n^{k-1}_{h}(s)\vee 1} + \frac{20}{3}\frac{4H L^k_{\delta}}{n^{k-1}_h(s)\vee 1} + \frac{400}{3}\frac{HL^k_{\delta}}{n^{k-1}_{h}(s) \vee 1} \\
     & \leq  \frac{1}{8H}\br*{\E_{\bs'\sim\hat{P}^{k-1}_{h}(s)}\brs*{\bar{V}^k_{h}(s,\bs') - V^*_{h}(s,\bs')} - E_{\bs'\sim P_{h}(s)}\brs*{\bar{V}^k_{h}(s,\bs') - V^*_{h}(s,\bs')} }\\
     &\quad+ \frac{1}{8H}\E_{\bs'\sim P_{h}(s)}\brs*{\bar{V}^k_{h}(s,\bs') - V^{\pi_k}_{h}(s,\bs')}  
     + \frac{20}{3}\sqrt{\frac{\VAR_{\bs'\sim P_{h}(s)}(V^{\pi^k}_{h}(s,\bs')) L^k_{\delta}}{n^{k-1}_{h}(s)\vee 1}}
     + \frac{700H^2}{n^{k-1}_{h}(s)\vee 1} .
\end{align*}
Substituting back to term $(i)$, we now have
\begin{align*}
    (i)& \leq 
    \br*{1+\frac{1}{8H}}\br*{\E_{\bs'\sim\hat{P}^{k-1}_{h}(s)}\brs*{\bar{V}^k_{h}(s,\bs') - V^*_{h}(s,\bs')} - E_{\bs'\sim P_{h}(s)}\brs*{\bar{V}^k_{h}(s,\bs') - V^*_{h}(s,\bs')} }\\
     &\quad+ \frac{1}{8H}\E_{\bs'\sim P_{h}(s)}\brs*{\bar{V}^k_{h}(s,\bs') - V^{\pi_k}_{h}(s,\bs')} 
     + \frac{20}{3}\sqrt{\frac{\VAR_{\bs'\sim P_{h}(s)}(V^{\pi^k}_{h}(s,\bs')) L^k_{\delta}}{n^{k-1}_{h}(s)\vee 1}}+\frac{700H^2L^k_{\delta}}{n^{k-1}_{h}(s) \vee 1}  .
\end{align*}

The next step in the proof involves bounding the first term of $(i)$. At this point, we remind that both values can be written as $\bar{V}^k_{h}(s,\bs')=\max_a\bar{V}^k_{h}(s,s'(a),a)$ and $V^*_{h}(s,\bs')= \max_a V^*_{h}(s,s'(a),a)$, inducing the lists $\bar{\ell}=\bar{\ell}^k_h(s)$ and $\ell^*=\ell^*_h(s)$, respectively; thus the expectations can be written as (see \Cref{appendix: transition lookahead list representation} for further details on the list representation, and in particular, \Cref{eq: diff in list representation}):
\begin{align*}
    &\E_{\bs'\sim \hat{P}^{k-1}_h(s)}\brs*{\bar{V}^k_{h}(s,\bs') - V^*_{h}(s,\bs')} - \E_{\bs'\sim P_h(s)}\brs*{\bar{V}^k_{h}(s,\bs') - V^*_{h}(s,\bs')}\\
    & \overset{(1)}= \E_{i,j\sim\hat{\mu}^{k}_{h}\br*{\cdot\vert s;\bar{\ell},\ell^*}}\brs*{\bar{V}^k_{h}(s,s'_{\bar{\ell}(i)},a_{\bar{\ell}(i)}) - V^*_{h}(s,s'_{\ell^*(j)},a_{\ell^*(j)})} \\
    &\quad- \E_{i,j\sim\mu\br*{\cdot\vert \bar{\ell},\ell^*, P_h(s)}}\brs*{\bar{V}^k_{h}(s,s'_{\bar{\ell}(i)},a_{\bar{\ell}(i)}) - V^*_{h}(s,s'_{\ell^*(j)},a_{\ell^*(j)})} \\
    & \overset{(2)}\leq \frac{1}{8H}\E_{i,j\sim\mu\br*{\cdot\vert \bar{\ell},\ell^*, P_h(s)}}\brs*{\bar{V}^k_{h}(s,s'_{\bar{\ell}(i)},a_{\bar{\ell}(i)}) - V^*_{h}(s,s'_{\ell^*(j)},a_{\ell^*(j)})}  + \frac{3H (SA)^2 L^k_{\delta}(2SA+ 8H\cdot 4SA/4)}{n^{k-1}_{h}(s)\vee 1} \\
    & \overset{(1)}\leq \frac{1}{8H}\E_{\bs'\sim P_h(s)}\brs*{\bar{V}^k_{h}(s,\bs') - V^*_{h}(s,\bs')} + \frac{30H^2 (SA)^3 L^k_{\delta}}{n^{k-1}_{h}(s)\vee 1} \\
    & \leq \frac{1}{8H}\E_{\bs'\sim P_h(s)}\brs*{\bar{V}^k_{h}(s,\bs') - V^{\pi^k}_{h}(s,\bs')} + \frac{30H^2 (SA)^3 L^k_{\delta}}{n^{k-1}_{h}(s)\vee 1}
\end{align*}
    Relations $(1)$ formulate the expectation using the list representations and backward, as done in \Cref{eq: diff in list representation}. 
    For inequality $(2)$ we rely on \Cref{lemma: transition different to next state expectation} with $\alpha=8H$ under the event $E^{\ell}(k)$ and the optimism, which ensures that the value difference is bounded in $[0,3H]$. We also remark that the support of the distributions is of size $(SA)^2$; were we to use the same result on the distributions  $\hat{P}^{k-1}_h(s)$ and $P_h(s)$, the support would be of size $S^A$, which would lead to an exponential additive factor. And so, we finally have a bound of 
\begin{align}
    (i)& \leq \frac{3}{8H}\E_{\bs'\sim P_{h}(s)}\brs*{\bar{V}^k_{h}(s,\bs') - V^{\pi_k}_{h}(s,\bs')} + \frac{20}{3}\sqrt{\frac{\VAR_{\bs'\sim P_{h}(s)}(V^{\pi^k}_{h}(s,\bs')) L^k_{\delta}}{n^{k-1}_{h}(s)\vee 1}}
     + \frac{735H^2(SA)^3L^k_{\delta}}{n^{k-1}_{h}(s) \vee 1}. \label{eq: transition lookahead regret term (i) better asympy}
\end{align}

\textbf{Combining both terms.} Substituting this and \Cref{eq: transition lookahead regret term (ii)} into \Cref{eq: value-difference transition lookahead}, we have 
\begin{align*}
    \bar{V}^{k}_h(s) - V^{\pi^k}_h(s) 
    &\leq  \frac{1}{2H}\E_{\bs'\sim P_{h}(s)}\brs*{\bar{V}^k_{h}(s,\bs') - V^{\pi_k}_{h}(s,\bs')} 
    + 9\sqrt{\frac{\VAR_{\bs'\sim P_{h}(s)}(V^{\pi^k}_{h}(s,\bs'))L^k_{\delta}}{n^{k-1}_{h}(s)\vee 1}} + \frac{750H^2(SA)^3 L^k_{\delta}}{n^{k-1}_{h}(s) \vee 1} \nonumber\\
    &\quad  + 2\E_{\bs'\sim P_h(s)}\brs*{b_{k,h}^r(s,\pi^k_h(s,\bs'))}+\E_{\bs'\sim P_h(s)}\brs*{\bar{V}^{k}_{h+1}(s'(\pi^k_h(s,\bs'))) - V^{\pi^k}_{h+1}(s'(\pi^k_h(s,\bs'))) }.
\end{align*}
and further bounding (using the concentration event $E^{r}(k)$
\begin{align*}
    \bar{V}^k_{h}(s,\bs'))  - V^{\pi^k}_{h}(s,\bs')
    &= \hat{r}_h^{k-1}(s,\pi^k_h(s,\bs')) + b_{k,h}^{r}(s,\pi^k_h(s,\bs')) +\bar{V}^k_{h+1}(s'(\pi^k_h(s,\bs')))\\
    &\quad- r_h^{k-1}(s,\pi^k_h(s,\bs')) - V^{\pi^k}_{h+1}(s'(\pi^k_h(s,\bs')))\\
    &\leq \bar{V}^k_{h+1}(s'(\pi^k_h(s,\bs'))) - V^{\pi^k}_{h+1}(s'(\pi^k_h(s,\bs'))) + 2b_{k,h}^{r}(s,\pi^k_h(s,\bs')),
\end{align*}
we finally get the decomposition 
\begin{align*}
    \bar{V}^{k}_h(s) - V^{\pi^k}_h(s) 
    &\leq  \br*{1+\frac{1}{2H}}\E_{\bs'\sim P_h(s)}\brs*{\bar{V}^{k}_{h+1}(s'(\pi^k_h(s,\bs'))) - V^{\pi^k}_{h+1}(s'(\pi^k_h(s,\bs'))) }\nonumber\\
    &\quad   + 9\sqrt{\frac{\VAR_{\bs'\sim P_{h}(s)}(V^{\pi^k}_{h}(s,\bs'))L^k_{\delta}}{n^{k-1}_{h}(s)\vee 1}} + \frac{750H^2(SA)^3 L^k_{\delta}}{n^{k-1}_{h}(s) \vee 1} + 3\E_{\bs'\sim P_h(s)}\brs*{b_{k,h}^r(s,\pi^k_h(s,\bs'))}.
\end{align*}
At this point, we choose to take $s=s_h^k$ and sum over all $k\in[K]$; specifically, for $\bs'=\bs'^k_{h+1}$, the action becomes $\pi^k_h(s,\bs')=a_h^k$ and $s'(\pi^k_h(s,\bs'))=s_{h+1}^k$. Formally, we can write the bound as 
\begin{align*}
    \sum_{k=1}^K\bar{V}^{k}_h(s_h^k) - V^{\pi^k}_h(s_h^k) 
    &\leq  \br*{1+\frac{1}{2H}}\sum_{k=1}^K\E\brs*{\bar{V}^{k}_{h+1}(s_{h+1}^k) - V^{\pi^k}_{h+1}(s_{h+1}^k)\vert F_{k,h-1} } \\
    &\quad + 3\sum_{k=1}^K\E\brs*{b_{k,h}^r(s_h^k,a_h^k)\vert F_{k,h-1}}+9\sum_{k=1}^K\sqrt{\frac{\VAR_{\bs'\sim P_{h}(s_h^k)}(V^{\pi^k}_{h}(s_h^k,\bs'))L^k_{\delta} }{n^{k-1}_{h}(s_h^k)\vee 1}}\\
    &\quad+ \sum_{k=1}^K\frac{750H^2(SA)^3 L^k_{\delta}}{n^{k-1}_{h}(s_h^k) \vee 1}.
\end{align*}
and, in particular, under the events $E^{\mathrm{diff}}$ and $E^{br}$, it holds that 
\begin{align*}
    \sum_{k=1}^K\bar{V}^{k}_h(s_h^k) - V^{\pi^k}_h(s_h^k) 
    &\leq  \br*{1+\frac{1}{2H}}^2\sum_{k=1}^K\br*{\bar{V}^{k}_{h+1}(s_{h+1}^k)) - V^{\pi^k}_{h+1}(s_{h+1}^k)} + 36H^2 \ln\frac{6HK(K+1)}{\delta}\\
    &\quad+ 3\sum_{k=1}^Kb_{k,h}^r(s_h^k,a_h^k) + 54 \ln\frac{6HK(K+1)}{\delta}  \\
    &\quad + 9\sum_{k=1}^K\sqrt{\frac{\VAR_{\bs'\sim P_{h}(s_h^k)}(V^{\pi^k}_{h}(s_h^k,\bs'))L^k_{\delta} }{n^{k-1}_{h}(s_h^k)\vee 1}} + \sum_{k=1}^K\frac{750H^2(SA)^3 L^k_{\delta}}{n^{k-1}_{h}(s_h^k) \vee 1}.
\end{align*}

To conclude the proof, we recursively apply this formula from $h=1$ to $h=H+1$ (where the values are zero) and use the optimism. This yields
\begin{align*}
    \Regret^T(K) 
    & = \sum_{k=1}^K V_1^*(s_h^k) -V^{\pi^k}_1(s_h^k) \\
    & \leq \sum_{k=1}^K \bar{V}_1^k(s_h^k) -V^{\pi^k}_1(s_h^k) \tag{Optimism} \\
    &\overset{(1)}\leq 9\br*{1+\frac{1}{2H}}^{2H}\sum_{k=1}^K\sum_{h=1}^H\frac{\sqrt{\VAR_{\bs'\sim P_{h}(s_h^k)}(V^{\pi^k}_{h} (s_h^k,\bs')) L^k_{\delta}}}{\sqrt{n^{k-1}_{h}(s_h^k)\vee 1}} \\
    &\quad + 3\br*{1+\frac{1}{2H}}^{2H}\sum_{k=1}^K\sum_{h=1}^H\sqrt{ \frac{L^k_{\delta} }{n^{k-1}_{h}(s_h^k,a_h^k)\vee 1}} \\
    &\quad+ \br*{1+\frac{1}{2H}}^{2H}\sum_{k=1}^K\sum_{h=1}^H\frac{750H^2(SA)^3 L^k_{\delta}}{n^{k-1}_{h}(s_h^k) \vee 1} + 90H^3\br*{1+\frac{1}{2H}}^{2H} \ln\frac{6HK(K+1)}{\delta} \\
    & \overset{(2)}\le 50\sqrt{H^3SK}L^K_{\delta} + 50\sqrt{2S}H^2\br*{L^K_{\delta}}^{1.5} \\
    &\quad + 9 \sqrt{L^K_{\delta}}\br*{SAH + 2\sqrt{SAH^2K}} 
    +2050H^3S^4A^3L^K_{\delta}\br*{2 + \ln(K)} + 250H^3L^K_{\delta} \\
    & = \Ocal\br*{\sqrt{H^2SK}\br*{\sqrt{H}+\sqrt{A}}L^K_{\delta}+H^3S^4A^3\br*{L^K_{\delta}}^2}.
\end{align*}
Relation $(1)$ is the recursive application of the difference alongside substitution of the reward bonuses, while relation $(2)$ is by \Cref{lemma: sum value variance bound transition lookahead} and \Cref{lemma: count sum bounds}. 
\end{proof}

\clearpage

\subsubsection{Lemmas for Bounding Bonus Terms}
\begin{lemma}
\label{lemma: sum value variance bound transition lookahead}
    Under the event $E^{\VAR}$ it holds that
    \begin{align*}
        \sum_{k=1}^K\sum_{h=1}^H \frac{\sqrt{\VAR_{\bs'\sim P_{h}(s_h^k)}(V^{\pi^k}_{h}(s_h^k,\bs')) }}{\sqrt{n^{k-1}_{h}(s_h^k)\vee 1}}
        \leq 2\sqrt{H^3SKL^K_{\delta}} + \sqrt{8S}H^2 L^K_{\delta}.
    \end{align*}
\end{lemma}
\begin{proof}
    Similar to \Cref{lemma: sum value variance bound reward lookahead}, we again rely on the lookahead version of the law of total variation to prove this bound. First, by Cauchy-Schwartz inequality, it holds that
    \begin{align*}
        \sum_{k=1}^K\sum_{h=1}^H \frac{\sqrt{\VAR_{\bs'\sim P_{h}(s_h^k)}(V^{\pi^k}_{h}(s_h^k,\bs')) }}{\sqrt{n^{k-1}_{h}(s_h^k)\vee 1}}
        \leq \sqrt{ \sum_{k=1}^K\sum_{h=1}^H\VAR_{\bs'\sim P_{h}(s_h^k)}(V^{\pi^k}_{h}(s_h^k,\bs')) }\sqrt{ \sum_{k=1}^K\sum_{h=1}^H \frac{1}{n^{k-1}_{h}(s_h^k)\vee 1}}.
    \end{align*}
    We use \Cref{lemma: count sum bounds} to bound the second term by
    \begin{align*}
        \sum_{k=1}^K\sum_{h=1}^H \frac{1}{n^{k-1}_{h}(s_h^k)\vee 1}
        \leq SH\br*{2 + \ln(K)}
    \end{align*}
    and focus on bounding the first term. Under $E^{\VAR}$, we have
    \begin{align*}
        &\sum_{k=1}^K\sum_{h=1}^H\VAR_{\bs'\sim P_{h}(s_h^k)}(V^{\pi^k}_{h}(s_h^k,\bs'))\\
        &\leq 2\sum_{k=1}^K \E\brs*{\sum_{h=1}^H\VAR_{\bs'\sim P_{h}(s_h^k)}(V^{\pi^k}_{h}(s_h^k,\bs'))|F_{k-1}}  + 4H^3 \ln\frac{6HK(K+1)}{\delta} \tag{Under $E^{\VAR}$}\\
        & = 2\sum_{k=1}^K \E\brs*{\br*{\sum_{h=1}^H r_h(s_h^k,a_h^k) - V_1^{\pi^k}(s_1^k) }^2|F_{k-1}}  + 4H^3 \ln\frac{6HK(K+1)}{\delta} \tag{By \Cref{lemma: ltv transition-lookahead} }\\
        & \leq 2H^2K  + 4H^3 \ln\frac{6HK(K+1)}{\delta},
    \end{align*}
    where the last inequality is since both the values and cumulative rewards are bounded in $[0,H]$. Combining both, we get
    \begin{align*}
        \sum_{k=1}^K\sum_{h=1}^H \frac{\sqrt{\VAR_{\bs'\sim P_{h}(s_h^k)}(V^{\pi^k}_{h}(s_h^k,\bs')) }}{\sqrt{n^{k-1}_{h}(s_h^k)\vee 1}}
        &\leq \sqrt{2H^2K  + 4H^3 \ln\frac{6HK(K+1)}{\delta}} \sqrt{SH\br*{2 + \ln(K)}} \\
        & \leq \sqrt{2H^2K  + 4H^3 \ln\frac{6HK(K+1)}{\delta}} \sqrt{2SH\ln\frac{6HK(K+1)}{\delta}} \\
        & \leq 2\sqrt{H^3SKL^K_{\delta}} + \sqrt{8S}H^2 L^K_{\delta}.
    \end{align*}
\end{proof}
\clearpage

\clearpage

\subsection{Example: Value Gain due to Transition Lookahead}
\label{appendix: transition lookahead example}

\begin{figure}[h]
    \centering
    \includegraphics[width=0.85\linewidth]{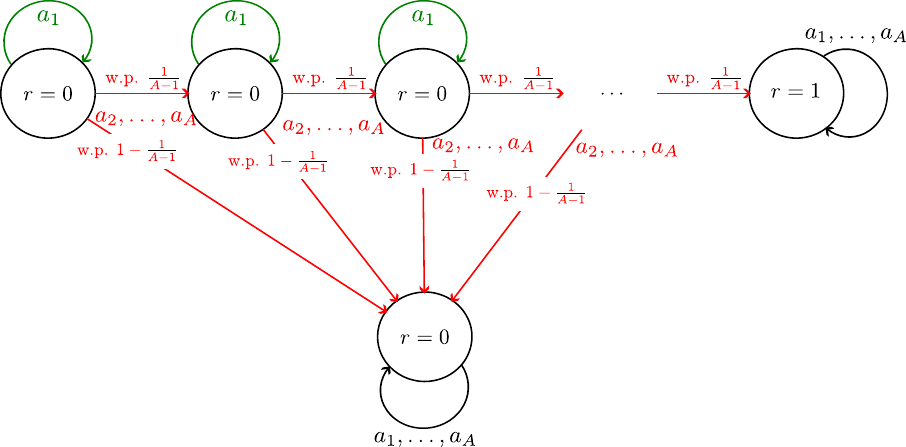} 
     \caption{Random chain: agents start at the left side and must reach its right side to collect a reward.}
     \label{figure: chain transition lookahead}
\end{figure}

We now present in further detail the example described at \Cref{section: comparison to standard RL}. This example is inspired by the one in Appendix C.3 in \citep{merlis2024value}, greatly simplifying it and achieving similar behavior for a much smaller environment. 

Agents start at the left side of a chain of length $H/2$ (depicted in \Cref{figure: chain transition lookahead}) and have two options: 
\begin{enumerate}
    \item Play a safe action $a_1$ that leaves the agent in the same state (in green), or,
    \item play one of the $A-1$ risky actions $a_2,\dots,a_A$ (in red). Each of these actions moves the agent forward in the chain w.p. $\frac{1}{A-1}$, but leads to a terminal non-rewarding state w.p. $1-\frac{1}{A-1}$.
\end{enumerate}
At the end of the chain, the last state is an absorbing state with a unit reward.

Without lookahead, all agents can do is try to randomly reach the end of the chain, succeeding with probability $(A-1)^{-H/2}$. In particular, such agents cannot collect more than $V^{no}\le H(A-1)^{-H/2}$. On the other hand, with transition lookahead, agents observe whether the risky actions allow moving forward in the chain or lead to the bad terminal state. If one action allows progressing in the chain (which happens w.p. $p=1-\br*{1-\frac{1}{A-1}}^{A-1}\ge1-\nicefrac{1}{e}$), a lookahead agent would take it, and otherwise, they will use $a_1$ to remain in the same state. In other words, optimal lookahead agents reach the reward after $H/2-1$ successful 'progression steps' with probability $p$ each. The probability of reaching the end of the chain using less than $5H/6$ steps is at least 
\begin{align*}
    \Pr\br*{\mathrm{Bin}\br*{\frac{5H}{6}-1,1-\frac{1}{e}}> \frac{H}{2}}\ge c_0,
    \quad \textrm{for some absolute } c_0>0. 
\end{align*}
Under this event, the agent collects $\frac{H}{6}$ rewards, so the lookahead value is at least $V^{T,*}\!\ge\! \frac{c_0H}{6}\!=\!\Omega(H)$.

To summarize, for this example, no lookahead optimal value is at most $\approx HA^{-H/2}$, while transition lookahead agents can collect a value of $\approx H$: transition lookahead increases the value by an exponential multiplicative factor. The difference between the two values is $G^T=\Omega(H)$, and following the discussion in \Cref{section: comparison to standard RL}, a sublinear transition lookahead regret would imply a negatively linear standard regret of $\Regret(K)\lesssim -HK$.

\begin{remark}
    The chain length was chosen to be $H/2$ for simplicity -- similar conclusions can be achieved for a length of $\approx 1-\nicefrac{1}{e}$. Then, the multiplicative increase in value due to transition lookahead would be $\approx (A-1)^{\br*{1-\frac{1}{e}}H}$, matching Proposition 2 in \citep{merlis2024value}. In fact, setting the transition from the last state of the chain to the terminal state (rendering it possible to earn only one unit of reward), the analysis coincides with the one in \citep{merlis2024value}. Following their exact derivation, the value with lookahead information is multiplicatively larger than its no-lookahead factor by an exponential factor of $\Theta\br*{(A-1)^{\min\brc*{\br*{1-\frac{1}{e}}H-1,S}-2}}$. This significantly improves the result in \citep{merlis2024value}, that only holds if $S\ge A^{\br*{1-\frac{1}{e}}H}$.
\end{remark}


\section{Auxiliary Lemmas}
In this appendix, we prove various auxiliary lemma that will be used throughout our proofs.
\subsection{Concentration results}
We first present and reprove a set of well-known concentration results.
\begin{lemma}
    \label{lemma: concentration results}
    Let $P$ be a distribution over a discrete set $\X$ of size $\abs*{\X}=M$ and let $X,X_1,\dots,X_n$ be independent samples from this distribution. Also, let $U:\X\mapsto[0,C]$ for some $C>0$ and define the empirical distribution $\hat{P}_n(x) = \frac{1}{n}\sum_{i=1}^n\Ind{x_i=x}$. Then, for any $\delta\in(0,1)$, each of the following events hold w.p. at least $1-\delta$:
    \begin{align*}
        &E^p = \brc*{\forall x\in\X, |P\br*{x} - \hat{P}_n(x)| \le \sqrt{\frac{2P(x)\ln\frac{2M}{\delta}}{n}} + \frac{2\ln\frac{2M}{\delta}}{3n}} \\
        &E^{pv1}=\brc*{\abs*{\sum_{x\in\X}\br*{\hat{P}_n(x)-P(x)} U(x)} \leq \sqrt{\frac{2\VAR_{P}(U(X))\ln\frac{2}{\delta}}{n}} + \frac{2C\ln\frac{2}{\delta}}{3n} }\\
        &E^{pv2} = \brc*{\abs*{\sqrt{\VAR_{\hat{P}_n}(U(X))} - \sqrt{\VAR_{P}(U(X))}} \leq 4C\sqrt{\frac{\ln\frac{2}{\delta}}{n\vee 1}}},
    \end{align*}
    where $\VAR_P(U(X)) = \sum_{x\in\X}P(x)U(x)^2 -  \br*{\sum_{x\in\X}P(x)U(x)}^2$.
\end{lemma}

\begin{proof}
    All the results require standard probability arguments and are stated for completeness.
    
    For the first event $E^p$, notice that each of the components $\hat{P}_n(x)$ is the empirical mean of independent Bernoulli random variables $X_i(x)$ of mean $P(x)$. Therefore, by Bernstein's inequality, recalling that the variance of the variable $Ber(p)$ is $p(1-p)$, we get w.p. at least $1-\frac{\delta}{M}$ that 
 \begin{align*}
     |P(x) - \hat{P}_n(x)| \le \sqrt{\frac{2P(x)(1-P(x))\ln\frac{2M}{\delta} }{n}} +\frac{2\ln\frac{2M}{\delta} }{3n}
     \leq \sqrt{\frac{2P(x)\ln\frac{2M}{\delta} }{n}} +\frac{2\ln\frac{2M}{\delta} }{3n}.
 \end{align*}
 Taking the union bound over all $x\in\X$ implies that $E^p$ holds w.p. at least $1-\delta$.

 For the second event $E^{pv1}$, we apply Bernstein's inequality on the variables $Y_i=U(X_i)$. The empirical mean is given by $\hat{Y}_n=\frac{1}{n}\sum_iU(X_i) = \sum_{x\in\X}\hat{P}_n(x)U(x)$ and its average is $\E[Y]=\sum_{x\in\X}P(x)U(x)$. Similarly, the variance of the random variables is $\VAR(Y)=\VAR_P(U(X))$. Thus, by Bernstein's inequality, w.p. at least  $1-\delta$,
 \begin{align*}
    \abs*{\hat{Y}_n - \E[Y]} \leq \sqrt{\frac{2\VAR(Y)\ln\frac{2}{\delta}}{n}} + \frac{2C\ln\frac{2}{\delta}}{3n}.
 \end{align*}
 Stating the bounds in terms of $X_i$ leads to the second event.

For the last event, we follow the analysis of \citep[][Lemma 19]{efroni2021confidence}, which in turn, relies on \citep[][Theorem 10]{maurer2009empirical}. Define $V_n=\frac{1}{2n(n-1)}\sum_{i,j=1}^n\br*{U(X_i) - U(X_j)}^2$. This is a well-known unbiased variance estimator, namely, $\E\brs*{V_n} = \VAR_{P}(U(X))$, and by \citep[][Theorem 10]{maurer2009empirical}, for any $\delta>0$ it holds w.p. at least $1-\delta$ that 
\begin{align*}
    \abs*{\sqrt{V_n} - \sqrt{\VAR_{P}(U(X))}} \leq C\sqrt{\frac{2\ln\frac{2}{\delta}}{n-1}},
\end{align*}
where we scaled the bound by $C$ to account for the values being in $[0,C]$.

Next, we relate $V_{n}$ to the empirical variance. By elementary algebra, we have
\begin{align*}
    V_{n}
    &=\frac{1}{2n(n-1)}\sum_{i,j=1}^n\br*{U(X_i) - U(X_j)}^2 \\
    &= \frac{1}{n}\sum_{i=1}^nU(X_i)^2 - \frac{1}{n(n-1)}\sum_{i\ne j}U(X_i)U(X_j) \\
    & =  \frac{1}{n}\sum_{i=1}^nU(X_i)^2 - \frac{n}{(n-1)}\br*{\frac{1}{n}\sum_{i}U(X_i)}^2 + \frac{1}{n(n-1)} \sum_{i=1}^nU(X_i)^2 \\
    & = \sum_{x\in\X}\hat{P}_n(x)U(x)^2 - \br*{\sum_{x\in\X}\hat{P}_n(x)U(x)}^2 
     + \frac{1}{n(n-1)}\sum_{i=1}^nU(X_i)^2 - \frac{1}{n^2(n-1)}\br*{\sum_{i=1}^nU(X_i)}^2.
\end{align*}
The first two terms are exactly the variance w.r.t. the empirical distribution; therefore, using the inequality $\abs*{\sqrt{a}-\sqrt{b}}\leq \sqrt{\abs{a-b}}$ for positive numbers, we have
\begin{align*}
    \abs*{\sqrt{V_{n}} - \sqrt{\VAR_{\hat{P}_n}(U(X))}}
    &\leq \sqrt{\abs*{\frac{1}{n(n-1)}\sum_{i=1}^nU(X_i)^2 - \frac{1}{n^2(n-1)}\br*{\sum_{i=1}^nU(X_i)}^2}}
    \leq \sqrt{\frac{C^2}{n-1}}.
\end{align*}
Combining both inequalities and recalling the trivial bound of $C$ on the difference, we get that w.p. at least $1-\delta$,
\begin{align*}
    \abs*{\sqrt{\VAR_{\hat{P}_n}(U(X))} - \sqrt{\VAR_{P}(U(X))}} &\leq \min\brc*{C\sqrt{\frac{2\ln\frac{2}{\delta}}{n-1}} + \sqrt{\frac{C^2}{n-1}},C}  \leq 4C\sqrt{\frac{\ln\frac{2}{\delta}}{n\vee 1}}.
\end{align*}     
\end{proof}

Next, we present a short lemma that allows moving between different spaces of probabilities.
\begin{lemma}
    \label{lemma: empirical expectation of subsets}
    Let $\X$ be a finite set and let $X_1,\dots,X_n\in\X$. Also, let $E_1,\dots,E_m\subseteq \X$ be a partition of the set $\X$, namely, for all $i\ne j$, $E_i\cap E_j=\emptyset$ and $\cup_{i=1}^m E_i=\X$. Finally, let $f:\X\mapsto \R$ such that for all $i\in[m]$ and $x\in E_i$, it holds that $f(x)=f(i)$, and define
    \begin{align*}
        \hat{P}_n(x) = \frac{1}{n}\sum_{\ell=1}^n\Ind{X_\ell=x},\quad\textrm{and},\quad \hat{Q}_n(i) = \frac{1}{n}\sum_{\ell=1}^n\Ind{X_\ell\in E_i}.
    \end{align*} 
    Then, the following hold:
    \begin{enumerate}
        \item $\hat{Q}_n(i) = \hat{P}_n(E_i)\triangleq \sum_{x\in E_i}\hat{P}_n(x)$ and, in particular, $\E_{i\sim \hat{Q}_n}\brs*{f(i)} = \E_{x\sim \hat{P}_n}\brs*{f(x)}$.
        \item If $P$ is a distribution over  $\X$ and $X_1,\dots,X_n\in\X$ are i.i.d. samples from $P$, then $\E[\hat{Q}_n(i)] = P(E_i)\triangleq Q(i)$. It also holds that $\E_{x\sim  P}\brs*{f(x)} = \E_{i\sim Q}\brs*{f(i)}$.
    \end{enumerate}
\end{lemma}
\begin{proof}
    For the first part, we have by definition that
    \begin{align*}
        \hat{Q}_n(i) 
        &=\frac{1}{n}\sum_{\ell=1}^n\Ind{X_\ell\in E_i}
        =\sum_{x\in\X}\frac{1}{n}\sum_{\ell=1}^n\Ind{X_\ell=x}\Ind{x\in E_i}
        =\sum_{x\in\X}\hat{P}_n(x)\Ind{x\in E_i} \\
        &=\sum_{x\in E_i}\hat{P}_n(x)
        = \hat{P}_n(E_i).
    \end{align*}    
    In particular, it holds that 
    \begin{align*}
        \E_{i\sim \hat{Q}_n}\brs*{f(i)}
        &= \sum_{i=1}^m \hat{Q}_n(i)f(i)
        = \sum_{i=1}^m \sum_{x\in E_i}\hat{P}_n(x)f(i)
        \overset{(1)}= \sum_{i=1}^m \sum_{x\in E_i}\hat{P}_n(x)f(x)
        \overset{(2)}=\sum_{x\in \X}\hat{P}_n(x)f(x)\\
        &= \E_{x\sim \hat{P}_n}\brs*{f(x)},
    \end{align*}
    where $(1)$ is since $f$ is constant inside $E_i$ and $(2)$ is since $\brc*{E_i}_{i=1}^m$ partition $\X$.
    
    For the second part of the statement, notice that since the samples are i.i.d., it holds that $\E\brs*{\hat{P}_n(x)} = P(x)$, and therefore,
    \begin{align*}
        \E[\hat{Q}_n(i)]
        =\E\brs*{\sum_{x\in E_i}\hat{P}_n(x)}
        = \sum_{x\in E_i}P(x)
        =P(E_i)
        =Q(i).
    \end{align*}
    Finally, as in the first part of the statement, it holds that 
    \begin{align*}
        \E_{i\sim Q}\brs*{f(i)}
        &= \sum_{i=1}^m Q(i)f(i)
        = \sum_{i=1}^m \sum_{x\in E_i}P(x)f(i)
        = \sum_{i=1}^m \sum_{x\in E_i}P(x)f(x)
        =\sum_{x\in \X}P(x)f(x)\\
        &= \E_{x\sim P}\brs*{f(x)}.
    \end{align*}
\end{proof}
Finally, we present two specialized concentration results that are needed for reward and transition lookahead, respectively.

\begin{lemma}
    \label{lemma:max concentration rewards}
    Let $X,X_1,\dots X_n\in\R^d$ be i.i.d. random vectors over $[0,1]$ and let $C\ge1$ be some constant. Then, for any $\delta\in(0,1)$, with probability at least $1-\delta$,
    \begin{align*}
        &\forall u\in[0,C]^d, & \abs*{\E\brs*{\max_{i\in[d]}\brc*{X(i)+u(i)}} - \frac{1}{n}\sum_{\ell=1}^n\max_{i\in[d]}\brc*{X_\ell(i)+u(i)} } \leq 3\sqrt{\frac{d\ln\frac{9Cn}{\delta}}{2n}}.
    \end{align*}
\end{lemma}
\begin{proof}
    Denote $m(u) = \E\brs*{\max_{i\in[d]}\brc*{X(i)+u(i)}}$ and $\hat{m}(u)=\frac{1}{n}\sum_{\ell=1}^n\max_{i\in[d]}\brc*{X_\ell(i)+u(i)}$ and fix any $u\in[0,C]^d$. Since the variables are bounded in $[0,1]$, their maximum is bounded almost surely in $[\max_iu(i),\max_iu(i)+1]$, namely, an interval of unit length. Therefore, by Hoeffding's inequality, for any $\delta'\in(0,1)$, w.p. $1-\delta'$
    $$ \abs*{m(u)-\hat{m}(u)} \leq \sqrt{\frac{\ln\frac{2}{\delta'}}{2n}}.$$

    Now, for some $\epsilon\in(0,C]$, let $u_\epsilon$ be the closest vector to $u$ on a grid $\brc*{0,\epsilon,2\epsilon,\dots,C}^d$. Then, it clearly holds that 
    \begin{align*}
        \abs*{m(u)-\hat{m}(u)} \leq \abs*{m(u_\epsilon)-\hat{m}(u_\epsilon)} +2\epsilon.
    \end{align*}
    Taking the union bound over all $\br*{\ceil*{\frac{C}{\epsilon}}+1}^d$ possible choices for $u_\epsilon$ and fixing $\delta' = \frac{\delta}{\br*{\ceil*{\frac{C}{\epsilon}}+1}^d}$, we get w.p. $1-\delta$ for all $u$ that 
    \begin{align*}
        \abs*{m(u)-\hat{m}(u)} \leq \sqrt{\frac{\ln\frac{2\br*{\ceil*{\frac{C}{\epsilon}}+1}^d}{\delta}}{2n}} +2\epsilon 
        \leq \sqrt{\frac{d\ln\frac{6C}{\epsilon\delta}}{2n}} +2\epsilon. 
    \end{align*}
    Now, fixing $\epsilon=\sqrt{\frac{d\ln\frac{6C}{\delta}}{2n}}$ and noting that $\frac{1}{\epsilon}\leq \sqrt{2n}$ for $C\ge1$, we get 
    \begin{align*}
        \abs*{m(u)-\hat{m}(u)} 
        \leq \sqrt{\frac{d\ln\frac{6C\sqrt{2n}}{\delta}}{2n}} +2\sqrt{\frac{d\ln\frac{6C}{\delta}}{2n}}
        \leq \sqrt{\frac{d\ln\frac{9Cn}{\delta}}{2n}} +2\sqrt{\frac{d\ln\frac{6C}{\delta}}{2n}}
        \leq 3\sqrt{\frac{d\ln\frac{9Cn}{\delta}}{2n}}.
    \end{align*}
\end{proof}

\clearpage
\begin{lemma}
    \label{lemma:max concentration transitions}
    Let $X,X_1,\dots X_n\in\R^d$ be i.i.d. random vectors with components supported over the discrete set $[m]$ and let $C\ge1$ be some constant. Then, uniformly over all $u\in[0,C]^{dm}$  w.p. $1-\delta$:
    \begin{align*}
    &\abs*{\E\brs*{\max_i\brc*{u(X(i),i)}} - \frac{1}{n}\sum_{\ell=1}^n\max_i\brc*{u(X_\ell(i),i)} } \\
    &\hspace{9.5em}\leq \sqrt{\frac{2md\ln\frac{6n}{\delta}\VAR\br*{\max_i\brc*{u(X(i),i)}}}{n}} + +\frac{8Cmd\br*{\ln\frac{6n}{\delta}}^{1.5}}{n}.
    \end{align*}
\end{lemma}
\begin{proof}
    We follow a similar path to \Cref{lemma:max concentration rewards} and use a covering argument. Denoting $w(u) = \E\brs*{\max_i\brc*{u(X(i),i)}}$ and $\hat{w}(u)=\frac{1}{n}\sum_{\ell=1}^n\max_i\brc*{u(X_\ell(i),i)}$, by Bernstein's inequality, for any $\delta'\in(0,1)$ and fixed $u\in[0,C]^{dm}$, it holds w.p. $1-\delta'$ that
    \begin{align}
        \abs*{w(u) - \hat{w}(u)} \leq \sqrt{\frac{2\VAR\br*{\max_i\brc*{u(X(i),i)}}\ln\frac{2}{\delta}}{n}} + \frac{2C\ln\frac{2}{\delta}}{3n}.
    \end{align}
    Now, for some $\epsilon\in(0,C]$, let $u_\epsilon$ be the closest matrix to $u$ on a grid $\brc*{0,\epsilon,2\epsilon,\dots,C}^{md}$ and denote $Z(u) = \max_i\brc*{u(X(i),i)}$ with samples $Z_i(u)$. By the smoothness of the max function, it holds that 
    \begin{align*}
        \abs*{Z(u) - Z(u_\epsilon)}\leq \epsilon.
    \end{align*}
    In particular, we also have that 
    \begin{align*}
        &\abs*{\E[Z(u)^2] - \E[Z(u_\epsilon)^2]} \leq \epsilon^2 +2C\epsilon, \qquad\textrm{and} \qquad \abs*{\E[Z(u)]^2 - \E[Z(u_\epsilon)]^2}\leq \epsilon^2 +2C\epsilon,
    \end{align*}
    so we have
    \begin{align*}
        &\abs*{\VAR\br*{\max_i\brc*{u(X(i),i)}} - \VAR\br*{\max_i\brc*{u_\epsilon(X(i),i)}}} = \abs*{\VAR\br*{Z(u)} - \VAR\br*{Z(u_\epsilon)}} \leq 2\epsilon^2 +4C\epsilon.
    \end{align*}
    Similarly, it holds that
    \begin{align*}
        \abs*{w(u)-\hat{w}(u)} \leq \abs*{w(u_\epsilon)-\hat{w}(u_\epsilon)} +2\epsilon.
    \end{align*}
    Taking the union bound over all $\br*{\ceil*{\frac{C}{\epsilon}}+1}^{md}$ possible choices for $u_\epsilon$ and fixing $\delta' = \frac{\delta}{\br*{\ceil*{\frac{C}{\epsilon}}+1}^{dm}}$, we get w.p. $1-\delta$ for all $u$ that 
    \begin{align*}
        \abs*{w(u)-\hat{w}(u)} 
        &\leq \sqrt{\frac{2\VAR\br*{\max_i\brc*{u_\epsilon(X(i),i)}}\ln\frac{2\br*{\ceil*{\frac{C}{\epsilon}}+1}^{md}}{\delta}}{n}} + \frac{2C\ln\frac{2\br*{\ceil*{\frac{C}{\epsilon}}+1}^{md}}{\delta}}{3n} +2\epsilon\\
        &\leq \sqrt{\frac{2md\VAR\br*{\max_i\brc*{u_\epsilon(X(i),i)}}\ln\frac{6C}{\epsilon\delta}}{n}} + \frac{2Cmd\ln\frac{6C}{\epsilon\delta}}{3} +2\epsilon \\
        & \leq \sqrt{\frac{2md\ln\frac{6C}{\epsilon\delta}\br*{\VAR\br*{\max_i\brc*{u(X(i),i)}}+2\epsilon^2 + 4C\epsilon}}{n}} + \frac{2Cmd\ln\frac{6C}{\epsilon\delta}}{3n} +2\epsilon \\
        &\leq \sqrt{\frac{2md\ln\frac{6C}{\epsilon\delta}\VAR\br*{\max_i\brc*{u(X(i),i)}}}{n}} + \sqrt{\frac{8mdC\epsilon\ln\frac{6C}{\epsilon\delta}}{n}}+ \sqrt{\frac{4md\epsilon^2\ln\frac{6C}{\epsilon\delta}}{n}}\\
        &\quad + \frac{2Cmd\ln\frac{6C}{\epsilon\delta}}{3n} +2\epsilon. 
    \end{align*}
    
    Now, fixing $\epsilon=\frac{C\ln\frac{6n}{\delta}}{n}$ and noticing that $\frac{6C}{\epsilon\delta}\leq \frac{6n}{\delta}$, we get 
    \begin{align*}
        \abs*{w(u)-\hat{w}(u)} 
        &\leq \sqrt{\frac{2md\ln\frac{6n}{\delta}\VAR\br*{\max_i\brc*{u(X(i),i)}}}{n}} + \frac{\sqrt{8md}C\ln\frac{6n}{\delta}}{n}+ \frac{\sqrt{4md}C\br*{\ln\frac{6n}{\delta}}^{1.5}}{n^{1.5}} \\
        &\quad+\frac{2Cmd\ln\frac{6n}{\delta}}{3n} +\frac{2C\ln\frac{6C}{\delta}}{n} \\
        & \leq \sqrt{\frac{2md\ln\frac{6n}{\delta}\VAR\br*{\max_i\brc*{u(X(i),i)}}}{n}} +\frac{8Cmd\br*{\ln\frac{6n}{\delta}}^{1.5}}{n}.
    \end{align*}
\end{proof}
\clearpage
\subsection{Count-Related Lemmas}
\begin{lemma}
\label{lemma: count sum bounds}
    The following bounds hold:
    \begin{align*}
        &\sum_{k=1}^K\sum_{h=1}^H \frac{1}{\sqrt{n^{k-1}_{h}(s_h^k,a_h^k)\vee 1}} \leq SAH + 2\sqrt{SAH^2K}, 
        & \sum_{k=1}^K\sum_{h=1}^H \frac{1}{n^{k-1}_{h}(s_h^k,a_h^k)\vee 1} \leq SAH\br*{2 + \ln(K)}, \\
        &\sum_{k=1}^K\sum_{h=1}^H \frac{1}{\sqrt{n^{k-1}_{h}(s_h^k)\vee 1}} \leq SH + 2\sqrt{SH^2K}, 
        & \sum_{k=1}^K\sum_{h=1}^H \frac{1}{n^{k-1}_{h}(s_h^k)\vee 1} \leq SH\br*{2 + \ln(K)}.
    \end{align*}
\end{lemma}    
\begin{proof}
    Recall that every time a state (or state-action) is visited, its visitation-count is increased by $1$, up to $n^{K-1}_h(s,a)$ at the last episode. therefore, we can write
    \begin{align*}
        \sum_{k=1}^K\sum_{h=1}^H \frac{1}{\sqrt{n^{k-1}_{h}(s_h^k,a_h^k)\vee 1}}
        &= \sum_{h=1}^H \sum_{s\in\Scal}\sum_{a\in\Acal}\sum_{k=1}^K\frac{\Ind{s_h^k=s,a_h^k=a}}{\sqrt{n^{k-1}_{h}(s,a)\vee 1}}\\
        & = \sum_{h=1}^H \sum_{s\in\Scal}\sum_{a\in\Acal}\sum_{i=0}^{n^{K-1}_h(s,a)}\frac{1}{\sqrt{i\vee 1}} \\
        & \leq \sum_{h=1}^H \sum_{s\in\Scal}\sum_{a\in\Acal}\br*{1+2\sqrt{n^{K-1}_h(s,a)}} \\
        & \leq SAH + 2\sqrt{SAH \sum_{h=1}^H\sum_{s\in\Scal}\sum_{a\in\Acal}n^{K-1}_h(s,a)}\tag{Jensen's inequality } \\
        &\leq SAH + 2\sqrt{SAH^2K}.
    \end{align*}
    where we bounded the total number of visits by the number of steps $HK$. Similarly, we also have
    \begin{align*}
        \sum_{k=1}^K\sum_{h=1}^H \frac{1}{n^{k-1}_{h}(s_h^k,a_h^k)\vee 1}
        & = \sum_{h=1}^H \sum_{s\in\Scal}\sum_{a\in\Acal}\sum_{i=0}^{n^{K-1}_h(s,a)}\frac{1}{i\vee 1} \\
        & \leq \sum_{h=1}^H \sum_{s\in\Scal}\sum_{a\in\Acal}\br*{2+\ln\br*{n^{K-1}_h(s,a)\vee 1}} 
        \leq SAH\br*{2 + \ln(K)}.
    \end{align*}
    We can likewise prove the inequalities for the state counts as follows:
    \begin{align*}
        \sum_{k=1}^K\sum_{h=1}^H \frac{1}{\sqrt{n^{k-1}_{h}(s_h^k)\vee 1}}
        &= \sum_{h=1}^H \sum_{s\in\Scal}\sum_{k=1}^K\frac{\Ind{s_h^k=s}}{\sqrt{n^{k-1}_{h}(s)\vee 1}}\\
        & = \sum_{h=1}^H \sum_{s\in\Scal}\sum_{i=0}^{n^{K-1}_h(s)}\frac{1}{\sqrt{i\vee 1}} \\
        & \leq \sum_{h=1}^H \sum_{s\in\Scal}\br*{1+2\sqrt{n^{K-1}_h(s)}} \\
        & \leq SH + 2\sqrt{SH \sum_{h=1}^H\sum_{s\in\Scal}n^{K-1}_h(s)}\tag{Jensen's inequality } \\
        &\leq SH + 2\sqrt{SH^2K},
    \end{align*}
    and
    \begin{align*}
        \sum_{k=1}^K\sum_{h=1}^H \frac{1}{n^{k-1}_{h}(s_h^k)\vee 1}
        & = \sum_{h=1}^H \sum_{s\in\Scal}\sum_{i=0}^{n^{K-1}_h(s)}\frac{1}{i\vee 1} 
         \leq \sum_{h=1}^H \sum_{s\in\Scal}\br*{2+\ln\br*{n^{K-1}_h(s)\vee 1}} 
         \leq SH\br*{2 + \ln(K)}.
    \end{align*}
\end{proof}

\clearpage
\subsection{Analysis of Variance terms}
\begin{lemma}
\label{lemma: variance difference bound}
    Let $P$ be a distribution over a finite set $\X$ and let $X\sim P$. Also, let $V_1,V_2:\X\mapsto[0,C]$ for some $C>0$ such that $V_1(x)\leq V_2(x)$ for all $x\in\X$. Then, for any $\alpha,n>0$, it holds that 
    \begin{align*}
        \frac{\sqrt{\VAR_{P}(V_2(X)) }}{\sqrt{n}}
        \leq \frac{\sqrt{\VAR_{P}(V_1(X)) }}{\sqrt{n}} + \frac{1}{\alpha}\E_P\brs*{V_2(X)-V_1(X)} + \frac{C\alpha}{4n}
    \end{align*}
\end{lemma}
\begin{proof}
    By \Cref{lemma:std diff}, we have
    \begin{align*}
        \sqrt{\VAR_{P}(V_2(X)) } - \sqrt{\VAR_{P}(V_1(X)) }
        &\leq \sqrt{\VAR_{P}(V_2(X)-V_1(X)) }\\
        & \leq \sqrt{\E_{P}\brs*{(V_2(X)-V_1(X))^2} }\\
        & \leq \sqrt{C\E_{P}\brs*{V_2(X)-V_1(X)}}
    \end{align*}
    where the last inequality is by the boundedness and since $V_1(x)\leq V_2(x)$. Thus, we can bound 
    \begin{align*}
        \frac{\sqrt{\VAR_{P}(V_2(X)) } - \sqrt{\VAR_{P}(V_1(X)) }}{\sqrt{n}}
        &\leq \frac{\sqrt{C\E_{P}\brs*{V_2(X)-V_1(X)}}}{\sqrt{n}}\\
        & = \sqrt{\E_{P}\brs*{V_2(X)-V_1(X)}}\cdot \sqrt{\frac{C}{n}} \\
        & \leq \frac{1}{\alpha}\E_{P}\brs*{V_2(X)-V_1(X)} + \frac{C\alpha}{4n},
    \end{align*}
    where last inequality is due to Young's inequality ($ab\leq \frac{1}{\alpha}a^2 + \frac{\alpha}{4}b^2$ for all $\alpha>0$).
\end{proof}

\begin{lemma}
\label{lemma: variance difference bound with different measures}
    Let $P,P'$ be distributions over a finite set $\X$ and let $X\sim P$. Also, let $V_1,V_2,V_3:\X\mapsto[0,C]$ for some $C>0$ such that $V_1(x)\leq V_2(x)\leq V_3(x)$ for all $x\in\X$. Finally, assume that 
    \begin{align*}
        \abs*{\sqrt{\VAR_{P}(V_2(X))} - \sqrt{\VAR_{P'}(V_2(X)) }} \leq \beta
    \end{align*}
    for some $\beta>0$. Then, for any $\alpha,n>0$, it holds that 
    \begin{align*}
        \frac{\sqrt{\VAR_{P'}(V_3(X)) }}{\sqrt{n}}
        &\leq \frac{\sqrt{\VAR_{P}(V_1(X)) }}{\sqrt{n}} + \frac{1}{\alpha}\E_{P'}\brs*{V_3(X)-V_2(X)}  
         + \frac{1}{\alpha}\E_{P}\brs*{V_2(X)-V_1(X)} 
         + \frac{C\alpha}{2n} + \frac{\beta}{\sqrt{n}} \\
         &\leq \frac{\sqrt{\VAR_{P}(V_1(X)) }}{\sqrt{n}} + \frac{1}{\alpha}\E_{P'}\brs*{V_3(X)-V_1(X)}  
         + \frac{1}{\alpha}\E_{P}\brs*{V_3(X)-V_1(X)} 
         + \frac{C\alpha}{2n} + \frac{\beta}{\sqrt{n}}.
    \end{align*}
\end{lemma}
\begin{proof}
    We decompose the l.h.s. as follows
    \begin{align*}
        \frac{\sqrt{\VAR_{P'}(V_3(X)) }}{\sqrt{n}} 
        &= \frac{\sqrt{\VAR_{P'}(V_3(X))} - \sqrt{\VAR_{P'}(V_2(X))}}{\sqrt{n}} 
         + \frac{\sqrt{\VAR_{P'}(V_2(X))} - \sqrt{\VAR_{P}(V_2(X))}}{\sqrt{n}} \\
        & \quad + \frac{\sqrt{\VAR_{P}(V_2(X))} - \sqrt{\VAR_{P}(V_1(X))}}{\sqrt{n}} 
        + \frac{\sqrt{\VAR_{P}(V_1(X))}}{\sqrt{n}}
    \end{align*}
    We bound the first and third terms using \Cref{lemma: variance difference bound} and bound the second term with the assumption and get 
    \begin{align*}
        \frac{\sqrt{\VAR_{P'}(V_3(X)) }}{\sqrt{n}} 
        &\leq  \frac{1}{\alpha}\E_{P'}\brs*{V_3(X)-V_2(X)}+\frac{C\alpha}{4n}  
         + \frac{\beta}{\sqrt{n}} \\
        & \quad + \frac{1}{\alpha}\E_{P}\brs*{V_2(X)-V_1(X)}+\frac{C\alpha}{4n}  
         + \frac{\sqrt{\VAR_{P}(V_1(X))}}{\sqrt{n}} \\
        &= \frac{\sqrt{\VAR_{P}(V_1(X)) }}{\sqrt{n}} + \frac{1}{\alpha}\E_{P'}\brs*{V_3(X)-V_2(X)} 
         + \frac{1}{\alpha}\E_{P}\brs*{V_2(X)-V_1(X)} 
         + \frac{C\alpha}{2n} + \frac{\beta}{\sqrt{n}} \\
         &\leq \frac{\sqrt{\VAR_{P}(V_1(X)) }}{\sqrt{n}} + \frac{1}{\alpha}\E_{P'}\brs*{V_3(X)-V_1(X)}  
         + \frac{1}{\alpha}\E_{P}\brs*{V_3(X)-V_1(X)} 
         + \frac{C\alpha}{2n} + \frac{\beta}{\sqrt{n}},
    \end{align*}
    where the last inequality uses the fact that $V_1(x)\leq V_2(x)\leq V_3(x)$ for all $x\in\X$. The last two bounds are the desired results.
\end{proof}

\clearpage


\section{Existing Results}
\begin{lemma}[Monotonic Bonuses,\citep{zhang2023settling}, Appendix C.1]
\label{lemma: bonus monotonicity}
For any $p\in\Delta^S$, $v\in\R_+^S$ s.t. $\norm*{v}_\infty\le H$, $\delta'\in(0,1)$ and positive integer $n$, define the function
\begin{align*}
    f(p,v,n) = p^Tv + \max\brc*{\frac{20}{3}\sqrt{\frac{\VAR_p(v)\ln\frac{1}{\delta'}}{n}},\frac{400}{9}\frac{H\ln\frac{1}{\delta'}}{n}}.
\end{align*}
    Then, the function $f(p,v,n)$ is non-decreasing in each entry of $v$.
\end{lemma}

\begin{lemma}[\citealt{efroni2021confidence}, Lemma 28]\label{lemma: transition different to next state expectation}
Let $Y\in \mathbb{R}^{S}$ be a  vector such that $0\leq Y(s) \leq H$ for all $s\in \Scal$. Let $P_1$ and $P_2$ be two transition models and $n\in \mathbb{R}^{SA}_+$. If  
\begin{align*}
        \brc*{\forall (s,a,s')\in \Scal\times\Acal\times \Scal, h\in [H]:\ |P_{2,h} (s'| s,a) - P_{1,h} (s'| s,a)| \le \sqrt{\frac{ C_1 L^k_{\delta} P_{1,h}(s'|s,a) }{n(s,a) \vee 1}} + \frac{C_2 L^k_{\delta}}{n(s,a)\vee 1}},
\end{align*}
for some $C_1,C_2>0$, then, for any $\alpha>0$,
$$
\abs*{\br*{P_{1,h} - P_{2,h}}Y(s,a)}\leq \frac{1}{\alpha} \E_{s'\sim P_{1,h}(\cdot|s,a)}\brs*{ Y(s')} + \frac{H S L^k_{\delta}(C_2+ \alpha C_1/4)}{n(s,a)\vee 1},
$$
\end{lemma}

\begin{lemma}[\citealt{efroni2021confidence}, Lemma 27]\label{lemma: consequences of optimism and freedman's inequality}
Let $\brc{Y_t}_{t\geq 1}$ be a real-valued sequence of random variables adapted to a filtration $\brc*{F_t}_{t\geq 0}$. Assume that for all $t\geq 1$ it holds that $0\leq Y_{t}\leq C$ a.s., and let $T\in \mathbb{N}$. Then each of the following inequalities holds with probability greater than $1-\delta$.
\begin{align*}
   &\sum_{t=1}^T \E[Y_t|F_{t-1}]\leq \br*{1+\frac{1}{2C}} \sum_{t=1}^T Y_t + 2(2C+1)^2 \ln\frac{1}{\delta},\\
   &\sum_{t=1}^T Y_t \leq 2\sum_{t=1}^T \E[Y_t|F_{t-1}] + 4C\ln\frac{1}{\delta}.
\end{align*}
\end{lemma}

\begin{lemma}[Standard Deviation Differences, e.g., \citealt{zanette2019tighter}, lines 48-51]
\label{lemma:std diff}
    Let $P\in\Delta_d$ be some distribution over $[d]$ and let $V_1,V_2\in\R^{d}$. Then, it holds that $$\sqrt{\VAR_P(V_1)}-\sqrt{\VAR_P(V_2)}\leq \sqrt{\VAR_P(V_1-V_2)}.$$
\end{lemma}

\begin{lemma}[Law of Total Variance, e.g.,~\citealt{zanette2019tighter}, Lemma 15]\label{lemma: ltv no-lookahead} 
For any no-lookahead policy $\pi$, it holds that
\begin{align*}
    \E\brs*{\sum_{h=1}^H\VAR(V^\pi_{h+1}(s_{h+1})\vert s_h ) |\pi,s_1} = \E\brs*{\br*{\sum_{h=1}^H r_h(s_h,a_h) - V_1^\pi(s_1) }^2|\pi,s_1},
\end{align*}
where $\VAR(V^\pi_{h+1}(s_{h+1})\vert s_h )$ is the variance of the value at step $s_{h+1}$ given state $s_h$ and under the policy $\pi$, due to the policy randomization and next-state transition probabilities.
\end{lemma}
\clearpage

\section*{NeurIPS Paper Checklist}

\begin{enumerate}

\item {\bf Claims}
    \item[] Question: Do the main claims made in the abstract and introduction accurately reflect the paper's contributions and scope?
    \item[] Answer: \answerYes{} 
    \item[] Justification: \textit{In the abstract, we accurately present the setting and its motivation, as well as a summary of the results, all of which are proved in the appendix.}

\item {\bf Limitations}
    \item[] Question: Does the paper discuss the limitations of the work performed by the authors?
    \item[] Answer: \answerYes{} 
    \item[] Justification: \emph{The main limitations in this work are a result of the studied setup -- some possible extensions and improvement are discussed in the future work section.}

\item {\bf Theory Assumptions and Proofs}
    \item[] Question: For each theoretical result, does the paper provide the full set of assumptions and a complete (and correct) proof?
    \item[] Answer: \answerYes{} 
    \item[] Justification: \textit{Proofs for all the stated results are provided in the appendix.}

    \item {\bf Experimental Result Reproducibility}
    \item[] Question: Does the paper fully disclose all the information needed to reproduce the main experimental results of the paper to the extent that it affects the main claims and/or conclusions of the paper (regardless of whether the code and data are provided or not)?
    \item[] Answer: \answerNA{} 
    \item[] Justification: \textit{The paper does not include experiments.}

\item {\bf Open access to data and code}
    \item[] Question: Does the paper provide open access to the data and code, with sufficient instructions to faithfully reproduce the main experimental results, as described in supplemental material?
    \item[] Answer: \answerNA{} 
    \item[] Justification: \textit{The paper does not include experiments.}

\item {\bf Experimental Setting/Details}
    \item[] Question: Does the paper specify all the training and test details (e.g., data splits, hyperparameters, how they were chosen, type of optimizer, etc.) necessary to understand the results?
    \item[] Answer: \answerNA{} 
    \item[] Justification: \textit{The paper does not include experiments.}

\item {\bf Experiment Statistical Significance}
    \item[] Question: Does the paper report error bars suitably and correctly defined or other appropriate information about the statistical significance of the experiments?
    \item[] Answer: \answerNA{} 
    \item[] Justification: \textit{The paper does not include experiments.}

\item {\bf Experiments Compute Resources}
    \item[] Question: For each experiment, does the paper provide sufficient information on the computer resources (type of compute workers, memory, time of execution) needed to reproduce the experiments?
    \item[] Answer: \answerNA{} 
    \item[] Justification: \textit{The paper does not include experiments.}
    
\item {\bf Code Of Ethics}
    \item[] Question: Does the research conducted in the paper conform, in every respect, with the NeurIPS Code of Ethics \url{https://neurips.cc/public/EthicsGuidelines}?
    \item[] Answer: \answerYes{} 
    \item[] Justification: \textit{The paper is purely theoretical and studies a fundamental decision-making model; any ethical issue that might arise would be a core issue in the ethics of applying machine learning, and not tied specifically to this work.}

\item {\bf Broader Impacts}
    \item[] Question: Does the paper discuss both potential positive societal impacts and negative societal impacts of the work performed?
    \item[] Answer: \answerNA{} 
    \item[] Justification: \textit{Due to the theoretical nature of the paper and the generality of the model, it is no direct societal impact.}
    
\item {\bf Safeguards}
    \item[] Question: Does the paper describe safeguards that have been put in place for responsible release of data or models that have a high risk for misuse (e.g., pretrained language models, image generators, or scraped datasets)?
    \item[] Answer: \answerNA{} 
    \item[] Justification: \textit{No data or models are released with this paper.}

\item {\bf Licenses for existing assets}
    \item[] Question: Are the creators or original owners of assets (e.g., code, data, models), used in the paper, properly credited and are the license and terms of use explicitly mentioned and properly respected?
    \item[] Answer: \answerNA{} 
    \item[] Justification: \textit{The paper does not use existing assets.}

\item {\bf New Assets}
    \item[] Question: Are new assets introduced in the paper well documented and is the documentation provided alongside the assets?
    \item[] Answer: \answerNA{} 
    \item[] Justification: \textit{The paper does not release new assets.}

\item {\bf Crowdsourcing and Research with Human Subjects}
    \item[] Question: For crowdsourcing experiments and research with human subjects, does the paper include the full text of instructions given to participants and screenshots, if applicable, as well as details about compensation (if any)? 
    \item[] Answer: \answerNA{} 
    \item[] Justification: \textit{The paper does not involve crowdsourcing nor research with human subjects.}

\item {\bf Institutional Review Board (IRB) Approvals or Equivalent for Research with Human Subjects}
    \item[] Question: Does the paper describe potential risks incurred by study participants, whether such risks were disclosed to the subjects, and whether Institutional Review Board (IRB) approvals (or an equivalent approval/review based on the requirements of your country or institution) were obtained?
    \item[] Answer: \answerNA{} 
    \item[] Justification: \textit{The paper does not involve crowdsourcing nor research with human subjects.}

\end{enumerate}
\end{document}